%% file: main.tex
\documentclass[accepted]{uai2026} % after acceptance, for a revised version; 
\input{files/aux_}

\newcommand{\papertitle}{Nonlinear Axiomatic Attribution for Cooperative Games}

\title{\papertitle}

\author[1]{Weida~Li}
\author[1]{Zhuanghua~Liu}
\author[2,3]{Yaoliang~Yu}
\author[1]{Bryan~Kian~Hsiang~Low}
\affil[1]{
	Department of Computer Science\\
	National University of Singapore\\
	Republic of Singapore
}
\affil[2]{
	School of Computer Science\\
	University of Waterloo\\
	Canada
}
\affil[3]{%
	Vector Institute\\
	Canada
}

\begin{document}
	\maketitle
	
	\begin{abstract}
		The Shapley value is a widely used concept in attribution problems, as it uniquely satisfies the axioms of linearity, consistency, equal treatment, and efficiency. Often, the inclusion AUC metric is used to evaluate the quality of player rankings, in order to identify positively participating players.
		However, it can be established that the Shapley value is not always reliable for this purpose. The core issue lies in its linearity: the Shapley value acts as a linear operator with an excessively large null space, which is likely to contain non-negligible perturbations that remain indistinguishable to the operator. To address this limitation, we explore the design of nonlinear axiomatic attribution methods. Inspired by the least core, which is a popular nonlinear substitute for the Shapley value, we introduce a class of nonlinear attribution methods that retain the remaining necessary axioms. Each method yields a contribution vector that is the unique optimal solution to a minimization problem, which aims to approximate utility functions as faithfully as possible. In terms of the inclusion AUC metric, our experiments demonstrate the potential effectiveness of these methods compared to Shapley value variants that relax only the efficiency axiom. Our code is available at \url{https://github.com/watml/nonlinear-axiom}.
	\end{abstract}

	\input{files/intro}
	\input{files/background}
	\input{files/motivation}
	\input{files/results}

\input{files/experiments}

	\input{files/conclusion}

	\clearpage
	\begin{acknowledgements} % will be removed in pdf for initial submission,
		% (without ‘accepted’ option in \documentclass)
		% so you can already fill it to test with the
		% ‘accepted’ class option
		This research is supported by the National Research Foundation Singapore and the Singapore Ministry of Digital Development and Innovation, National AI Group under the AI Visiting Professorship Programme (award number AIVP-$2024$-$001$). 
		YY gratefully acknowledges NSERC and CIFAR for funding support.
	\end{acknowledgements}

	\bibliography{files/references}
	\bibliographystyle{plainnat}

	\newpage
	\onecolumn
	
	\title{\papertitle \\(Supplementary Material)}
	\maketitle
	
	\appendix
	\input{files/appendix}
	
\end{document}

%% file: files/aux_.tex
\usepackage[american]{babel}

\usepackage[utf8]{inputenc} % allow utf-8 input
\usepackage[T1]{fontenc}    % use 8-bit T1 fonts
\usepackage{url}            % simple URL typesetting
\usepackage{booktabs}       % professional-quality tables
\usepackage{amsfonts}       % blackboard math symbols
\usepackage{nicefrac}       % compact symbols for 1/2, etc.
\usepackage{microtype}      % microtypography
\usepackage{xcolor}         % colors

\usepackage[breaklinks,colorlinks=true,linkcolor=blue,citecolor=blue,urlcolor=blue]{hyperref}
\usepackage{dsfont}
\usepackage{amsmath}
\usepackage{amssymb}
\usepackage{wrapfig}
\usepackage{natbib}
\usepackage{mathtools}
\usepackage{bm}
\DeclareMathOperator*{\minimize}{minimize}
\DeclareMathOperator*{\maximize}{maximize}
\DeclareMathOperator*{\argmax}{arg\,max}
\DeclareMathOperator*{\argmin}{arg\,min}
\mathtoolsset{showonlyrefs}
\usepackage{flexisym}
\usepackage{dsfont}
\usepackage{amsthm}
\usepackage{thmtools}
\usepackage{graphicx}
\graphicspath{{figs/}}
\usepackage{enumitem}
\usepackage{algorithm}
\usepackage[linesnumbered,ruled,vlined,algo2e]{algorithm2e}
\usepackage{thm-restate}
\usepackage{stmaryrd}

\theoremstyle{plain}
\newtheorem{theorem}{Theorem}[section]

\newtheorem{lemma}[theorem]{Lemma}
\newtheorem{corollary}[theorem]{Corollary}
\theoremstyle{definition}

\theoremstyle{remark}

%% file: files/intro.tex
\section{Introduction}
\label{sec:intro}
Originally, the concept of the Shapley value was introduced by \citet{shapley1953value} to define a fair allocation of contributions among $n$ players who participate in a cooperative game, represented by a utility function $U \colon 2^{[n]} \to \mathbb{R}$, where $[n] \coloneq \{1, 2, \dots, n\}$. It is considered fair because it uniquely satisfies the axioms of linearity, consistency, equal treatment, and efficiency. The effectiveness of the Shapley value and its variants obtained by relaxing the efficiency axiom has been demonstrated by \citet{jia2019towards,kwon2022beta,li2023robust,wang2023data} in data attribution, where data are treated as players and $U(S)$ is defined as the performance of a model trained on the subset $S$ of available training data. One such variant, namely the Banzhaf value \citep{banzhaf1965weighted}, was used in context attribution by \citet{cohen2024contextcite}, where $U(S)$ represents the probability of a statement generated by a large language model when only the parts of the context text prescribed by $S$ are used. Its popularity has also been prominently observed in feature attribution, where $U(S)$ is the predicted value of the model when features outside the subset $S$ are considered missing \citep[e.g.,][]{lundberg2017unified,kwon2022weightedshap}.

However, \citet{kumar2020problems,yan2021if} have raised concerns about the imposed linearity axiom, as its necessity remains unclear.
Meanwhile, \citet{bilodeau2024impossibility,wang2024rethinking} have theoretically established that the use of the Shapley value can be unreliable by exploiting linearity.
Accordingly, a nonlinear alternative, the least core \citep{shapley1971cores}, has attracted attention \citep{yan2021if,gemp2024approximating}. In general, the least core may contain infinitely many contribution allocations, so we instead consider the egalitarian least core \citep{arin2008axiomatic,yan2021if}.\footnote{The least core is a convex subset of $\mathbb{R}^{n}$, and the egalitarian least core is defined as the element in the least core with the smallest $\ell_2$ norm.}
Still, the egalitarian least core is constrained by the efficiency axiom, which has been empirically shown to be unnecessary, and even unfavorable, by \citet{kwon2022beta,kwon2022weightedshap}.
By relaxing only the efficiency axiom, one arrives at the family of semi-values \citep{dubey1981value}, which includes Beta Shapley values \citep{kwon2022beta} and weighted Banzhaf values \citep{li2023robust}, uniquely characterized by the axioms of linearity, consistency, equal treatment, and monotonicity.\footnote{The Shapley value also satisfies the monotonicity axiom.} In this work, we theoretically explore the design of attribution methods by relaxing only the axioms of linearity and efficiency.

\paragraph{Our contributions.} Often, the quality of player rankings is evaluated using the inclusion Area Under the Curve (AUC) metric, when the goal is to identify positively contributing players \citep{petsiuk2018rise,kwon2022beta,covert2023learning}. 
Under this metric, it can be demonstrated (Lemmas~\ref{lem:two-sided least core} and~\ref{lem:liearized AUC}) that applying the egalitarian least core to rank players is equivalent to approximately maximizing the inclusion AUC by substituting an additive approximation of the utility function $U$. This perspective follows from the imposed efficiency axiom. Likewise, the use of semi-values in player ranking can be interpreted in a similar way (Corollary~\ref{cor:how semi-values work}).
\begin{itemize}
	\item We then present a theoretical result demonstrating that the Shapley value is not reliable for maximizing the inclusion AUC, due to a vulnerability introduced by the linearity axiom (Theorem~\ref{thm:unreliable}).
	
	\item In Section~\ref{subset:pathological}, based on the Shapley value, we show via a pathological example  that using the sum of contribution vectors of different utility functions as the contribution vector of their sum can lead to undesirable behavior w.r.t. maximizing the inclusion AUC.
\end{itemize}
These observations motivate us to explore nonlinear attribution methods that do not impose the linearity axiom, yet still satisfy the remaining desired axioms. In particular, the interpretation of applying the egalitarian least core as well as semi-values to rank players points us toward starting from a faithful additive approximation of utility functions.
Consequently,
\begin{itemize}
	\item We introduce a mathematically grounded class of attribution methods that relaxes only the linearity and efficiency axioms (Theorem~\ref{thm:p norm is good});\footnote{As an unexpected finding, we also concurrently developed another class of such nonlinear axiomatic attribution methods via optimizing the inclusion AUC metric \citep{li2026treegrad}.}
	
	\item We also lay out a theoretical ground for approximately computing the egalitarian least core and our introduced nonlinear attribution method $\phi^{\infty}$ (Theorems~\ref{thm:sovleinfinity} and~\ref{thm:solveELC});
\end{itemize}
Finally, we conduct experiments to verify the practical effectiveness of the proposed methods in improving player ranking quality under the inclusion AUC metric.

%% file: files/background.tex
\section{Background}
\label{sec:bg}
\subsection{Notation}
Let $[n] \coloneq \{1, 2, \dots, n\}$ denote a set of players, where $n \geq 2$. For each $n$, $\mathcal{G}_{n} \coloneq \{U_{n} \colon 2^{[n]} \to \mathbb{R}\}$ refers to the set of utility functions over $n$ players. Throughout the paper, the subscript $n$ indicates the number of players. We then define the space of all such games as $\mathcal{G} \coloneq \bigcup_{n \geq 2} \mathcal{G}_{n}$.
An attribution method $\phi$ maps each utility function $U_{n} \in \mathcal{G}$ to a contribution vector $\phi(U_{n}) \in \mathbb{R}^{n}$, where $\phi_{i}(U_{n})$ measures the contribution of the $i$-th player in the cooperative game defined by $U_{n}$.
The all-one vector in $\mathbb{R}^{m}$ is denoted by $\mathbf{1}_{m}$.
Given $\mathbf{x} \in \mathbb{R}^{n}$ and $b \in \mathbb{R}$, we define an additive utility function $A_{n}(\cdot; \mathbf{x}, b) \in \mathcal{G}_{n}$ by letting
\begin{equation}
	\begin{gathered}
		A_{n}(S; \mathbf{x},b) \coloneq b + \sum_{i\in S}x_{i} \ \text{ for every }\ S \subseteq  [n],
	\end{gathered}
\end{equation}
where by default an empty sum is taken as 0. Plus, $\left\llbracket \cdot \right\rrbracket$ denotes the indicator function.

\subsection{Axioms in Cooperative Game Theory}
In what follows, we present a list of axioms commonly used in the literature \citep[e.g., ][]{shapley1953value,dubey1981value,weber1988probabilistic,li2023robust}.
\begin{itemize}
	\item \textbf{Linearity}: $\phi(c\cdot U_{n} + U_{n}') = c\cdot \phi(U_{n}) + \phi(U_{n}') $ for every $U_{n}, U_{n}' \in \mathcal{G}$ and $c \in \mathbb{R}$.
	
	\item \textbf{Consistency}: Suppose there exists $c \in \mathbb{R}$ such that $U_{n+1}(S) = U_{n}(S\cap [n]) + c\cdot\left\llbracket n+1 \in S\right\rrbracket$ for every $S \subseteq [n+1]$. Then, (i) $\phi_{n+1}(U_{n+1}) = c$, and (ii) $\phi_{i}(U_{n}) = \phi_{i}(U_{n+1})$ for every $i \in [n]$.
	
	\item  \textbf{Equal Treatment}: For some $i,j \in [n]$, if $U_{n}(S\cup \{i\}) = U_{n}(S\cup \{j\})$ for every $S \subseteq [n]\setminus \{i,j\}$, then $\phi_{i}(U_{n}) = \phi_{j}(U_{n})$.
	
	\item \textbf{Monotonicity}: For some $i \in [n]$, (i) if $U_{n}(S\cup \{i\})\geq U_{n}(S)$ for every $S \subseteq [n]\setminus\{ i\}$, then $\phi_{i}(U_{n})\geq 0$; (ii) if $U_{n}(S\cup \{i\})\leq U_{n}(S)$ for every $S \subseteq [n]\setminus \{i\}$, then $\phi_{i}(U_{n})\leq 0$.
	
	\item \textbf{Efficiency}: $\sum_{i\in [n]}\phi_{i}(U_{n}) = U_{n}([n]) - U_{n}(\emptyset)$ for every $U_{n} \in \mathcal{G}$.
	
	\item \textbf{Translation Invariance}: $\phi(U_{n}+C_{n}) = \phi(U_{n})$ whenever $C_{n} \in \mathcal{G}$ is a constant utility function.
\end{itemize}
We note that the consistency axiom implies both the dummy axiom used by \citet{weber1988probabilistic} and the projection axiom introduced by \citet{dubey1981value}. Besides, translation invariance is often implied by linearity and other axioms. In other words, translation invariance may be viewed as a weak version of linearity.

\subsection{Semi-Values}
The family of semi-values is uniquely characterized by all the above axioms except efficiency \citep{dubey1981value}. Specifically, each semi-value corresponds to a Borel probability measure $\mu$ over the interval $[0, 1]$. Accordingly, we use $\phi^{\mu}$ to denote the semi-value parameterized by $\mu$. The contribution of the $i$-th player in a game $U_n$ is given by
\begin{equation}\label{eq:semi-values}
	\begin{gathered}
		\phi^{\mu}_{i}(U_{n}) = \sum_{S\subseteq [n]\setminus i} p_{n,|S|+1}^{\mu} [U_{n}(S\cup \{i\}) - U_{n}(S)]\\
		\text{where }\ p_{n,k}^{\mu} = \int_{0}^{1} t^{k-1}(1-t)^{n-k} \mathrm{d}\mu(t) .
	\end{gathered}
\end{equation}
When the efficiency axiom is additionally imposed, $\mu$ is uniquely determined as the uniform distribution on $[0, 1]$, resulting in the well-known Shapley value. For the weighted Banzhaf value WB$(\tau)$, where $\tau \in (0, 1)$, the corresponding measure $\mu$ is the Dirac delta distribution $\delta_{\tau}$. For the Beta Shapley value $\text{Beta}(\alpha, \beta)$, where $\alpha, \beta \geq 1$, the measure $\mu$ is defined by
\begin{equation}
	\begin{gathered}
		\mu(S) \propto \int_{S} t^{1-\beta}(1-t)^{1-\alpha}\mathrm{d}t
	\end{gathered}
\end{equation}
for every measurable subset $S$ of $[0,1]$.
In particular, Beta$(1,1)$ corresponds to the Shapley value.
Empirically, \citet{kwon2022beta,kwon2022weightedshap,li2023robust} demonstrate that Beta Shapley values and weighted Banzhaf values often perform better in downstream tasks, challenging the necessity of the efficiency axiom. Heuristically, when only player rankings are of interest, the efficiency axiom, as a normalization step, is often regarded as redundant.

%% file: files/motivation.tex
\section{Motivations} \label{sec:motivations}
\subsection{The Egalitarian Least Core} 
Following \citet{yan2021if}, the least core $\mathrm{LC}(U_n)$ of a utility function $U_n$ is a subset of $\mathbb{R}^n$ containing all the optimal solutions of $\mathbf{x}$ to the minimization problem
\begin{equation} \label{op:one-sided least core}
	\begin{gathered}
		\minimize_{\mathbf{x}\in \mathbb{R}^{n}, e\in \mathbb{R}}\ e\\
		\text{s.t. }\ U_{n}(S) - A_{n}(S; \mathbf{x}, U_{n}(\emptyset)) \leq e, \ \forall\ S \subsetneq [n] \\ 
		\text{and }\ A_{n}([n]; \mathbf{x}, U_{n}(\emptyset)) = U_{n}([n]). 
	\end{gathered}
\end{equation}
According to \citet{blum1972direct}, it is clear that the least core $\mathrm{LC}(U_{n})$ is always non-empty.
The egalitarian least core $\phi^{\mathrm{ELC}}$ is then defined as the least-norm element in $\mathrm{LC}(U_n)$. Precisely,
\begin{equation}
	\begin{gathered}
		\phi^{\mathrm{ELC}}(U_{n}) \coloneq \argmin_{\mathbf{x}\in \mathrm{LC}(U_{n})} \|\mathbf{x}\|_{2} .
	\end{gathered}
\end{equation}
The least norm solution is critical for ensuring that $\phi^{\mathrm{ELC}}$ satisfies the axiom of equal treatment \citep{yan2021if}.
To the best of our knowledge, no existing work has established whether $\phi^{\mathrm{ELC}}$ satisfies the other axioms listed above. Therefore, for completeness, we analyze the axiomatic properties of the egalitarian least core.
\begin{restatable}{theorem}{AxiomOfEgalitarian} \label{thm:egalitarian}
	The egalitarian least core $\phi^{\mathrm{ELC}}$ is nonlinear and satisfies the axioms of consistency, equal treatment, monotonicity, efficiency, and translation invariance.
\end{restatable}

This result indicates that the egalitarian least core relaxes only the linearity axiom.
Interestingly, we observe that the efficiency axiom effectively transforms the one-sided error bounds in the problem~\eqref{op:one-sided least core} into two-sided ones. The following lemma makes this precise.

\begin{lemma} \label{lem:two-sided least core}
	Given a utility function $U_{n} \in \mathcal{G}$, a vector $\mathbf{x}' \in \mathrm{LC}(U_{n})$ if and only if $\mathbf{x}'$ is an optimal solution to the problem
	\begin{equation} \label{op:two-sided least core}
		\begin{gathered}
			\minimize_{\mathbf{x} \in \mathbb{R}^{n}, e \in \mathbb{R}}\ e\\
			\text{s.t. }\ -e \leq A_{n}(S; \mathbf{x}, U_{n}(\emptyset)) - U_{n}(S), \ \forall\ S \subsetneq [n]\\
			A_{n}(S; \mathbf{x}, U_{n}(\emptyset)) - U_{n}(S)\leq e + B(S), \ \forall\ S \subsetneq [n]\\
			\text{and }\ \sum_{i\in [n]} x_{i} = U_{n}([n]) - U_{n}(\emptyset)
		\end{gathered}
	\end{equation}
	where $B(S) \coloneq U_{n}([n]) + U_{n}(\emptyset) - U_{n}(S) - U_{n}([n]\setminus S)$.
\end{lemma}
\begin{proof}
	Suffice it to demonstrate that $A_{n}(S; \mathbf{x}, U_{n}(\emptyset)) \leq e + B(S)$ is equivalent to
	\begin{equation}
		\begin{gathered}
			U_{n}([n]\setminus S) - A_{n}([n]\setminus S; \mathbf{x}, U_{n}(\emptyset)) \leq e.
		\end{gathered}
	\end{equation}
	Using the efficiency constraint,
	\begin{equation}
		\begin{gathered}
			A_{n}([n]\setminus S; \mathbf{x}, U_{n}(\emptyset)) + A_{n}(S; \mathbf{x}, U_{n}(\emptyset)) \\
			= U_{n}([n]) + U_{n}(\emptyset),
		\end{gathered}
	\end{equation}
	which leads to
	\begin{equation}
		\begin{gathered}
			U_{n}([n]\setminus S) - A_{n}([n]\setminus S; \mathbf{x},U_{n}(\emptyset)) \\
			= A_{n}(S; \mathbf{x},U_{n}(\emptyset)) - U_{n}(S) - B(S) .
		\end{gathered}
	\end{equation}
\end{proof}
In summary, the egalitarian least core produces an additive utility function $A_n(\cdot; \phi^{\mathrm{ELC}}(U_n), U_n(\emptyset))$ that seeks to approximate $U_n$ under the problem~\eqref{op:two-sided least core}.

\subsection{Evaluating Player Rankings}
In data attribution \citep{wang2024rethinking}, the utility function $U_n(S)$ is typically defined as the performance of a model trained on the subset $S$ of available training data. Then, the grand set of players $[n]$ corresponds to the set of training data, and $\phi_i(U_n)$ quantifies the contribution of the $i$-th data point to the performance of trained models. In this context, the goal is to identify positively contributing data points, which is closely related to maximizing the inclusion AUC, defined as
\begin{equation} \label{op:auc}
	\begin{gathered}
		\maximize_{\pi \in \Pi_{n}}\ \mathrm{AUC}(\pi; U_{n}) \coloneq \sum_{k=1}^{n} U_{n}(S_{k}(\pi)) \\ 
		\text{where }\ S_{k}(\pi) \coloneq \{\pi(1),\pi(2),\dots,\pi(k)\} .
	\end{gathered}
\end{equation}
Here, $\Pi_n$ denotes the set of all permutations of $[n]$. We note that maximizing the inclusion AUC also applies to context attribution \citep{cohen2024contextcite}, where the aim is to identify the influential parts of the context that generate a statement.
In general, solving the problem~\eqref{op:auc} is computationally challenging due to the factorial number of possible permutations. Nevertheless, the problem becomes tractable when the utility function is additive.

\begin{lemma} \label{lem:liearized AUC}
	Let $\mathbf{x} \in \mathbb{R}^n$, and define $\pi_{\mathbf{x}}$ as the permutation satisfying $x_{\pi_{\mathbf{x}}(j)} \geq x_{\pi_{\mathbf{x}}(j+1)}$ for every $1 \leq j < n$.
	Then, $\pi_{\mathbf{x}}$ is optimal to the problem
	\begin{equation}
		\begin{gathered}
			\maximize_{\pi \in \Pi_{n}}\ \mathrm{AUC}(\pi; A_{n}(\cdot; \mathbf{x}, b))
		\end{gathered}
	\end{equation}
	where $b\in\mathbb{R}$ is arbitrary.
\end{lemma}
\begin{proof}
	This result follows immediately from the observation
	\begin{equation}
		\begin{gathered}
			\mathrm{AUC}(\pi; A_{n}(\cdot; \mathbf{x}, b))  = \sum_{k=1}^{n}(n+1-k)\cdot x_{\pi(k)}
		\end{gathered}
	\end{equation}
	and the rearrangement inequality.
\end{proof}

Combining Lemmas~\ref{lem:two-sided least core} and~\ref{lem:liearized AUC}, we interpret the mechanism of the egalitarian least core in feature ranking as follows:
\begin{itemize}
	\item Obtain the best additive approximation of $U_{n}$ by solving the minimization problem~\eqref{op:two-sided least core}, yielding the additive utility function $A_n(\cdot; \phi^{\mathrm{ELC}}(U_n), U_n(\emptyset))$.
	
	\item Substitute the obtained $A_{n}(\cdot; \phi^{\mathrm{ELC}}(U_{n}), U_{n}(\emptyset))$ into the maximization problem~\eqref{op:auc} to derive an approximate optimal ranking.
\end{itemize}

Meanwhile, semi-values admit a similar interpretation. Specifically, \citet[Theorem 2]{li2024one} proved that there exists a function $T^{\mu} \colon \mathcal{G} \to \mathbb{R}$ such that
\begin{equation}
	\begin{gathered}
		\phi^{\mu}(U_{n}) = \mathbf{x}^{*}_{\mu}(U_{n}) + T^{\mu}(U_{n})\cdot \mathbf{1}_{n}
	\end{gathered}
\end{equation}
where $(\mathbf{x}^{*}_{\mu}(U_{n}), b^{*}_{\mu}(U_{n}))$ is the unique optimal solution to the least square problem
\begin{equation}
	\begin{gathered}
		\minimize_{\mathbf{x}\in \mathbb{R}^{n}, b\in R}\  \sum_{\emptyset \subsetneq S \subsetneq [n]} p_{n-1,|S|}^{\mu} [U_{n}(S) - A_{n}(S; \mathbf{x}, b)]^{2} .
	\end{gathered}
\end{equation}
This immediately leads to the following corollary.
\begin{corollary} \label{cor:how semi-values work}
	Given a utility function $U_{n} \in \mathcal{G}$, let $\pi_{\mu}$ be a ranking of $[n]$ such that $\phi^{\mu}_{\pi_\mu(j)}(U_{n}) \geq \phi^{\mu}_{\pi_\mu(j+1)}(U_{n})$ for every $1\leq j < n$. Then, $\pi_{\mu}$ is an optimal solution to
	\begin{equation}
		\begin{gathered}
			\maximize_{\pi \in \Pi_{n}}\ \mathrm{AUC}[\pi; A_{n}(\cdot; \mathbf{x}_{\mu}^{*}(U_{n}), b_{\mu}^{*}(U_{n}))] .
		\end{gathered}
	\end{equation}
\end{corollary}

This corollary is immediate by observing that the induced player rankings of $\phi^{\mu}(U_{n})$ and $\mathbf{x}_{\mu}^{*}$ are the same.

%\begin{remark}
In summary, many popular axiomatic attribution methods for ranking players (e.g., Beta Shapley values, weighted Banzhaf values, and the egalitarian least core) can be viewed as implicitly learning an additive approximation to the true utility function $U_n$, which is then used to derive an approximate optimal ranking. This motivates us to next explore nonlinear, non-efficient attribution methods through the same lens of approximation.
%\end{remark}

%% file: files/results.tex
\section{Main Results} 
\label{sec:main results}
\subsection{How Unreliable Is the Shapley Value}
Before proceeding, we demonstrate that the Shapley value may fail to distinguish between utility functions in terms of approximately maximizing the inclusion AUC.

\begin{restatable}{theorem}{UnreliableShapley} \label{thm:unreliable}
	Let $\Pi_{n}$ be the set of all permutations of $[n]$ with $n\geq 3$, and let $\phi^{\mathrm{Shap}}$ denote the Shapley value. Then, $\Pi_{n}$ can be partitioned into $\{ \Pi_{n,j} \}_{j\in\mathcal{J}}$ with $|\mathcal{J}|\geq 2$ such that, for every $\Pi_{n,j}$ and every $U_{n} \in \mathcal{G}_{n}$, there exists a perturbation $U_{n}' \in \mathcal{G}_{n}$ satisfying
	\begin{equation}
		\begin{gathered}
			\phi^{\mathrm{Shap}}(U_{n}) = \phi^{\mathrm{Shap}}(U_{n} + U_{n}'),\\
			\text{yet }\ 
			\argmax_{\pi \in \Pi_{n}} \mathrm{AUC}(\pi; U_{n}+U_{n}') \subseteq  \Pi_{n, j}.
		\end{gathered}
	\end{equation}
\end{restatable}

Theorem~\ref{thm:unreliable} states that, preserving the same Shapley value, the AUC-optimal permutation can be manipulated to lie in any $\Pi_{n,j}$.
Specifically, the partition in Theorem~\ref{thm:unreliable} is constructed by exploiting the null space of the Shapley value. In other words, the linearity axiom induces an excessively large null space that can be detrimental. We note that this is not the only negative perspective established so far. For example, \citet[Theorem 2]{wang2024rethinking} showed that the Shapley value is not reliable for maximizing the utility $U(S)$. 
On the other hand, \citet[Theorem B.3]{bilodeau2024impossibility} demonstrated that the Shapley value fails to distinguish local behaviors of different functions. In particular, all these negative results rely on exploiting the null space induced by the linearity axiom.

\subsection{A Pathological Example Induced by Linearity} \label{subset:pathological}
The inarguable merit of linearity is that it divides the calculation of $\phi(U + U')$ into $\phi(U) + \phi(U')$. However, we demonstrate that such a convenience comes with pathological examples in maximizing the problem~\eqref{op:auc}. Let $2^{[4]}$ be ordered as
\begin{equation}
	\begin{gathered}
		\emptyset, \{1\}, \{2\}, \{1, 2\}, \{3\}, \{1, 3\}, \{2,3\}, \{1,2,3\}, \{4\}, \{1,4\}, \\
		\{2,4\}, \{1,2,4\}, \{3,4\}, \{1,3,4\}, \{2,3,4\}, \{1,2,3,4\},
	\end{gathered}
\end{equation}
which is the binary ordering used in \citet{grabisch2000equivalent}. Then, every utility function in $\mathcal{G}_{4}$ can be expressed as a vector. Let
\begin{equation}
	\begin{aligned}
		U_{4} = (&4, -3, -4, -8, 2, 6, 3, -8,\\
		&4, -10, -3, 8, 5, 1, 4, -3),\\
		U'_{4}= (&-5, -2, -7, 7, -10, 4, -4, 9,\\
		& -1, -6, -1, 1, -3, -2, -10, -5).
	\end{aligned}
\end{equation}
Particularly, they are constructed such that $\pi^{\mathrm{Shap}}(U_{4})$ and $\pi^{\mathrm{Shap}}(U'_{4})$ each exactly maximize the problem~\eqref{op:auc}, where $\pi^{\mathrm{Shap}}$ denotes the player ranking induced by the Shapley value. However, one can verify that $\pi^{\mathrm{Shap}}(U_{4}+U'_{4})$ instead \emph{minimizes} the inclusion AUC. This pathological example clearly demonstrates the unexpected behavior that arises from simply adding the Shapley values of $U_{4}$ and $U'_{4}$ when maximizing the inclusion AUC. Moreover, there exist many such examples, as the one presented here was instantly found via random search.

%\begin{remark}
The above example raises a concern about the common practice of using the sum of contribution vectors of different utility functions as the contribution vector of their sum. This phenomenon also serves as our motivation to explore what nonlinear axiomatic attribution methods would look like.
%\end{remark}

\subsection{Nonlinear Attribution Methods}
Looking at the problem~\eqref{op:two-sided least core} of the least core, the term $B$ appears to hinder faithful approximation of $U_{n}$. Recall that this term arises from enforcing the efficiency axiom. To relax this axiom, it is natural to allow two-sided approximation errors for each subset $S$, i.e., bounding $A_n(S; \mathbf{x}, b) - U_{n}(S)$ from both sides. Moreover, to better fit $U_{n}$, we treat the bias term $b$ as a variable to be optimized, rather than fixing it straightforwardly to $U_{n}(\emptyset)$.
All in all, we consider the minimization problem below that attempts to faithfully approximate the utility function using an additive one:
\begin{equation} \label{op:p norm}
	\begin{gathered}
		\minimize_{\mathbf{x}\in \mathbb{R}^{n}, b \in \mathbb{R}}\  \|\mathbf{A}\mathbf{x} + b\cdot\mathbf{1}_{2^{n}} - \mathbf{V}(U_{n})\|_{p}
	\end{gathered}
\end{equation}
where $\mathbf{A} \in \mathbb{R}^{2^{n} \times n}$, and $\mathbf{V}(U_{n}) \in \mathbb{R}^{2^{n}}$ contains $U_{n}(S)$ for all subsets $S \subseteq [n]$. Each row $A_{S}$ of $\mathbf{A}$ is the indicator vector of the subset $S$, i.e., $A_{S,i} = 1$ if $i \in S$, and $0$ otherwise.
Note that when $p = \infty$, the problem~\eqref{op:p norm} coincides with the problem~\eqref{op:two-sided least core} with the efficiency axiom removed and $B(S) \equiv 0$. We note that \citet{CharnesGKR88} considered a similar approximation with $p\in\{1,2,\infty\}$.

Let $\mathrm{OS}_{p}(U_{n})$ denote the set of all optimal solutions of $\mathbf{x}$ to the minimization problem~\eqref{op:p norm}. Then, We define the $p$-norm attribution method $\phi^{p}(U_{n})$ by selecting the minimum $\ell_{2}$-norm solution among $\mathrm{OS}_{p}(U_{n})$:

\begin{equation}
	\begin{gathered}
		\phi^{p}(U_{n}) \coloneq \argmin_{\mathbf{x}\in \mathrm{OS}_{p}(U_{n})} \|\mathbf{x}\|_{2} .
	\end{gathered}
\end{equation}

For $p \in (1, \infty)$, the problem admits a unique solution since the objective is strictly convex. However, for $p \in \{1, \infty\}$, multiple optimal solutions may exist. As an example for $p = \infty$, consider $n = 3$ and define $U_{3}$ such that
\begin{equation}
	\begin{gathered}
		U(S) \coloneq \begin{cases}
			0, & S = \emptyset,\\
			0.5, & S \in \{\{2\}, \{3\}\},\\
			1, &\text{otherwise.}
		\end{cases}
	\end{gathered}
\end{equation}
Then, one can verify that the solution of $b = 0.25$ and $\mathbf{x} = (0.5, 0.25 + t, 0.25 - t)$ is optimal for every $t \in [-0.25, 0.25]$.

\begin{algorithm}[t]
	\DontPrintSemicolon
	\caption{Approximation of $\phi^{p}(U_{n})$}
	\label{alg:appr}
	\KwIn{Utility function $U_{n}$, regularization parameter $\eta = 10^{-12}$, $p \in (1, \infty) \cup \{\infty, \mathrm{ELS}\}$, and sample budget $B$}
	\KwOut{Approximation $\hat{\mathbf{x}}^{*} \in \mathbb{R}^{n}$ of $\phi^{p}(U_{n})$}
	
	$\mathcal{S} \gets [\ ]$
	
	\For{$b = 1,2,\dots, B$}{
		Uniformly sample a subset $S_{b} \subseteq [n]$ not already in $\mathcal{S}$;
		
		\If{$p = \mathrm{ELS}$}{
			Resample $S_{b}$ if $S_{b} \in \{\emptyset, [n]\}$
		}
		
		$\mathcal{S}.\mathrm{append}(S_{b})$
	}
	\If{$p \in (1,\infty) \cup \{\infty\}$}{

		Construct $\mathbf{A} \in \mathbb{R}^{B\times n}$ and $\mathbf{V} \in \mathbb{R}^{B}$ using $\mathcal{S}$ and $U_{n}$
		
		$\eta \gets 0$ if $p \in (1, \infty)$
		
		$\hat{\mathbf{x}}^{*}\gets\minimize_{\mathbf{x}\in \mathbb{R}^{n}, b\in \mathbb{R}}\ \|\mathbf{Ax} + b\cdot\mathbf{1}_{B} - \mathbf{V}\|_{p} + \eta\cdot \|\mathbf{x}\|_{2}^{2}$
	}\ElseIf{$p = \mathrm{ELC}$}{
		$\hat{\mathbf{x}}^{*} \gets \minimize_{\mathbf{x} \in \mathbb{R}^{n},e\in \mathbb{R}}\ e + \eta\cdot \|\mathbf{x}\|_{2}^{2}\,$ subject to $\,U_{n}(S) - U_{n}(\emptyset) - \sum_{i\in S} x_{i} \leq e$ for every $S \in \mathcal{S}$, $0\leq e$, and $\,\sum_{i\in[n]}x_{i} = U_{n}([n]) - U_{n}(\emptyset)$
	}
	
	\Return{$\hat{\mathbf{x}}^{*}$}
\end{algorithm}

\begin{figure*}[t]
	\centering
	\begin{tabular}{ccc}
		\includegraphics[width=0.3\linewidth]{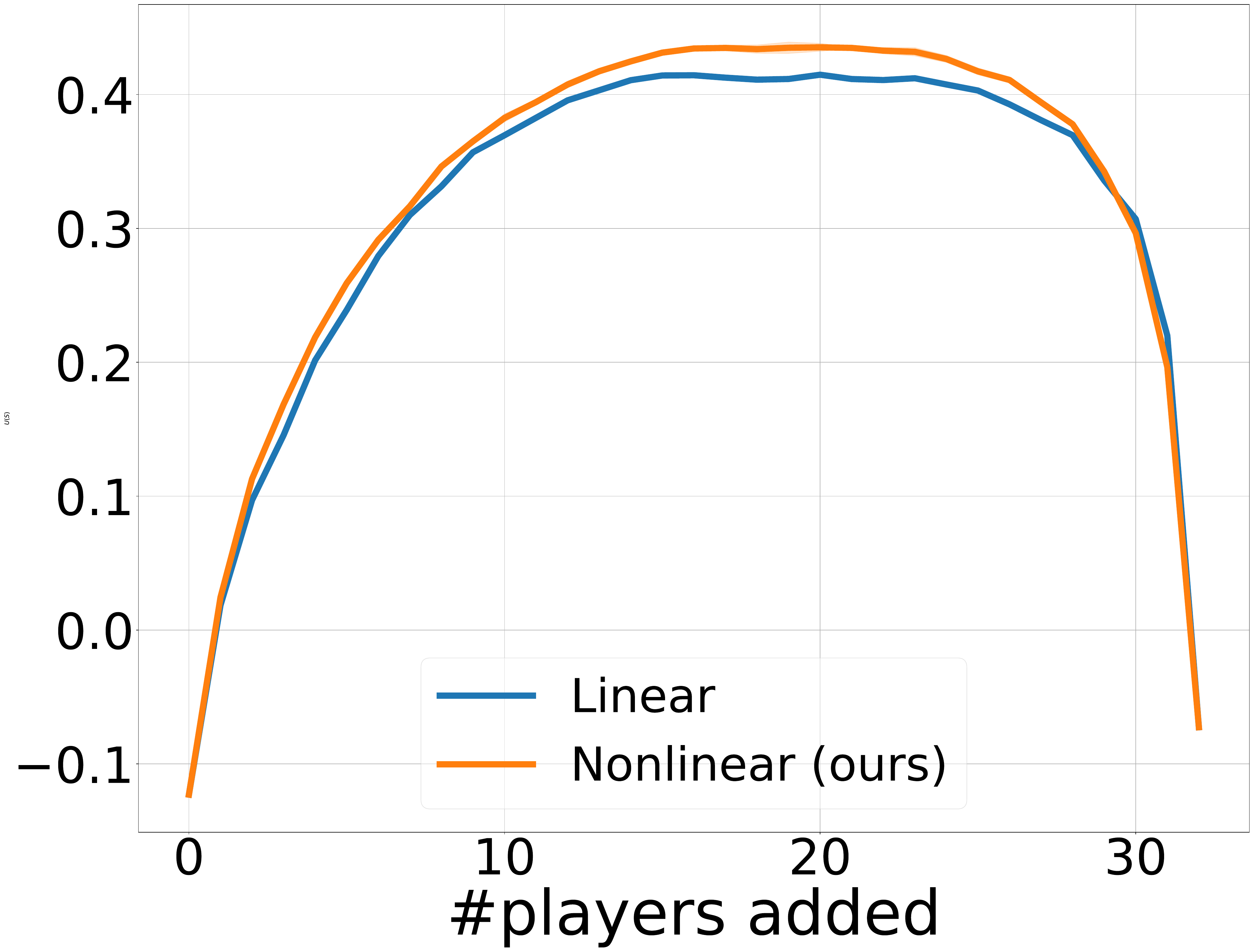} & \includegraphics[width=0.3\linewidth]{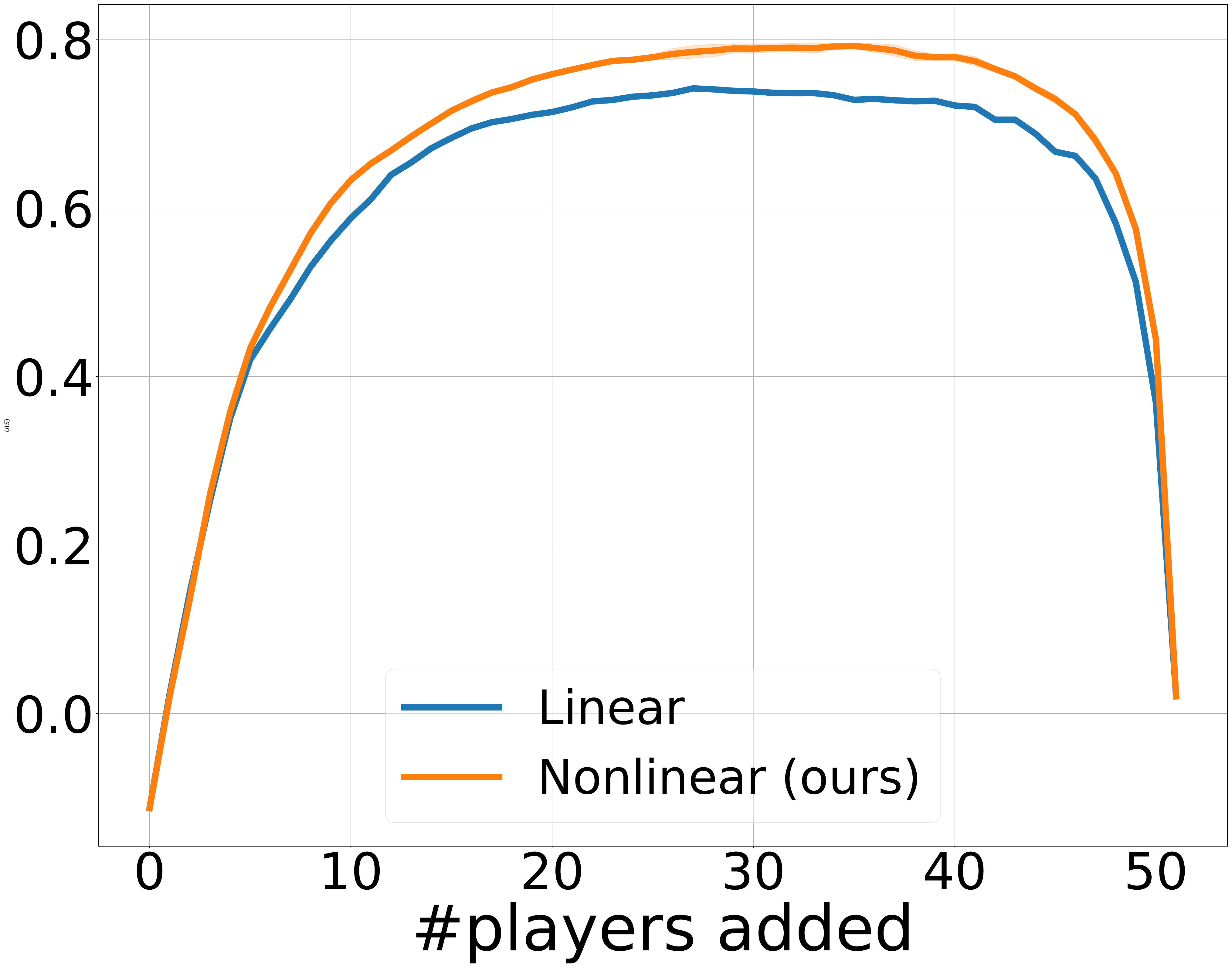} &
		\includegraphics[width=0.3\linewidth]{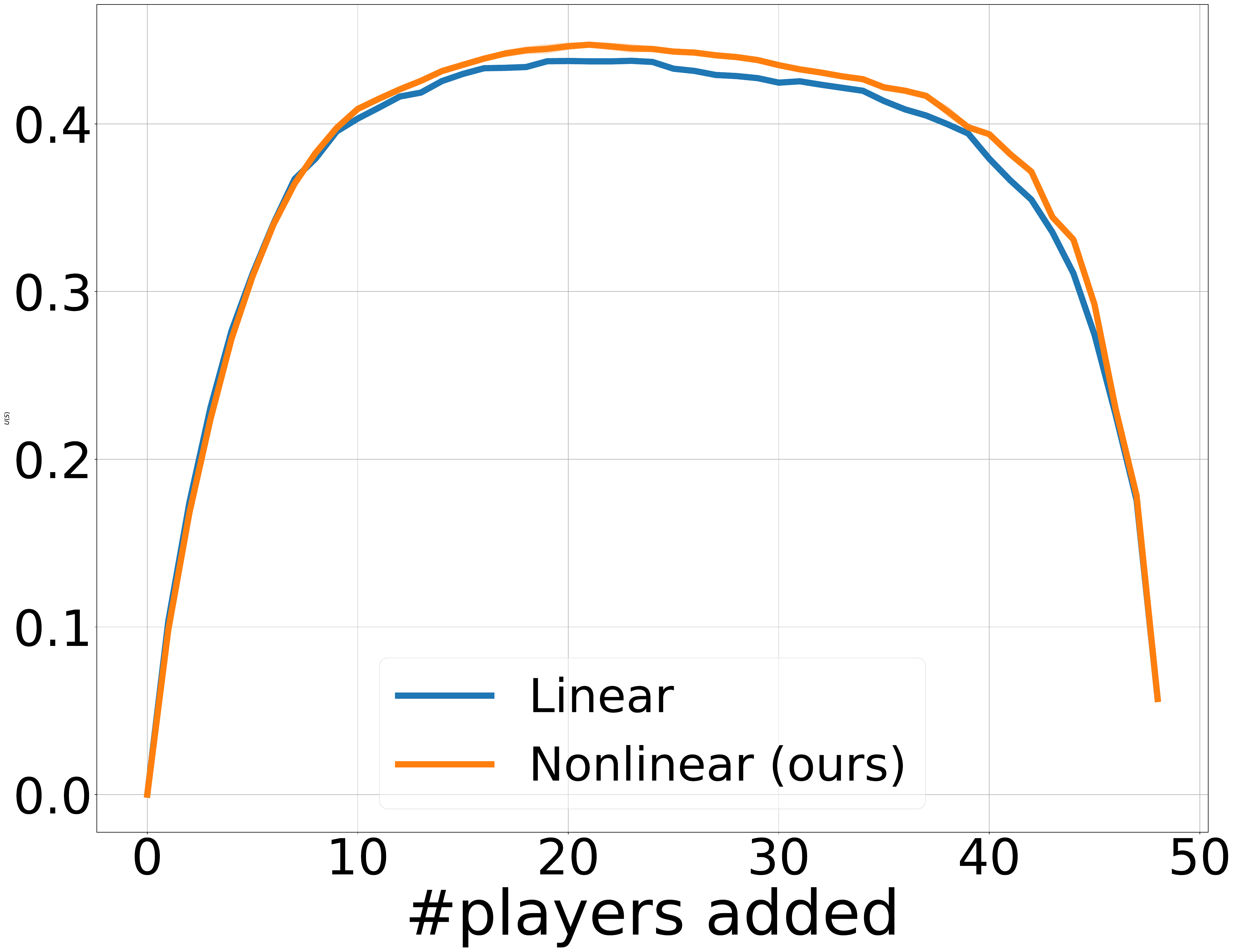} \\
		\phantom{aaaaa}GPSP ($n=32$) & \phantom{aaaaa}FOTP ($n=51$) & \phantom{aaaaaa}wave\_energy ($n=48$)\\
		\includegraphics[width=0.3\linewidth]{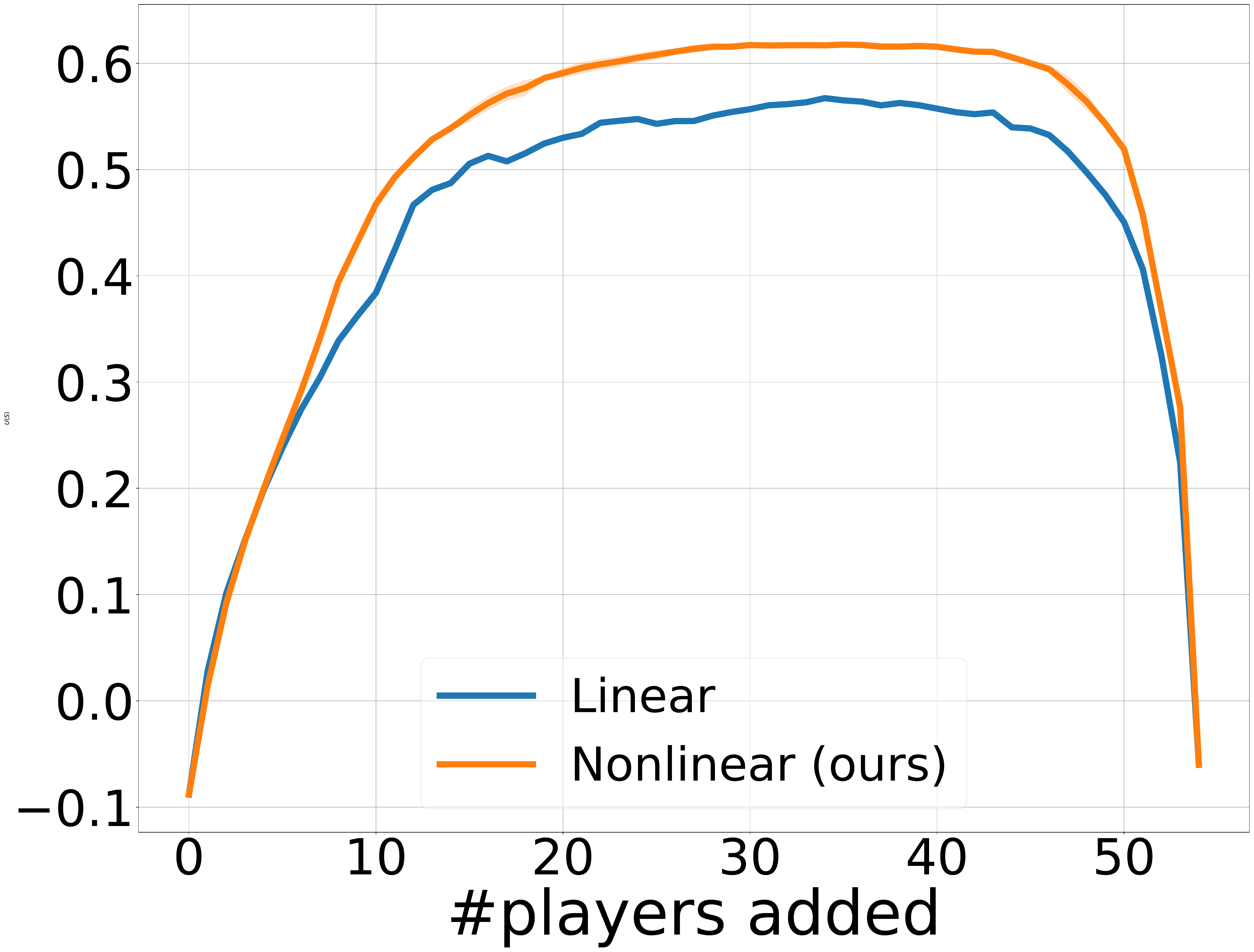} & \includegraphics[width=0.3\linewidth]{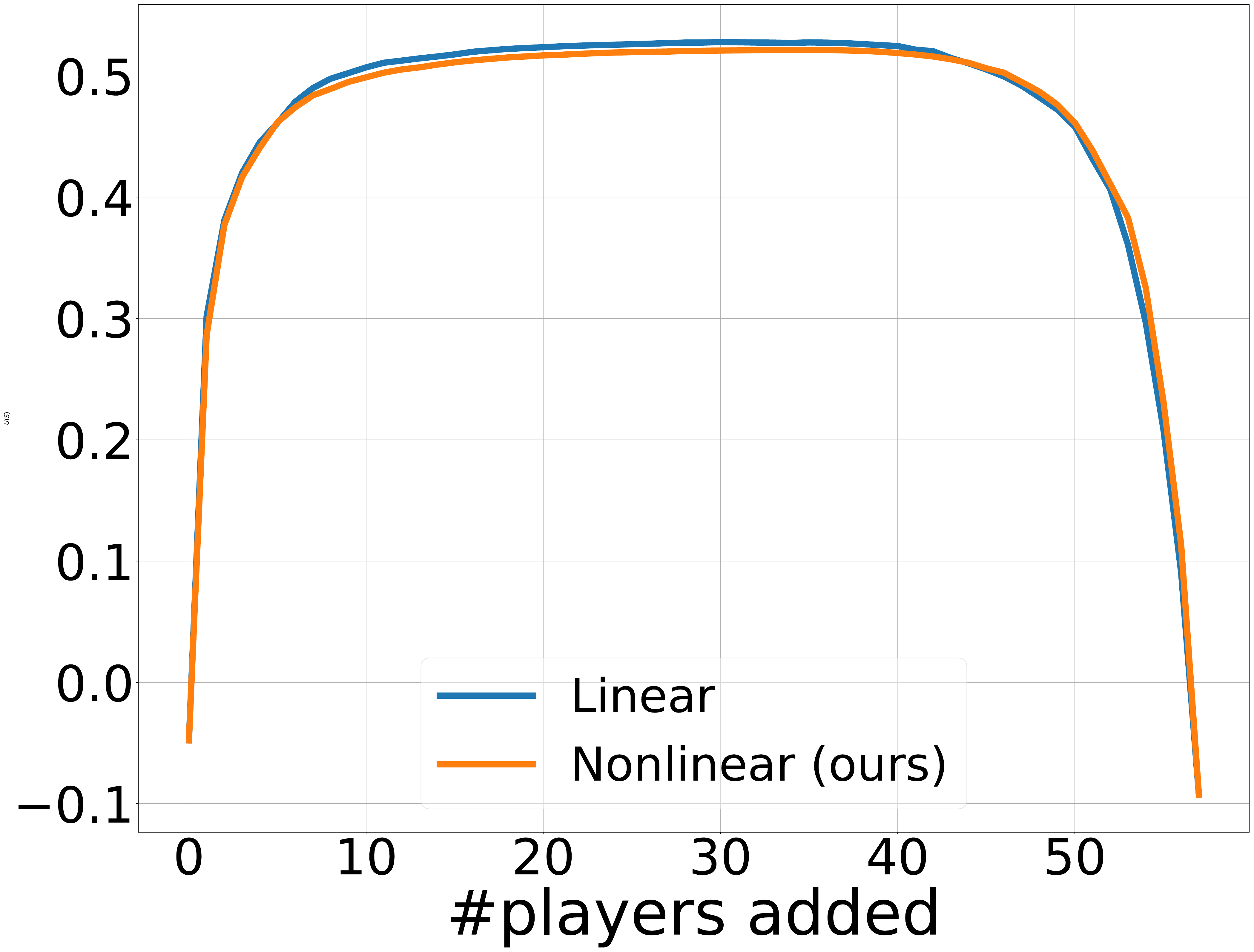} &
		\includegraphics[width=0.3\linewidth]{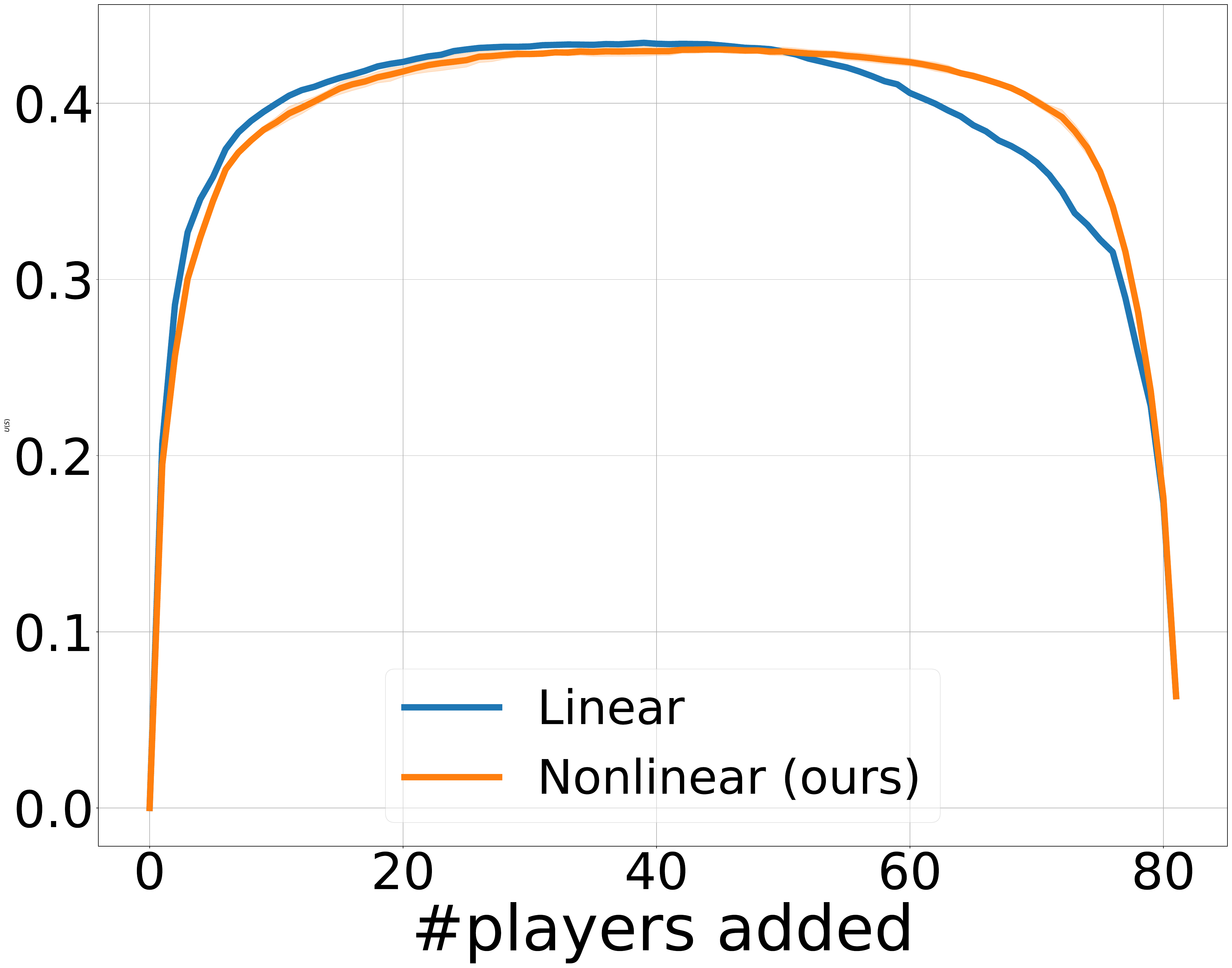} \\
		\phantom{aaaaa}jannis ($n=54$) & \phantom{aaaaa}spambase ($n=57$) & \phantom{aaaaaa}superconduct ($n=81$)
	\end{tabular}
	\caption{Comparison of attribution methods on six datasets for player ranking. A larger area under the curve indicates better performance. Since all utility functions contain more than $30$ players, nonlinear attribution methods are approximated by sampling $\#\mathrm{players}\times 1,000$ subsets, with mean and standard deviation reported over five random seeds. Linear methods are computed exactly. The linear and nonlinear curves shown are the best among $11$ and $6$ candidates respectively under the inclusion AUC metric.}
	\label{fig:appr}
\end{figure*}

\begin{restatable}{proposition}{OneNorm} \label{prop:p=1 is not good}
	The $1$-norm attribution method $\phi^{1}$ does not satisfy the consistency axiom.
\end{restatable}

Such a failure of the $1$-norm attribution method $\phi^{1}$ arises because $\minimize_{b \in \mathbb{R}} \sum_{j=1}^{m} |b - y_{j}|$ does not necessarily yield a unique solution. We provide a concrete counterexample to prove Proposition~\ref{prop:p=1 is not good}.

\begin{proof}
	Define $U_{2} \in \mathcal{G}_{2}$ by letting
	\begin{equation}
		\begin{gathered}
			U_{2}(\emptyset) = 0,\ U_{2}(\{1\}) = 3,\ U_{2}(\{2\}) = 2,\ U_{2}(\{1,2\}) = 1.
		\end{gathered}
	\end{equation}
	Then, $\phi^{1}(U_{2}) = (0, 0)$ as the solution of $b=1.5$ and $\mathbf{x}=(0,0)$ is optimal to the problem
	\begin{equation}
		\begin{gathered}
			\minimize_{\mathbf{x}\in \mathbb{R}^{2}, b\in \mathbb{R}} \sum_{S\subseteq [2]} |U_{2}(S)-b-\sum_{i\in S}x_{i}| .
		\end{gathered}
	\end{equation}
	The optimality is due to that the problem is convex and $\mathbf{0}$ is a subgradient at that point. In particular, $(b, 0, 0)$ with $b \in [1, 2]$ is optimal to the problem.
	Then, we construct a utility function $U_{3} \in \mathcal{G}_{3}$ by letting
	\begin{equation}
		\begin{gathered}
			U_{3}(S) \coloneq \begin{cases}
				U_{2}(S), & 3 \not\in S,\\
				U_{2}(S) + 1, & \text{otherwise.}
			\end{cases}
		\end{gathered}
	\end{equation}
	Note that
	\begin{equation}
		\begin{gathered}
			\min_{\mathbf{x} \in \mathbb{R}^{3}, b\in \mathbb{R}} \sum_{S\subseteq [3]}|U_{3}(S) - b - \sum_{i\in S}x_{i}|\ = 2e .
		\end{gathered}
	\end{equation}
	Then, one can verify that the solution of $b=2$ and $\mathbf{x}=(0,0,0)$ is optimal, which suggests that $\phi^{1}(U_{3}) = (0, 0, 0)$. With this example, it is clear that $\phi^{1}$ violates the axiom of consistency.
\end{proof}

Nevertheless, for the remaining values of $p$, the resulting attribution methods are axiomatically sound.

\begin{restatable}{theorem}{pNorm} \label{thm:p norm is good}
	For $p\in (1, \infty)\cup \{\infty\}$, the $p$-norm attribution method $\phi^{p}$ satisfies the axioms of consistency, equal treatment, monotonicity and translation invariance. When $n>2$, $\phi^{p}$ is linear if and only if $p=2$.
\end{restatable}

\paragraph{Possible Weighted Extensions.}
For $p \in (1, \infty)$, we can define weighted versions of the $p$-norm objective while retaining the axiomatic guarantees of Theorem~\ref{thm:p norm is good}. 
Given $\omega \in (0, 1)$ and $p \in (1, \infty)$, define $\phi^{p, \omega}(U_{n})$ as the unique optimal solution of $\mathbf{x}$ to the problem
\begin{equation}
	\begin{gathered}
		\minimize_{\mathbf{x}\in \mathbb{R}^{n}, b\in \mathbb{R}}\sum_{S \subseteq  [n]} \omega^{|S|}(1-\omega)^{n-|S|}\left|A_{n}(S; \mathbf{x}, b) - U_{n}(S)\right|^{p} .
	\end{gathered}
\end{equation} 
Then, $\phi^{2,\omega}$ is the weighted Banzhaf value parameterized by $\omega$ \citep{marichal2011weighted,li2023robust}. Though we did not find a rigorous
way to prove that $\phi^{p,\omega}$ is linear if and only if $p=2$, we conjecture that this claim holds true.

\begin{figure*}[t]
	\centering
	\begin{tabular}{ccc}
		\includegraphics[width=0.3\linewidth]{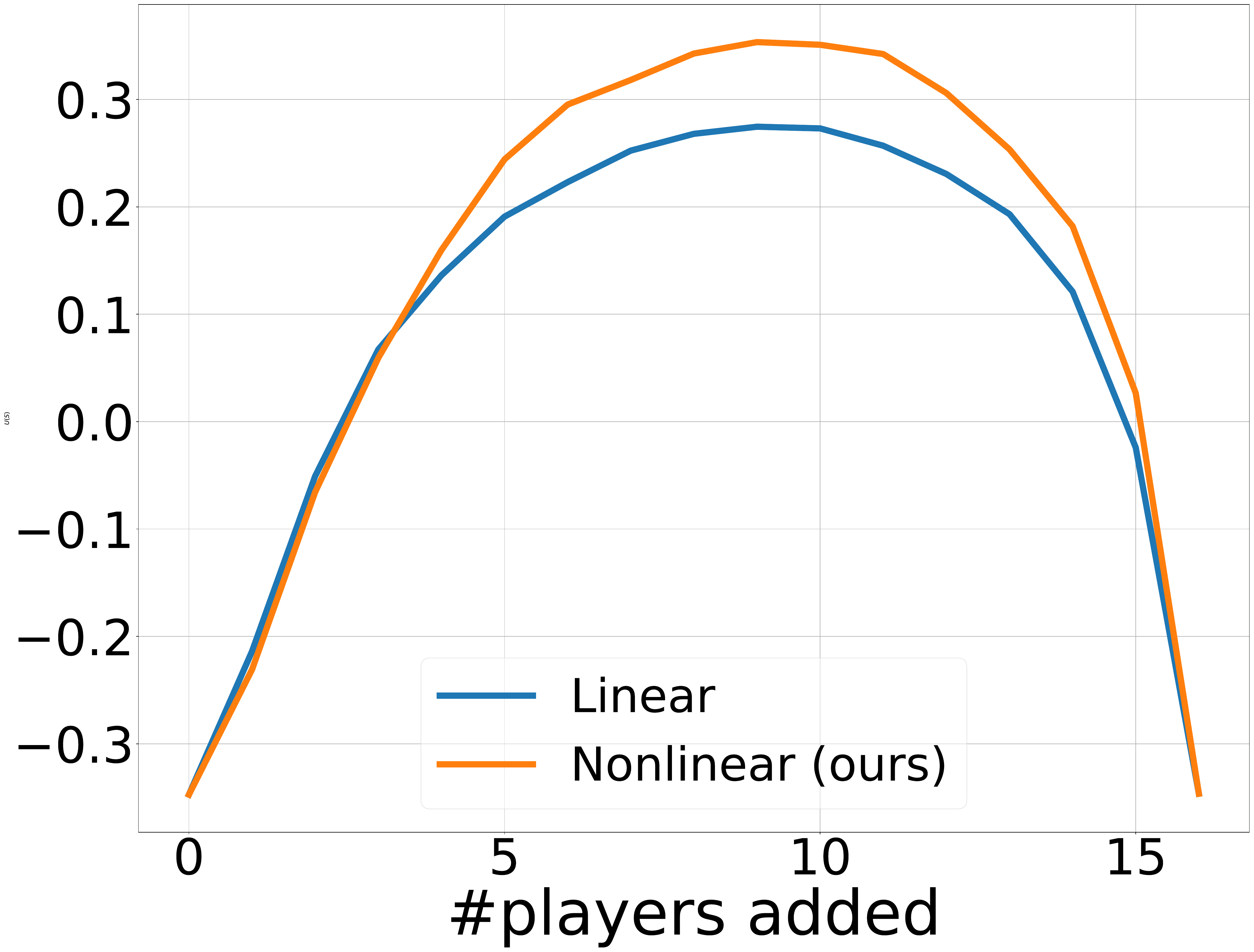} & \includegraphics[width=0.3\linewidth]{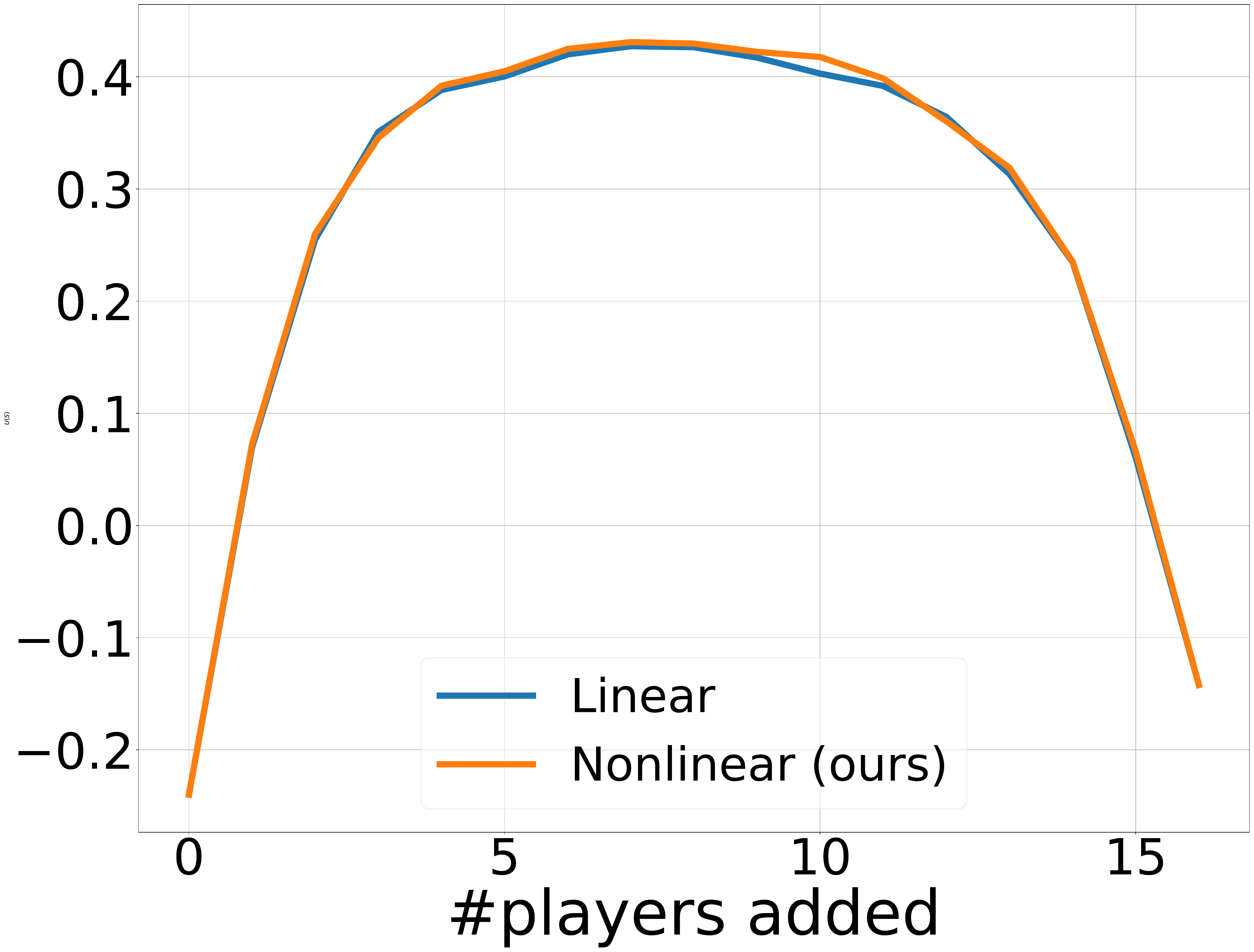} &
		\includegraphics[width=0.3\linewidth]{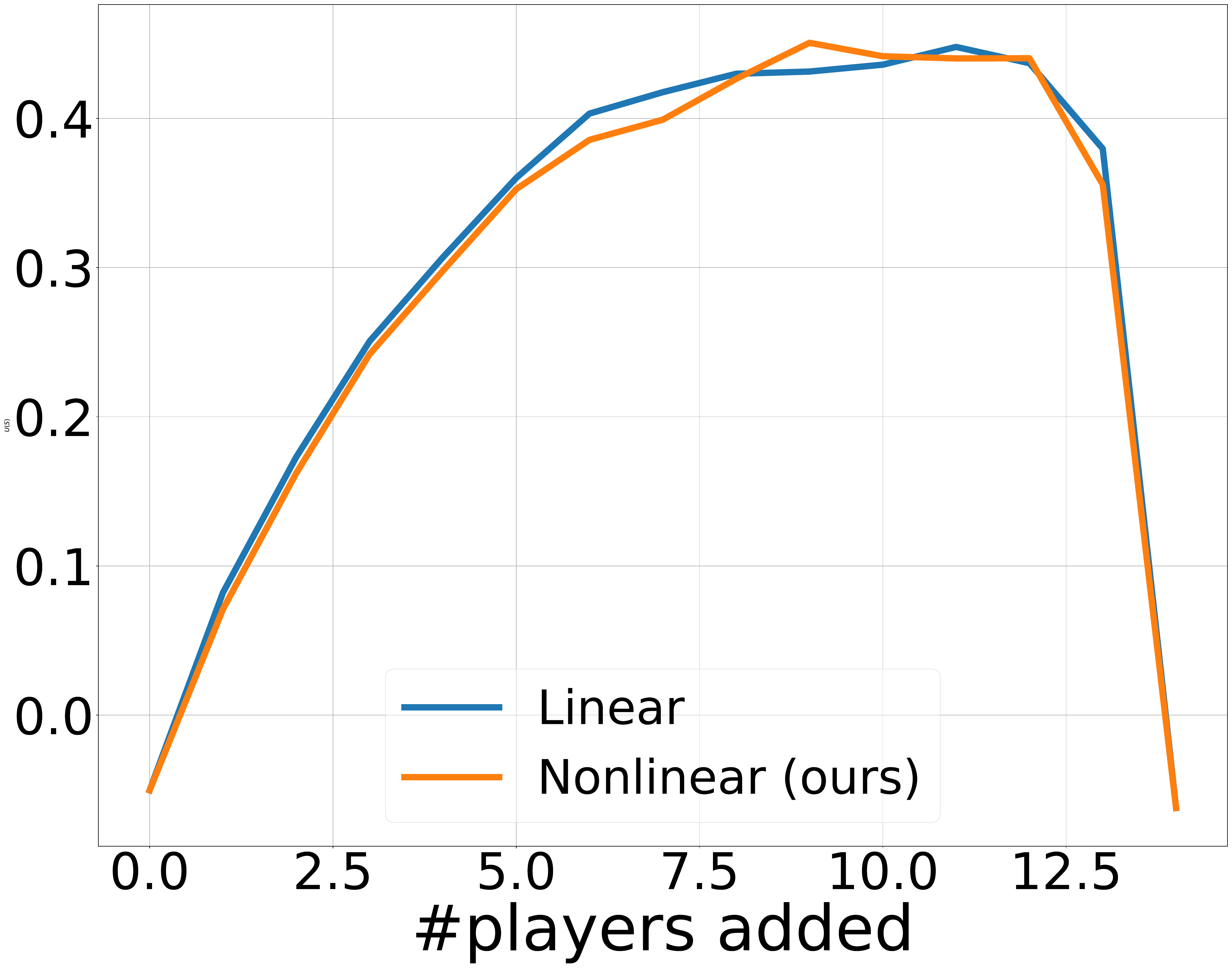} \\
		\phantom{aaaaa}letter ($n=16$) & \phantom{aaaaa}pendigits $(n=16)$ & \phantom{aaaaa}EES $(n=14)$\\
		\includegraphics[width=0.3\linewidth]{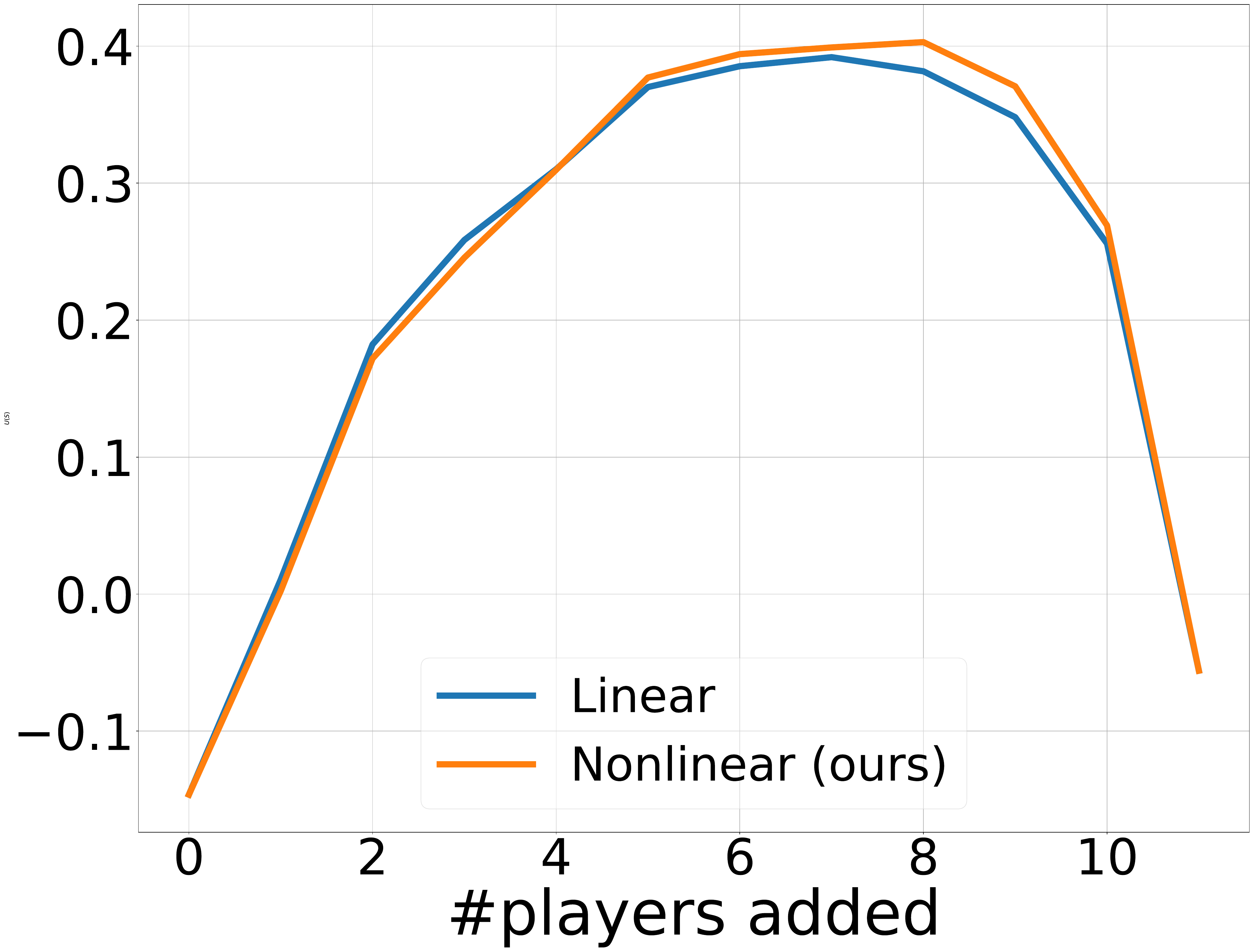} & \includegraphics[width=0.3\linewidth]{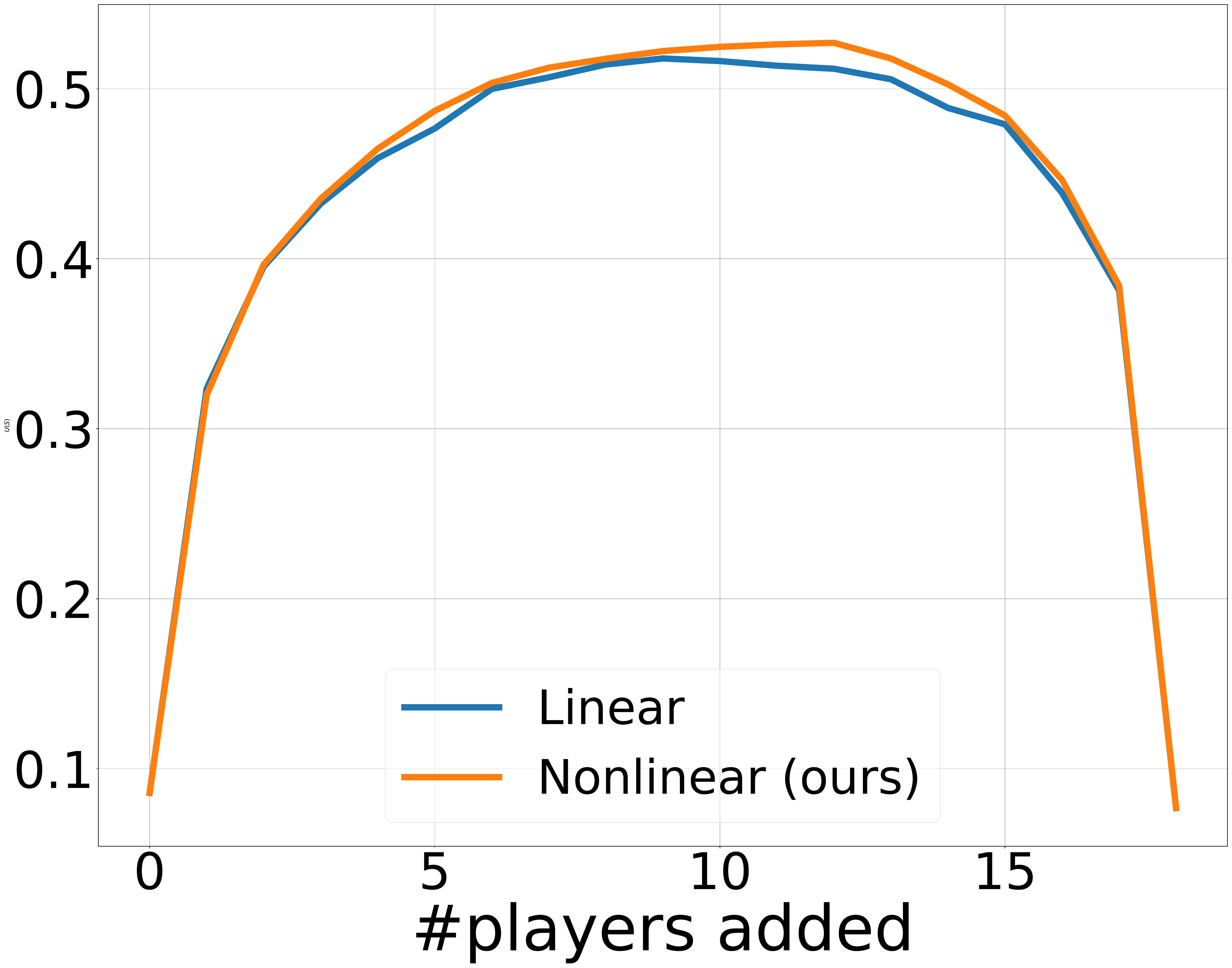} &
		\includegraphics[width=0.3\linewidth]{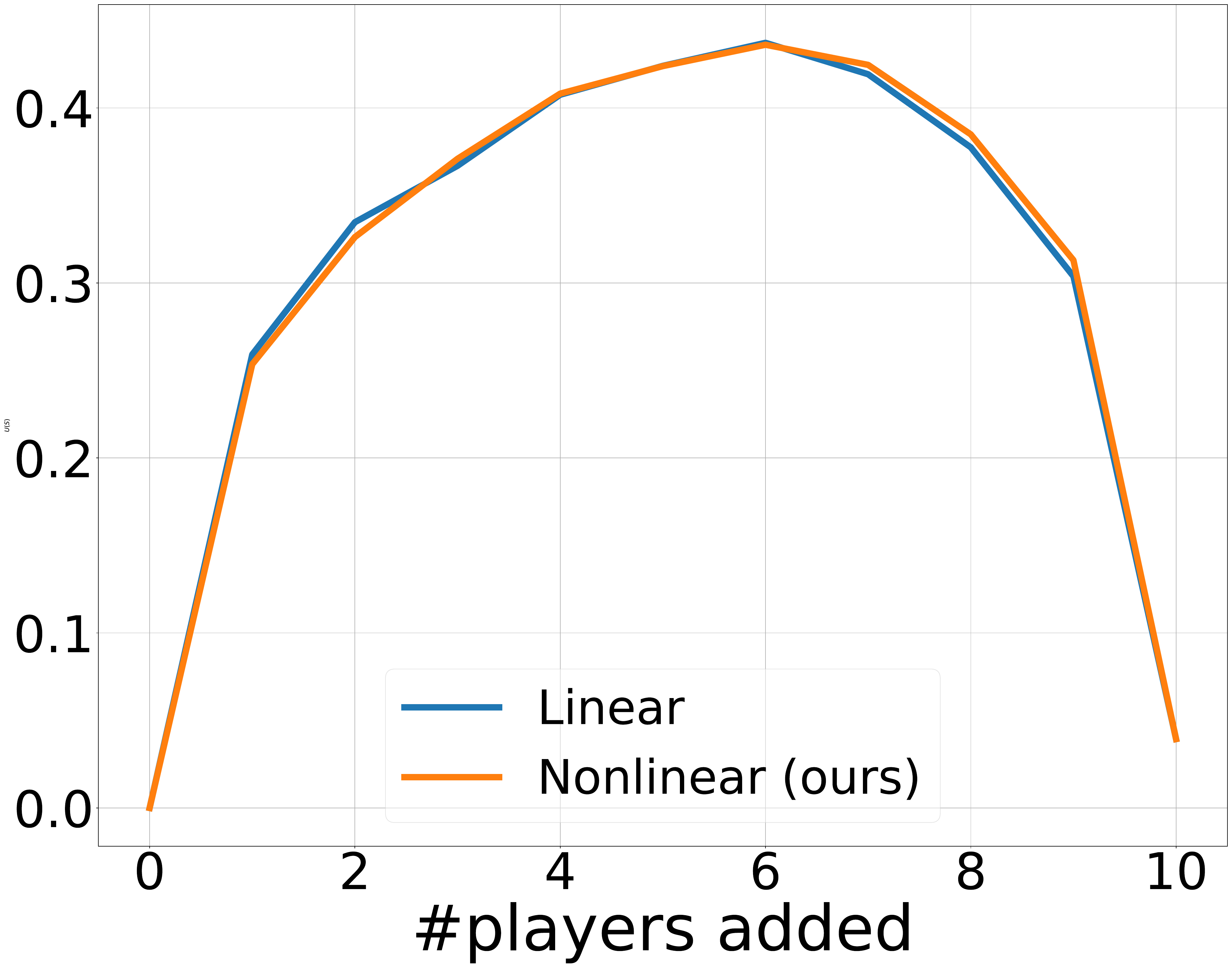} \\
		\phantom{aaaaa}WQW ($n=11$) & \phantom{aaaaa}elevators ($n=18$) & \phantom{aaaaa}credit ($n=10$)
	\end{tabular}
	\caption{Comparison of attribution methods on six datasets for feature ranking. A larger area under the curve indicates better performance. Since all utility functions contain fewer than $20$ players, all subsets are used for nonlinear methods and exact computations are performed for linear methods. The linear and nonlinear curves shown are the best among $11$ and $6$ candidates respectively under the inclusion AUC metric.}
	\label{fig:exact}
\end{figure*}

\begin{corollary}
	For every $\omega \in (0,1)$ and every $p \in (1, \infty)$, the weighted $p$-norm attribution method $\phi^{p,\omega}$ satisfies the axioms of consistency, equal treatment, monotonicity and translation invariance.
\end{corollary}

\subsection{Optimization} 
Note that when $n$ is sufficiently large, all of the mentioned attribution methods can only be approximated, unless the utility function $U_{n}$ exhibits certain favorable properties \citep[e.g.,][]{jia2019efficient,lundberg2020local}.

For $p \in (1, \infty)$, $\phi^{p}(U_{n})$ can be effectively approximated by solving the minimization problem~\eqref{op:p norm}, since its solution is unique. However, this is not the case for $\phi^{\infty}$, as it is defined as the optimal solution with the smallest $\ell_{2}$-norm among potentially many minimizers. To facilitate the approximation of $\phi^{\infty}$, we establish the following theoretical result.

\begin{restatable}{theorem}{SolveInfinity} \label{thm:sovleinfinity}
	For every $\eta>0$, let $\phi^{\infty}(U_{n}; \eta)$ denote the $\mathbf{x}$-part of an optimal solution to 
	\begin{equation}
		\begin{gathered}
			\minimize_{\mathbf{x}\in \mathbb{R}^{n}, b\in \mathbb{R}}\  \|\mathbf{Ax} + b\cdot\mathbf{1}_{2^{n}} - \mathbf{V}(U_{n})\|_{\infty} + \eta\cdot\|\mathbf{x}\|_{2}^{2} .
		\end{gathered}
	\end{equation}
	Then, we have $\phi^{\infty}(U_{n}; \eta) \rightarrow \phi^{\infty}(U_{n})$ as $\eta \rightarrow 0$.
\end{restatable}

The key to proving this result is that the constraint $\|\mathbf{A}\mathbf{x} - b\cdot\mathbf{1}_{n} - \mathbf{V}(U_{n})\|_{\infty} \leq e$ implies boundedness of $(\mathbf{x}, b)$. While previous studies have focused on approximating the least core, it remains unclear how to approximate the egalitarian least core $\phi^{\mathrm{ELC}}$. Nevertheless, the regularization technique used for approximating $\phi^{\infty}$ can be extended to $\phi^{\mathrm{ELC}}$ as well.

\begin{restatable}{theorem}{SolveELC} \label{thm:solveELC}
	For every $\eta>0$, let $\phi^{\mathrm{ELC}}(U_{n}; \eta)$ denote the $\mathbf{x}$-part of an optimal solution to
	\begin{equation}
		\begin{gathered}
			\minimize_{\mathbf{x}\in \mathbb{R}^{n}, e\in \mathbb{R}}\ e + \eta\cdot\|\mathbf{x}\|_{2}^{2} \\
			\text{subject to }\ U_{n}(S) - A_{n}(S; \mathbf{x}, U_{n}(\emptyset)) \leq e \ \text{ for every } S \subsetneq [n] \\ 
			\text{and }\ A_{n}([n]; \mathbf{x}, U_{n}(\emptyset)) = U_{n}([n]). 
		\end{gathered}
	\end{equation}
	Then, we have $\phi^{\mathrm{ELC}}(U_{n}; \eta) \rightarrow \phi^{\mathrm{ELC}}(U_{n})$ as $\eta \rightarrow 0$.
\end{restatable}

Accordingly, we propose Algorithm~\ref{alg:appr} to approximate the attribution methods of interest when exact computation is infeasible. Specifically, all the minimization problems we solve are convex, and thus off-the-shelf convex solvers can be directly applied.

%% file: files/experiments.tex
\begin{table*}[t]
	\centering
	\caption{Average inclusion AUC scores of various attribution methods. The best results are shown in bold, and the second-best results are underlined. Note that Beta Shapley reports the best one over $10$ candidates, as described in~\eqref{eq:weightedshap}.}
	\label{tab:auc}
	\resizebox{\textwidth}{!}{
		\begin{tabular}{lrrrrrrrr}
			\toprule
			dataset & $1.5$-norm& $2.5$-norm & $5$-norm & $10$-norm & $\infty$-norm & ELC & Banzhaf & Beta Shapley\\ \midrule
			GPSP & $10.11$ & $10.45$ & $\underline{10.77}$ & $\textbf{10.82}$ & $10.37$ & $10.19$ & $10.32$ & $10.28$\\
			FOTP & $30.72$ & $31.94$ & $\underline{33.50}$ & $\textbf{33.68}$ & $31.97$ & $30.08$ & $31.52$ & $30.56$\\
			wave\_energy & $17.63$ & $\underline{17.88}$ & $\textbf{18.09}$ & $17.78$ & $16.90$ & $17.16$ & $17.77$ & $17.65$\\
			jannis & $21.70$ & $25.83$ & $\textbf{27.27}$ & $\underline{27.21}$ & $25.23$ & $25.19$ & $24.53$ & $24.21$\\
			spambase & $26.37$ & $26.67$ & $\underline{26.74}$ & $26.63$ & $26.07$ & $26.14$ & $26.67$ & $\textbf{26.83}$\\
			superconduct & $31.39$ & $31.62$ & $\underline{31.67}$ & $31.50$ & $30.35$ & $\textbf{31.95}$ & $31.58$ & $31.64$\\
			letter & $1.62$ & $2.01$ & $2.48$ & $\textbf{2.59}$ & $\underline{2.50}$ & $2.14$ & $1.85$ & $1.44$\\
			pendigits & $4.69$ & $4.81$ & $\textbf{4.84}$ & $4.81$ & $4.71$ & $\underline{4.82}$ & $4.78$ & $4.52$\\
			EES & $4.26$ & $4.38$ & $\underline{4.40}$ & $4.26$ & $4.18$ & $4.06$ & $4.35$ & $\textbf{4.44}$\\
			WQW & $2.76$ & $\underline{2.86}$ & $\textbf{2.89}$ & $2.85$ & $2.79$ & $2.79$ & $2.84$ & $2.68$\\
			elevators & $8.03$ & $8.03$ & $8.04$ & $8.03$ & $7.95$ & $\textbf{8.15}$ & $8.01$ & $\underline{8.12}$\\
			credit & $\underline{3.37}$ & $\textbf{3.38}$ & $3.29$ & $3.17$ & $3.07$ & $3.13$ & $\underline{3.37}$ & $3.36$\\

			\bottomrule
	\end{tabular}}
\end{table*}

\section{Experiments} \label{sec:experiments}
The goal of our experiments is to evaluate whether the nonlinear attribution methods introduced in this work outperform semi-values in ranking players.  According to Theorem~\ref{thm:p norm is good}, the considered nonlinear attribution methods include $\phi^{p}$ with $p \in \{1.5, 2.5, 5, 10, \infty\}$ as well as $\phi^{\mathrm{ELC}}$. We use the CVXPY library with the CLARABEL solver to solve the minimization problems in Algorithm~\ref{alg:appr} \citep{diamond2016cvxpy,agrawal2018rewriting}.

\paragraph{Utility functions.} 
The utility function we employ is $U_{f}^{\mathbf{x}}(S) \coloneq \mathbb{E}_{\mathbf{X}_{[n]\setminus S}}[f(\mathbf{x}_{S}, \mathbf{X}_{[n]\setminus S})]$, where $f: \mathbb{R}^{n}\to \mathbb{R}$ is a trained gradient boosted tree (GBT) with five estimators, $\mathbf{x} \in \mathbb{R}^{n}$ is an input instance, and $\mathbf{x}_{S}$ is the restriction of $\mathbf{x}$ on $S$. Then, the number of players corresponds to the number of features.
We follow the path-dependent definition proposed by \citet[Algorithm 1]{lundberg2020local}. For this type of utility functions, the corresponding semi-values can be computed in polynomial time \citep{karczmarz2022improved,Yu2022linear,muschalik2024beyond}, and we use the numerically stable version developed by \citet{li2026treegrad}. In general, $\phi^{\mu}(f_{\mathbf{x}})$ can only be approximated, and our rationale is to avoid such approximation error. All GBTs are trained using the scikit-learn library \citep{pedregosa2011scikit}.
We train one GBT for each dataset. 
To account for the variability of $U_{f}^{\mathbf{x}}$, for each trained GBT $f$, the reported curve is averaged over $100$ different instances of $\mathbf{x}$. For classification datasets, $f$ outputs the logit of a randomly selected class.

\paragraph{Task.} 
The task is to rank players, where performance is evaluated using the inclusion AUC metric defined in problem~\eqref{op:auc}. A higher value indicates better performance. 
According to Lemma~\ref{lem:liearized AUC}, the player ranking $\pi$ induced by a specific contribution vector $\boldsymbol\phi$ satisfies that $\phi_{\pi(j)} \geq \phi_{\pi(j+1)}$ for every $1\leq j < n$. 
Intuitively, $U_{f}^{\mathbf{x}}$ is supposed to increase sharply after including the players that contribute the most.

\paragraph{Datasets.} 
We employ $12$ datasets in total, which are summarized in Table~\ref{tab:sum}. In particular, the datasets we selected are equally divided into two group. Each dataset in the first group contains fewer than 20 features, whereas those in the second group contain more than 30 features. For decision trees trained using the first group of datasets, we solve the exact minimization problem for each considered nonlinear attribution method. For the other group of datasets, the budget $B$ in Algorithm~\ref{alg:appr} is set to $\#\mathrm{players} \times 1,000$. For the two regression datasets, we scale the predicted values for better presentation.

\paragraph{Baselines.} We compare the selected nonlinear attribution methods to Beta Shapley values \citep{kwon2022beta}. 
%It evaluates a list of finitely many semi-values, and then select the one that achieves the highest inclusion AUC. Intuitively, this may partially alleviate the issue introduced by the linearity axiom.
The considered candidates include Beta$(\alpha,\beta)$ with
\begin{equation} \label{eq:weightedshap}
	\begin{aligned}
		(\alpha, \beta) \in \{&(16,1), (8,1), (4,1), (2,1), (1,1), \\
		&(1,2), (1,4), (1,8), (1,16), (1,32)\},
	\end{aligned}
\end{equation}
which were used in \citep{kwon2022weightedshap}.
Meanwhile, we also compare with another well-known semi-value, the Banzhaf value \citep{banzhaf1965weighted}. Therefore, a total of $11$ linear attribution methods are used.

For reproducibility, we fix the random seed for training GBTs and selecting logits. When performing approximation, we average the results over five different random seeds and report the mean along with the corresponding standard deviation.
More experimental details and statistical results can be found in Appendix~\ref{app:exp}.

\subsection{Empirical Observations}
Results for utility functions containing fewer than $20$ players are shown in Figure~\ref{fig:exact}, while those for datasets with more than $30$ players are presented in Figure~\ref{fig:appr}. 
For clarity, we plot only the inclusion curves of the best linear and nonlinear candidates that maximize the inclusion AUC. In Appendix~\ref{app:exp}, we also plot the inclusion curves for all employed nonlinear attribution methods.
The average inclusion AUC scores are reported in Table~\ref{tab:auc}.
We make the following observations: (i) nonlinear attribution methods tend to outperform linear attribution methods under the inclusion AUC metric; and (ii) no single attribution method consistently outperforms others across all datasets.

%% file: files/conclusion.tex
\section{Conclusion}
\label{sec:con}
While the Shapley value is widely adopted for its axiomatic properties, we demonstrate that its linearity in utility functions can lead to unreliable/undesirable player ranking under the inclusion AUC metric. Moreover, the efficiency axiom has also been empirically proved unfavorable. Accordingly, we introduce a class of attribution methods by relaxing the linearity and efficiency axioms while preserving other desirable axioms such as consistency, equal treatment, monotonicity, and translation invariance.
In particular, these introduced nonlinear attribution methods are based on faithful additive approximation of utility functions.
%Experiments across different utility functions confirm that nonlinear attribution methods tend to outperform linear ones in player ranking. 
We view our work as the first step to establish nonlinear axiomatic attribution methods through the lens of additive approximation.

%% file: files/appendix.tex
\section{Missing Proofs}
\label{sec:proofs}
\AxiomOfEgalitarian*
\begin{proof}
	\textbf{Efficiency}. This is imposed by the constraint $\sum_{i \in [n]} x_{i} = U_{n}([n]) - U_{n}(\emptyset)$ that appears in the problem~\eqref{op:one-sided least core}.
	
	\textbf{Translation Invariance}. This is clear as it is $\{U_{n}(S_{1}) - U_{n}(S_{2})\}$ that matters in the problem~\eqref{op:one-sided least core}.
	
	\textbf{Equal Treatment}. Suppose there exist $i, j \in [n]$ such that $U_{n}(S\cup \{i\}) = U_{n}(S\cup \{j\})$ for every $S\subseteq [n]\setminus\{i,j\}$. For the sake of contradiction, assume $\phi^{\mathrm{ELC}}_{i}(U_{n}) > \phi^{\mathrm{ELC}}_{j}(U_{n}) $.
	Define a contribution vector $\boldsymbol\phi\in \mathbb{R}^{n}$ by letting
	\begin{equation}
		\begin{gathered}
			\phi_{k} \coloneq \begin{cases}
				\phi^{\mathrm{ELC}}_{k}(U_{n}), & k \in [n]\setminus\{i,j\}\\
				\frac{\phi_{i}^{\mathrm{ELC}}(U_{n}) + \phi_{j}^{\mathrm{ELC}}(U_{n})}{2}, & \text{otherwise.}
			\end{cases}
		\end{gathered}
	\end{equation}
	Let
	\begin{equation}
		\begin{gathered}
			F(S)\coloneq U_{n}(S\cup \{i\}) - U_{n}(\emptyset) - \sum_{k\in S}\phi_{k}^{\mathrm{ELC}}(U_{n}) .
		\end{gathered}
	\end{equation}
	Since 
	\begin{equation}
		\begin{gathered}
			F(S) - \phi_{i} \leq \max(F(S)-\phi_{i}^{\mathrm{ELC}}(U_{n}),\  F(S)-\phi_{j}^{\mathrm{ELC}}(U_{n}),\ 0)
		\end{gathered}
	\end{equation}
	for every $S \subseteq  [n]\setminus\{i, j\}$,
	$\boldsymbol\phi$ with $e = \max_{S\subseteq  [n]} U_{n}(S) - \sum_{i\in S} \phi_{i}^{\mathrm{ELC}}(U_{n})$ is another optimal solution to the problem~\eqref{op:one-sided least core} . However, 
	\begin{equation}
		\begin{gathered}
			2\times \left( \frac{\phi_{i}^{\mathrm{ELC}}(U_{n}) + \phi_{j}^{\mathrm{ELC}}(U_{n})}{2} \right)^{2} < \phi_{i}^{\mathrm{ELC}}(U_{n})^{2} + \phi_{j}^{\mathrm{ELC}}(U_{n})^{2} , 
		\end{gathered}
	\end{equation}
	a contradiction to that $\phi^{\mathrm{ELC}}(U_{n})$ is the least-norm optimal solution.
	
	\textbf{Consistency}. Suppose there exists $C\in \mathbb{R}$ such that $U_{n+1}(S) = U_{n}(S\cap [n]) + C\cdot\mathds{1}_{n+1 \in S}$ for every $S \subseteq  [n+1]$. First, we prove that $\phi_{n+1}^{\mathrm{ELC}}(U_{n+1}) \geq C$. 
	For convenience, we write
	\begin{equation}
		\begin{gathered}
			d(S) \coloneqq U_{n+1}(S) - U_{n+1}(\emptyset) - \sum_{i\in S} \phi_{i}^{\mathrm{ELC}}(U_{n+1})\ \text{ for every }\ S \subseteq  [n+1]\\
			\text{and }\ e = \max_{S\subseteq [n]} d(S) \geq C-U_{n+1}^{\mathrm{ELC}}(U_{n+1}) > 0.
		\end{gathered}
	\end{equation}
	For the sake of contradiction, assume $\phi_{n+1}^{\mathrm{ELC}}(U_{n+1}) < C$. 
	Then, we have
	\begin{equation}
		\begin{gathered}
			d(S) < d(S) + (C - U_{n+1}^{\mathrm{ELC}}(U_{n+1})) = d(S\cup \{n+1\}) \ \text{ for every }\ S\subseteq [n] .
		\end{gathered}
	\end{equation}
	It indicates that
	\begin{equation}
		\begin{gathered}
			d(S) = e \implies n+1 \in S .
		\end{gathered}
	\end{equation}
	Define a contribution $\boldsymbol\phi \in \mathbb{R}^{n+1}$ by letting
	\begin{equation}
		\begin{gathered}
			\phi_{k} \coloneq \begin{cases}
				\phi_{k}^{\mathrm{ELC}}(U_{n+1}) + \frac{C - \phi_{n+1}^{\mathrm{ELC}}(U_{n+1})}{3}, & k = n+1,\\[3pt]
				\phi_{k}^{\mathrm{ELC}}(U_{n+1}) - \frac{C - \phi_{n+1}^{\mathrm{ELC}}(U_{n+1})}{3n}, & \text{otherwise.} 
			\end{cases}
		\end{gathered}
	\end{equation}
	Then, $\boldsymbol\phi$ satisfies the efficiency constraint, and for every $S \subseteq  [n]$
	\begin{equation}
		\begin{aligned}
			U_{n+1}(S) - U_{n+1}(\emptyset) - \sum_{i\in S}\phi_{i} &\leq d_{n+1}(S) + \frac{C - \phi_{n+1}^{\mathrm{ELC}}(U_{n+1})}{3} \\
			&< d_{n+1}(S\cup \{n+1\}) - \frac{C - \phi_{n+1}^{\mathrm{ELC}}(U_{n+1})}{3}\\
			&\leq U_{n+1}(S\cup \{n+1\}) - U_{n+1}(\emptyset) - \sum_{i\in S\cup \{n+1\}} \phi_{i} .
		\end{aligned}
	\end{equation}
	In particular,
	\begin{equation}
		\begin{gathered}
			U_{n+1}(S\cup \{n+1\}) - U_{n+1}(\emptyset) - \sum_{i\in S\cup \{n+1\}} \phi_{i} < d(S\cup \{n+1\})
		\end{gathered}
	\end{equation}
	for every $S \subsetneq [n]$,
	which contradicts the optimality of $(\phi^{\mathrm{ELC}}(U_{n+1}),\,e)$. Therefore, we have $\phi_{n+1}^{\mathrm{ELC}}(U_{n+1}) \geq C$.
	
	To have $\phi_{n+1}^{\mathrm{ELC}}(U_{n+1})\leq C$, the same argument applies by noticing that
	\begin{equation}
		\begin{gathered}
			e \geq d([n]) > d([n+1]) = 0 
		\end{gathered}
	\end{equation}
	if we assume $\phi_{n+1}^{\mathrm{ELC}}(U_{n+1}) > C$.
	As a result, we have $\phi_{n+1}^{\mathrm{ELC}}(U_{n+1}) = C$.
	With a moment's thought, $\phi_{n+1}^{\mathrm{ELC}}(U_{n+1}) = C$ implies
	\begin{equation}
		\begin{gathered}
			\phi_{k}^{\mathrm{ELC}}(U_{n+1}) = \phi_{k}^{\mathrm{ELC}}(U_{n}) \ \text{ for every }\ k \in [n] .
		\end{gathered}
	\end{equation}
	
	\textbf{Monotonicity}.
	Suppose there exists $j \in [n]$ such that $U_{n}(S\cup \{j\}) \geq U_{n}(S)$ for every $S\subseteq [n]\setminus\{j\}$. 
	Assume, for the sake of contradiction, that $\phi_{j}^{\mathrm{ELC}} < 0$. Then,
	\begin{equation}
		\begin{gathered}
			r \coloneq \max_{S\subseteq [n]}U_{n}(S)-U_{n}(\emptyset)-\sum_{i\in S}\phi_{i}^{\mathrm{ELC}}(U_{n}) \geq U_{n}(\{j\}) - U_{n}(\emptyset) - \phi_{j}^{\mathrm{ELC}}(U_{n}) > 0,
		\end{gathered}
	\end{equation}	
	and, for every $S \subseteq  [n]\setminus\{j\}$,
	\begin{equation}
		\begin{gathered}
			U_{n}(S) - U_{n}(\emptyset) - \sum_{i\in S}\phi^{\mathrm{ELC}}_{i}(U_{n})< U_{n}(S\cup \{j\})-U_{n}(\emptyset) - \sum_{i\in S\cup \{j\}}\phi_{i}^{\mathrm{ELC}}(U_{n}) ,
		\end{gathered}
	\end{equation}
	which indicates
	\begin{equation}
		\begin{gathered}
			U_{n}(S) - U_{n}(\emptyset) - \sum_{i\in S}\phi_{i}^{\mathrm{ELC}}(U_{n}) = r \implies j \in S .
		\end{gathered}
	\end{equation}
	Select $\delta > 0$ such that
	\begin{equation}
		\begin{gathered}
			\delta < \min_{S\subseteq [n]\setminus\{j\}} U_{n}(S\cup \{j\}) - U_{n}(S) - \phi_{j}^{\mathrm{ELC}}(U_{n}) .
		\end{gathered}
	\end{equation}
	Define $\boldsymbol\phi \in \mathbb{R}^{n}$ by letting
	\begin{equation}
		\begin{gathered}
			\phi_{k} \coloneq \begin{cases}
				\phi_{k}^{\mathrm{ELC}}(U_{n}) + \frac{\delta}{3}, & k = j\\[3pt]
				\phi_{k}^{\mathrm{ELC}}(U_{n}) - \frac{\delta}{3(n-1)}, & \text{otherwise.}
			\end{cases}
		\end{gathered}
	\end{equation}
	Observe that $\boldsymbol\phi$ satisfies the efficiency constraint, and for every $S\subseteq [n]\setminus\{j\}$
	\begin{equation}
		\begin{aligned}
			U_{n}(S) - U_{n}(\emptyset) - \sum_{i\in S}\phi_{i} &\leq U_{n}(S) - U_{n}(\emptyset) - \sum_{i\in S}\phi_{i}^{\mathrm{ELC}}(U_{n}) + \frac{\delta}{3}\\
			&< U_{n}(S\cup \{j\}) - U_{n}(\emptyset) - \sum_{i\in S\cup \{j\}}\phi_{i}^{\mathrm{ELC}}(U_{n}) - \frac{\delta}{3}\\
			&\leq U_{n}(S\cup \{j\}) - U_{n}(\emptyset) - \sum_{i\in S\cup \{j\}}\phi_{i} . 
		\end{aligned}
	\end{equation}
	Specifically,
	\begin{equation}
		\begin{gathered}
			U_{n}(S\cup \{j\}) - U_{n}(\emptyset) - \sum_{S\cup \{j\}} \phi_{i} < U_{n}(S\cup \{j\}) - U_{n}(\emptyset) - \sum_{S\cup \{j\}}\phi_{i}^{\mathrm{ELC}}(U_{n})
		\end{gathered}
	\end{equation}
	for every $S \subsetneq  [n]\setminus\{j\}$, contradicting the optimality of $(\phi^{\mathrm{ELC}}(U_{n}),\,r)$. Consequently, we have $\phi_{j}^{\mathrm{ELC}}(U_{n}) \geq 0$. 
	
	On the other hand, suppose there exists $j \in [n]$ such that $U_{n}(S\cup \{j\}) \leq U_{n}(S)$ for every $S\subseteq[n]\setminus\{j\}$. Then, $\phi_{j}^{\mathrm{ELC}}(U_{n}) \leq 0$ can be proved similarly by noting that
	\begin{equation}
		\begin{gathered}
			r \geq U_{n}([n]\setminus\{j\}) - U_{n}(\emptyset) - \sum_{i \in [n]\setminus\{j\}}\phi_{i}^{\mathrm{ELC}}(U_{n}) > U_{n}([n]) - U_{n}(\emptyset) - \sum_{i\in [n]}\phi_{i}^{\mathrm{ELC}}(U_{n}) = 0
		\end{gathered}
	\end{equation}
	if we assume $\phi_{j}^{\mathrm{ELC}}(U_{n}) > 0$.

	\textbf{Non-Linearity}. Define $U_{2}\in \mathcal{G}_{2}$ by letting
	\begin{equation}
		\begin{gathered}
			U_{2}(\emptyset) = U_{2}(\{1, 2\}) = 0,\ U_{2}(\{1\}) = -1, \,\text{ and }\, U_{2}(\{2\}) = -2 .
		\end{gathered}
	\end{equation}
	Then, it is clear that $\phi^{\mathrm{ELC}}(U_{2}) = (0, 0)^{\top}$. If $\phi^{\mathrm{ELC}}$ is linear, it would reduce to the Shapley value \citep{shapley1953value}, as we have proved that $\phi^{\mathrm{ELC}}$ satisfies the axioms of equal treatment, consistency, and efficiency. However, the Shapley value of $U_{2}$ is $(\frac{1}{2}, -\frac{1}{2})^{\top}$, which indicates that $\phi^{\mathrm{ELC}}$ must be non-linear.
\end{proof}

\UnreliableShapley*
\begin{proof}
	According to \citet{yokote2016new}, a basis of the null space of the Shapley operator is $\{U_{T}\}_{T\in\mathcal{N}}$ where $\mathcal{N} \coloneq \{S\subseteq [n]\mid |S|\not=1\}$. Specifically,
	\begin{equation}
		\begin{gathered}
			U_{\emptyset}(S)  = 1 \ \text{ for every }\ S \subseteq [n],\\
			\text{and }\ U_{T}(S) = \begin{cases}
				1, & |T\cap S| = 1,\\
				0, & \text{otherwise.}
			\end{cases}
		\end{gathered}
	\end{equation}
	Let $2^{[n]}$ and $\mathcal{N}$ be ordered so that every $U_{T}$ can be written as a vector $\mathbf{a}_{T} \in \{0, 1\}^{2^{n}}$, all of which constitute a matrix $\mathbf{A} \in \{0, 1\}^{2^{n}\times|\mathcal{N}|}$. Then, every utility function $U'_{n}$ in this null space can be expressed as $U'_{n} = \mathbf{A}\boldsymbol\alpha$ where $\boldsymbol\alpha \in \mathbb{R}^{|\mathcal{N}|}$.
	
	Given a permutation $\pi \in \Pi_{n}$, we write
	\begin{equation}
		\begin{gathered}
			S_{k}(\pi) = \{\pi_{1}, \pi_{2},\dots, \pi_{k}\} \ \text{ for every } 0\leq k\leq n.
		\end{gathered}
	\end{equation}
	For each permutation $\pi \in \Pi_{n}$, it can be represented by a vector $\mathbf{p}_{\pi}\in\{0, 1\}^{2^{n}}$ whose $S$-th entry is $1$ if $S \in \{S_{k}(\pi)\colon 0\leq k \leq n\}$ and $0$ otherwise. Then, we have
	\begin{equation}
		\begin{gathered}
			\mathrm{AUC}(\pi; U_{n}') = \mathbf{p}_{\pi}^{\top}\mathbf{A}\boldsymbol\alpha - U'_{n}(\emptyset).
		\end{gathered}
	\end{equation}
	For convenience, we write $\mathbf{d}_{\pi}^{\top} = \mathbf{p}_{\pi}^{\top}\mathbf{A}$ for every $\pi \in \Pi_{n}$. Next, we prove that every $\mathbf{d}_{\pi}^{\top}$ can be re-ordered to the same vector. Note that this is equivalent to ordering $\mathcal{N}$ according to the given $\pi$. Without loss of generality, we demonstrate this ordering process by using $n=4$ as an example.
	
	For subsets that include $\pi_{1}$, their order is generated by
	\begin{equation}
		\begin{aligned}
			& \{\pi_{1}\} \\
			& \{\pi_{1}\},\ \{\pi_{1},\pi_{4}\},\\
			& \{\pi_{1}\},\ \{\pi_{1},\pi_{4}\},\ \{\pi_{1},\pi_{3}\},\ \{\pi_{1},\pi_{4},\pi_{3}\}\\
			& \{\pi_{1}\},\ \{\pi_{1},\pi_{4}\},\ \{\pi_{1},\pi_{3}\},\ \{\pi_{1},\pi_{4},\pi_{3}\},\ \{\pi_{1},\pi_{2}\},\ \{\pi_{1},\pi_{4},\pi_{2}\},\ \{\pi_{1},\pi_{3},\pi_{2}\},\ \{\pi_{1},\pi_{4},\pi_{3},\pi_{2}\} . 
		\end{aligned}
	\end{equation}
	Following these ordered subsets, the subsets that include $\pi_{2}$ but not $\pi_{1}$  are ordered as
	\begin{equation}
		\begin{aligned}
			& \{\pi_{2}\}\\
			& \{\pi_{2}\},\ \{\pi_{2},\pi_{4}\}\\
			& \{\pi_{2}\},\ \{\pi_{2},\pi_{4}\},\ \{\pi_{2},\pi_{3}\},\ \{\pi_{2},\pi_{4},\pi_{3}\} .
		\end{aligned}
	\end{equation}
	Finally, it is
	\begin{equation}
		\begin{aligned}
			& \{\pi_{3}\}\\
			& \{\pi_{3}\},\ \{\pi_{3},\pi_{4}\} .
		\end{aligned}
	\end{equation}
	The empty set is appended to these ordered non-empty subsets. The order is finalized by removing all the singleton subsets. Precisely, the $\pi$-specific order for $\mathcal{N}$ is
	\begin{equation}
		\begin{gathered}
			\{\pi_{1},\pi_{4}\},\ \{\pi_{1},\pi_{3}\},\ \{\pi_{1},\pi_{4},\pi_{3}\},\ \{\pi_{1},\pi_{2}\},\ \{\pi_{1},\pi_{4},\pi_{2}\},\ \{\pi_{1},\pi_{3},\pi_{2}\},\ \{\pi_{1},\pi_{4},\pi_{3},\pi_{2}\},\\
			\{\pi_{2},\pi_{4}\},\ \{\pi_{2},\pi_{3}\},\ \{\pi_{2},\pi_{4},\pi_{3}\},\ \{\pi_{3},\pi_{4}\},\  \emptyset.
		\end{gathered}
	\end{equation}
	One can verify that, using this order for $\mathcal{N}$, $\mathbf{d}_{\pi}^{\top} = (3,2,2,1,1,1,1,2,1,1,1,5)$. In other words, for $1\leq i\leq 2^{n}-1$, the $i$-th entry of $\mathbf{d}_{\pi}$ is equal to the number of the occurrences of the corresponding subset in the ordering process,  whereas the last entry is equal to $n+1$.
	
	From now on, assume the order of $\mathcal{N}$ is fixed. 
	Consequently, we can partition $\Pi_{n}$ into $\{\Pi_{n,j}\}_{j}$ such that, for every $\pi, \pi' \in \Pi_{n,j}$, $\mathbf{d}_{\pi} = \mathbf{d}_{\pi'}$. From the ordering process, it is clear that $|\{\Pi_{n,j}\}_{j}|\geq 2$ if $n\geq 3$.

	For each $\Pi_{n,j}$, let $U'_{n} = \mathbf{A}\mathbf{d}_{\pi}$ where $\pi \in \Pi_{n,j}$. It is clear that $\mathrm{AUC}(\pi; U_{n}') > \mathrm{AUC}(\pi'; U_{n}')$ for every $\pi' \not\in \Pi_{n,j}$. By choosing an appropriate scalar $c > 0$, we have the result by setting $U'_{n} = c\cdot \mathbf{A}\mathbf{d}_{\pi}$.
	
\end{proof}

\pNorm*
\begin{proof}
	We divide our proof into two parts. The first part is to demonstrate the properties of $\phi^{\infty}$, whereas the second part goes for $p \in (1, \infty)$. Obviously, $\phi^{p}$ with $p\in(0,\infty)\cup\{\infty\}$ satisfies the axiom of translation invariance due to the bias variable $b$.
	
	For convenience, we write
	\begin{equation}
		\begin{gathered}
			d(S, \mathbf{x}; U_{n}) \coloneq U_{n}(S) - b_{p}(U_{n}) - \sum_{i\in S}x_{i}, \\
			\text{and }\ e_{p}(U_{n}) \coloneq \min_{\mathbf{x}\in \mathbb{R}^{n}, b\in \mathbb{R}}\|\mathbf{Ax}+b\cdot\mathbf{1}_{2^{n}}-\mathbf{V}(U_{n})\|_{p}.
		\end{gathered}
	\end{equation}
	where $b_{p}(U_{n})$ is the optimal value of $b$ in then minimization problem with $\mathbf{x} = \phi^{\infty}(U_{n})$.
	
	Note that
	\begin{equation} \label{eq: e^inf}
		\begin{gathered}
			e_{\infty}(U_{n}) = \frac{\max_{S\subseteq [n]}d(S,\phi^{\infty}(U_{n}); U_{n}) - \min_{S\subseteq [n]}d(S, \phi^{\infty}(U_{n}); U_{n})}{2} .
		\end{gathered}
	\end{equation}
	
	Besides, for $p\in(0, \infty)$, define $\psi_{p}: \mathbb{R}\to\mathbb{R}$ by letting
	\begin{equation}
		\begin{gathered}
			\psi_{p}(t) = \mathrm{sign}(t)\cdot|t|^{p-1} .
		\end{gathered}
	\end{equation}
	
	\textbf{$\phi^{\infty}$ satisfies the axiom of consistency}.
	
	This is straightforward by observing that the bias variable $b$ ensures that there exist two subsets $S, T \subseteq [n]$ such that
	\begin{equation}
		\begin{gathered}
			d(S, \phi^{\infty}(U_{n}); U_{n}) = e_{\infty}(U_{n}) \ \text{ and }\ d(T, \phi^{\infty}(U_{n}); U_{n}) = -e_{\infty}(U_{n}).
		\end{gathered}
	\end{equation}
	
	\textbf{$\phi^{\infty}$ satisfies the axiom of monotonicity}.
	Suppose there exists $j \in [n]$ such that $U_{n}(S\cup \{j\}) \geq U_{n}(S)$ for every $S\subseteq [n]\setminus\{j\}$. Assume, for the sake of contradiction, that $\phi^{\infty}_{j}(U_{n}) < 0$. 
	Then, for every subset $\mathbf{S} \subseteq  [n]\setminus \{j\}$
	\begin{equation}
		\begin{gathered}
			d(S, \phi^{\infty}(U_{n}); U_{n}) < d(S\cup\{j\}, \phi^{\infty}(U_{n}); U_{n}) .
		\end{gathered}
	\end{equation}
	It indicates that
	\begin{equation}
		\begin{gathered}
			d(S, \phi^{\infty}(U_{n}); U_{n}) = e_{\infty}(U_{n}) \implies j \in S\ \text{ and }\ d(S, \phi^{\infty}(U_{n}); U_{n}) = -e_{\infty}(U_{n}) \implies j \not\in S .
		\end{gathered}
	\end{equation}
	Select
	\begin{equation}
		\begin{gathered}
			0 < \delta < \min_{S\subseteq [n]\setminus \{j\}}d(S\cup \{j\},\phi^{\infty}(U_{n}); U_{n}) - \min_{S\subseteq [n]\setminus\{j\}} d(S, \phi^{\infty}(U_{n}); U_{n}) .
		\end{gathered}
	\end{equation}
	Define $\boldsymbol\phi \in \mathbb{R}^{n}$ by letting
	\begin{equation}
		\begin{gathered}
			\phi_{k} \coloneq \begin{cases}
				\phi_{k}^{\infty}(U_{n}), & k \not= j,\\
				\phi_{k}^{\infty}(U_{n}) + \delta, &\text{otherwise.}
			\end{cases}
		\end{gathered}
	\end{equation}
	Clearly, 
	\begin{equation}
		\begin{gathered}
			\min_{S\subseteq [n]\setminus\{j\}}d(S\cup\{j\}, \boldsymbol\phi; U_{n}) < \min_{S\subseteq [n]\setminus\{j\}} d(S\cup\{j\}, \phi^{\infty}(U_{n}); U_{n}),
		\end{gathered}
	\end{equation}
	which, by Eq.~\eqref{eq: e^inf}, contradicts the optimality of $(\phi^{\infty}(U_{n}),\, b_{\infty}(U_{n}))$. Therefore, we must have $\phi_{j}^{\infty}(U_{n})\geq 0$. The other case can be proved similarly.
	
	\textbf{$\phi^{\infty}$ satisfies the axiom of equal treatment}.
	Suppose there exist $i, j \in [n]$ such that $U_{n}(S\cup \{i\}) = U_{n}(S\cup \{j\})$ for every $S\subseteq [n]\setminus\{i,j\}$. For the sake of contradiction, assume that $\phi^{\infty}_{i}(U_{n}) > \phi^{\infty}(U_{n})$. Define $\boldsymbol\phi \in \mathbb{R}^{n}$ by letting
	\begin{equation}
		\begin{gathered}
			\phi_{k} \coloneq \begin{cases}
				\frac{\phi_{i}^{\infty}(U_{n})+\phi_{j}^{\infty}(U_{n})}{2}, & k \in \{i,j\},\\
				\phi_{k}^{\infty}(U_{n}), & \text{otherwise.}
			\end{cases}
		\end{gathered}
	\end{equation} 
	Observe that for every $S \subseteq [n]\setminus\{i,j\}$
	\begin{equation}
		\begin{gathered}
			|d(S\cup \{i\}, \boldsymbol\phi; U_{n})| < \max\left( d(S\cup \{i\}, \phi^{\infty}(U_{n}); U_{n}),\, d(S\cup \{j\}, \phi^{\infty}(U_{n}); U_{n}) \right) .
		\end{gathered}
	\end{equation}
	It means that the solution of $(\boldsymbol\phi, b^{*})$ is also \begin{equation}
		\begin{gathered}
			2\times \left( \frac{\phi_{i}^{\infty}(U_{n}) + \phi_{j}^{\infty}(U_{n})}{2} \right)^{2} < \phi_{i}^{\infty}(U_{n})^{2} + \phi_{j}^{\infty}(U_{n})^{2} , 
		\end{gathered}
	\end{equation}
	a contradiction to that $\phi^{\infty}(U_{n})$ is the least-norm optimal solution.

	\textbf{$\phi^{\infty}$ is not linear}.
	Assume, for the sake of contradiction, that $\phi^{\infty}$ is linear. Since $\phi^{\infty}$ satisfies the axioms of linearity, consistency, equal treatment, and monotonicity, according to \citet{weber1988probabilistic}, $\phi^{\infty}(U_{n})$ can be expressed as
	\begin{equation} \label{eq: formula for inf if linear}
		\begin{gathered}
			\phi_{i}^{\infty}(U_{n}) = \sum_{S\subseteq [n]\setminus\{i\}} q_{|S|+1}[U_{n}(S\cup \{i\})- U_{n}(S)]
		\end{gathered}
	\end{equation}
	where $\mathbf{q} \in \mathbb{R}^{n}$ is non-negative and satisfies that $\sum_{k=1}^{n}\binom{n-1}{k-1}q_{k} = 1$. For $t \in [3, 5]$, define $U_{3}^{t} \in \mathcal{G}_{3}$ by letting
	\begin{equation}
		\begin{gathered}
			U_{3}^{t}(\emptyset) = 0,\ U_{3}^{t}(\{1\}) = U_{3}^{t}(\{2,3\}) = U_{3}^{t}(\{1,2,3\}) = 4,\ U_{3}^{t}(\{2\}) = U_{3}^{t}(\{3\}) = 3,\\
			\text{and }\ U_{3}^{t}(\{1,2\}) = U_{3}^{t}(\{1,3\}) = t.
		\end{gathered}
	\end{equation}
	Using the constraints for $\emptyset$ and $\{1\}$, we have
	\begin{equation}
		\begin{gathered}
			b_{\infty}(U_{3}^{t}) \in [-e_{\infty}(U_{3}^{t}),\, e_{\infty}(U_{3}^{t})]\ \text{ and }\  b+\phi_{1}^{\infty}(U_{3}^{t}) \in [4-e_{\infty}(U_{3}^{t}),\, 4+e_{\infty}(U_{3}^{t})], 
		\end{gathered}
	\end{equation}
	which leads to
	\begin{equation} \label{eq:inf x1}
		\begin{gathered}
			\phi_{1}^{\infty}(U_{3}^{t}) \in [4-2e_{\infty}(U_{3}^{t}),\, 4+2e_{\infty}(U_{3}^{t})].
		\end{gathered}
	\end{equation}
	Using the constraint for $\{2, 3\}$, there is
	\begin{equation} \label{eq:inf b+x2+x3}
		\begin{gathered}
			b^{*} + \phi_{2}^{\infty}(U_{3}^{t}) + \phi_{3}^{\infty}(U_{3}^{t}) \in [4-e_{\infty}(U_{3}^{t}),\, 4+e_{\infty}(U_{3}^{t})]
		\end{gathered}
	\end{equation}
	Combining Eqs.~\eqref{eq:inf x1} and~\eqref{eq:inf b+x2+x3} yields
	\begin{equation}
		\begin{gathered}
			b^{*} + \phi_{1}^{\infty}(U_{3}^{t}) + \phi_{2}^{\infty}(U_{3}^{t}) + \phi_{3}^{\infty}(U_{3}^{t}) \in [8-3e_{\infty}(U_{3}^{t}),\, 8+3e_{\infty}(U_{3}^{t})].
		\end{gathered}
	\end{equation}
	On the other hand, from the constraint for $\{1,2,3\}$, there is
	\begin{equation}
		\begin{gathered}
			b^{*} + \phi_{1}^{\infty}(U_{3}^{t}) + \phi_{2}^{\infty}(U_{3}^{t}) + \phi_{3}^{\infty}(U_{3}^{t}) \in [4-e_{\infty}(U_{3}^{t}),\, 4+e_{\infty}(U_{3}^{t})] .
		\end{gathered}
	\end{equation}
	Solving $4+e_{\infty}(U_{3}^{t})=8-3e_{\infty}(U_{3}^{t})$ yields an lower bound for $e_{\infty}(U_{3}^{t})$, i.e., $e_{\infty}(U_{3}^{t})\geq 1$. If $e^{*}=1$, following the derivation of Eq.~\eqref{eq:inf x1}, the lower bounds for $\phi_{1}^{\infty}(U_{3}^{t})$, $\phi_{2}^{\infty}(U_{3}^{t})$, and $\phi_{3}^{\infty}(U_{3}^{t})$ are $2$, $1$, and $1$, respectively. Then, one can verify that the solution of $b=1$ and $\mathbf{x}=(2,1,1)$ is indeed optimal, leading to $e_{\infty}(U_{3}^{t})=1$. In other words, $\phi^{\infty}(U_{3}^{t}) = (2,1,1)$ for every $t \in [3, 5]$.
	
	Then, it indicates that 
	\begin{equation}
		\begin{gathered}
			q_{1} = 1,\ q_{2} = q_{3} = 0.
		\end{gathered}
	\end{equation}
	However, Eq.~\eqref{eq: formula for inf if linear} produces $\phi^{\infty}(U_{3}^{t}) = (4, 3, 3)$, contradicting to the uniqueness of $\phi^{\infty}(U_{3}^{t})$. Therefore, $\phi^{\infty}$ must be nonlinear.
	
	\textbf{$\phi^{p}$ with $p\in(1, \infty)$ satisfies the axiom of consistency}.
	Suppose there exists $C\in \mathbb{R}$ such that $U_{n+1}(S) = U_{n}(S\cap [n]) + C\cdot\mathds{1}_{n+1 \in S}$ for every $S \subseteq  [n+1]$. Since the corresponding minimization problem is strictly convex, using the first order condition for optimality, we have
	\begin{equation} \label{eq:first order condition}
		\begin{gathered}
			\sum_{S \subseteq  [n]} \psi_{p}(d(S, \phi^{p}(U_{n});U_{n})) = 0\\
			\text{and }\ \sum_{S\subseteq [n]\setminus\{i\}}\psi_{p}(d(S\cup\{i\}, \phi^{p}(U_{n}); U_{n})) = 0 \ \text{ for every }\ i \in [n].
		\end{gathered}
	\end{equation}
	Define $\boldsymbol\phi \in \mathbb{R}^{n+1}$ by letting $\phi_{k}=\phi_{k}^{p}(U_{n})$ if $k\in [n]$ and $C$ otherwise. Define
	\begin{equation}
		\begin{gathered}
			d'(S,\mathbf{x}) \coloneq U_{n+1}(S) - b_{p}(U_{n}) - \sum_{i\in S} x_{i}.
		\end{gathered}
	\end{equation}
	Verifying the first order condition for $U_{n+1}$, 
	\begin{equation}
		\begin{gathered}
			\sum_{S\subseteq  [n+1]} \psi_{p}(d'(S, \boldsymbol\phi)) = 2\sum_{S\subseteq [n]} \psi_{p}(d(S, \phi^{p}(U_{n}); U_{n})) = 0,\\
			\sum_{S\subseteq [n+1]\setminus\{i\}} \psi_{d}(d'(S\cup\{i\}, \boldsymbol\phi)) 
			=2\sum_{S\subseteq [n]\setminus\{i\}}\psi_{p}(d(S\cup\{i\}, \phi^{p}(U_{n}); U_{n})) = 0 \ \text{ for every } i \in [n],\\
			\sum_{S\subseteq [n]} \psi_{d}(d'(S\cup\{n+1\}, \boldsymbol\phi)) = \sum_{S \subseteq  [n]}\psi_{d}(d(S,\phi^{p}(U_{n}); U_{n})) = 0.
		\end{gathered}
	\end{equation}

	\textbf{$\phi^{p}$ with $p\in(1, \infty)$ satisfies the axiom of monotonicity}.
	Suppose there exists $j \in [n]$ such that $U_{n}(S\cup \{j\}) \geq U_{n}(S)$ for every $S\subseteq [n]\setminus\{j\}$. Using Eq.~\eqref{eq:first order condition}, we have
	\begin{equation} \label{eq:both are zero}
		\begin{gathered}
			\sum_{S\subseteq [n]\setminus\{j\}} \psi_{d}(d(S, \phi^{p}(U_{n}); U_{n})) = \sum_{S\subseteq [n]\setminus\{j\}} \psi_{d}(d(S\cup\{j\}, \phi^{p}(U_{n}); U_{n})) = 0
		\end{gathered}
	\end{equation}
	Notice that the function $\psi_{d}$ is strictly increasing. If $\phi_{j}^{p}(U_{n}) < 0$, for every $S\subseteq [n]\setminus\{j\}$, we have
	\begin{equation}
		\begin{gathered}
			d(S,\phi^{p}(U_{n}); U_{n}) < d(S\cup\{j\}, \phi^{p}(U_{n}); U_{n}) .
		\end{gathered}
	\end{equation}
	Then,
	\begin{equation}
		\begin{gathered}
			\sum_{S\subseteq [n]\setminus\{j\}} \psi_{d}(d(S, \phi^{p}(U_{n}); U_{n})) < \sum_{S\subseteq [n]\setminus\{j\}} \psi_{d}(d(S\cup\{j\}, \phi^{p}(U_{n}); U_{n})),
		\end{gathered}
	\end{equation}
	which makes Eq.~\eqref{eq:both are zero} impossible. Therefore, $\phi_{j}^{p}(U_{n}) \geq 0$. The other case can be proved similarly.
	
	\textbf{$\phi^{p}$ with $p\in(1, \infty)$ satisfies the axiom of equal treatment}.
	Suppose there exist $i, j \in [n]$ such that $U_{n}(S\cup \{i\}) = U_{n}(S\cup \{j\})$ for every $S\subseteq [n]\setminus\{i,j\}$. Using Eq.~\eqref{eq:first order condition}, we have
	\begin{equation}
		\begin{gathered}
			\sum_{S\subseteq [n]\setminus\{i,j\}} \phi_{p}(d(S\cup\{i\}, \phi^{p}(U_{n}); U_{n})) = \sum_{S\subseteq [n]\setminus\{i,j\}} \phi_{p}(d(S\cup\{j\}, \phi^{p}(U_{n}); U_{n})).
		\end{gathered}
	\end{equation}
	Since the function $\phi_{p}$ is strictly increasing, this equality implies that $\phi_{i}^{p}(U_{n}) = \phi_{j}^{p}(U_{n})$.
	
	\textbf{$\phi^{p}$ with $p\in(1, \infty)$ is linear if and only if $p=2$}. 
	It is obvious that $\phi^{2}$ is linear; therefore, we focus on proving the converse. Suppose $\phi^{p}$ with $p \in (1, \infty)$ is linear. Since $\phi^{p}$ satisfies the axioms of linearity, consistency, equal treatment, and monotonicity, according to \citet{weber1988probabilistic}, we have
	\begin{equation} \label{eq:formula for p if linear}
		\begin{gathered}
			\phi^{p}_{i}(U_{n}) = \sum_{S\subseteq [n]\setminus \{i\}} \gamma^{[n]}_{|S|+1}[U_{n}(S\cup\{i\}) - U_{n}(S)] \ \text{ for every } i \in [n]
		\end{gathered}
	\end{equation}
	where $\boldsymbol\gamma^{[n]} \in \mathbb{R}^{n}$ is non-negative and satisfies that $\sum_{k=1}^{n}\binom{n-1}{k-1}\gamma^{[n]}_{k} = 1$.
	Define $U_{2} \in \mathcal{G}_{2}$ by letting
	\begin{equation}
		\begin{gathered}
			U_{2}(\emptyset) = 0 \ \text{ and }\ U_{2}(\{1\}) = U_{2}(\{2\}) = U_{2}(\{1,2\}) = 1 .
		\end{gathered}
	\end{equation}
	Then, 
	\begin{equation}
		\begin{gathered}
			\phi_{1}^{p}(U_{2}) = \phi_{2}^{p}(U_{2}) = \gamma^{[2]}_{1}.
		\end{gathered}
	\end{equation}
	Using the first order condition, as described in Eq.~\eqref{eq:first order condition}, there is
	\begin{equation}
		\begin{gathered}
			\psi_{p}(-b_{p}(U_{2})) = \psi_{p}(\gamma^{[2]}_{1} + b_{p}(U_{2}) - 1) \ \text{ and }\ \psi_{p}(1-\gamma^{[2]}_{1} -b_{p}(U_{2})) = \psi_{p}(2\gamma^{[2]}_{1} + b_{p}(U_{2}) - 1) .
		\end{gathered}
	\end{equation}
	Since $\psi_{p}$ is strictly increasing, it leads to
	\begin{equation}
		\begin{gathered}
			-b_{p}(U_{2}) = \gamma^{[2]}_{1} + b_{p}(U_{2}) - 1 \ \text{ and }\ 1-\gamma^{[2]}_{1} -b_{p}(U_{2}) = 2\gamma^{[2]}_{1} + b_{p}(U_{2}) - 1,
		\end{gathered}
	\end{equation}
	which, combined with $\gamma_{1}^{[2]} + \gamma_{2}^{[2]} = 1$, yields
	\begin{equation}
		\begin{gathered}
			\gamma_{1}^{[2]} = \gamma_{2}^{[2]} = 0.5 .
		\end{gathered}
	\end{equation}
	Then, employing the axiom of consistency, we obtain
	\begin{equation} \label{eq: constraint for gamma^[3]}
		\begin{gathered}
			\gamma^{[3]}_{1} + \gamma^{[3]}_{2} = \gamma^{[3]}_{2} + \gamma^{[3]}_{3} = 0.5 .
		\end{gathered}
	\end{equation}
	It is worth pointing out that $\phi^{p}$ is linear for every $p \in (1,\infty)$ if $n=2$, and thus we require $n\geq 3$. Define $U_{3} \in \mathcal{G}_{3}$ by letting
	\begin{equation}
		\begin{gathered}
			U_{3}(\emptyset) = U_{3}(\{2\}) = U_{2}(\{1,2\}) = U_{2}(\{1,3\}) = 0,\ U_{3}(\{1\}) = U_{3}(\{1,2,3\}) = 1,\\
			U_{2}(\{3\}) = -2\gamma^{[3]}_{1},\,\text{ and }\, U_{2}(\{2,3\}) = 1-2\gamma^{[3]}_{1} .
		\end{gathered}
	\end{equation}
	Using the first order condition, as described in Eq.~\eqref{eq:first order condition}, there is
	\begin{equation} \label{eq: foc for U3}
		\begin{gathered}
			\psi_{p}(d(\{3\}, \phi^{p}(U_{3}); U_{3})) + \psi_{p}(d(\{1,3\}, \phi^{p}(U_{3}); U_{3})) + \psi_{p}(d(\{2,3\}, \phi^{p}(U_{3}); U_{3})) \\
			+ \psi_{p}(d(\{1,2,3\}, \phi^{p}(U_{3}); U_{3})) = 0 .
		\end{gathered}
	\end{equation}
	Using Eqs.~\eqref{eq:formula for p if linear} and~\eqref{eq: constraint for gamma^[3]}, one can verify that
	\begin{equation}
		\begin{gathered}
			d(\{3\}, \phi^{p}(U_{3}); U_{3}) = d(\{1, 3\}, \phi^{p}(U_{3}); U_{3}) = -b_{p}(U_{3}) - \gamma^{[3]}_{1} - \gamma^{[3]}_{3},\\
			\text{and }\ d(\{2,3\}, \phi^{p}(U_{3}); U_{3}) = d(\{1,2,3\}, \phi^{p}(U_{3}); U_{3}) = 1-3\gamma^{[3]}_{3}-b_{p}(U_{3}) .
		\end{gathered}
	\end{equation}
	Thus, Eq.~\eqref{eq: foc for U3} can be simplified as
	\begin{equation}
		\begin{gathered}
			\psi_{p}(1-3\gamma^{[3]}_{3}-b_{p}(U_{3})) = \psi_{p}(b_{p}(U_{3}) + \gamma_{1}^{[3]}+\gamma_{3}^{[3]}) .
		\end{gathered}
	\end{equation}
	Since $\psi_{p}$ is strictly increasing, it leads to
	\begin{equation}
		\begin{gathered}
			1-3\gamma^{[3]}_{3}-b_{p}(U_{3}) = b_{p}(U_{3}) + \gamma_{1}^{[3]}+\gamma_{3}^{[3]},
		\end{gathered}
	\end{equation}
	which yields
	\begin{equation} \label{eq: b for U3}
		\begin{gathered}
			b_{p}(U_{3}) = 0.5 - 2\gamma_{3}^{[3]} - 0.5\gamma_{1}^{[3]}.
		\end{gathered}
	\end{equation}
	Using the first order condition again, 
	\begin{equation}
		\begin{gathered}
			\psi_{p}(d(\emptyset, \phi^{p}(U_{3}); U_{3})) + \psi_{p}(d(\{2\}, \phi^{p}(U_{3}); U_{3})) + \psi_{p}(d(\{3\}, \phi^{p}(U_{3}); U_{3})) \\
			+ \psi_{p}(d(\{2,3\}, \phi^{p}(U_{3}); U_{3})) = 0 .
		\end{gathered}
	\end{equation}
	Using Eqs~\eqref{eq:formula for p if linear},~\eqref{eq: constraint for gamma^[3]} and~\eqref{eq: b for U3}, 
	\begin{equation}
		\begin{gathered}
			d(\emptyset, \phi^{p}(U_{3}); U_{3}) = 0.75 - 2.5 \gamma_{2}^{[3]},\\
			d(\{2\}, \phi^{p}(U_{3}); U_{3}) = 0.25 - 1.5\gamma_{2}^{[3]},\\
			d(\{3\}, \phi^{p}(U_{3}); U_{3}) = -0.5\gamma_{2}^{[3]} - 0.25,\\
			\text{and }\ d(\{2,3\}, \phi^{p}(U_{3}); U_{3}) = 0.5\gamma_{2}^{[3]} + 0.25.
		\end{gathered}
	\end{equation}
	Since the function $\psi_{p}$ is odd and strictly increasing, we have
	\begin{equation}
		\begin{gathered}
			0.75 - 2.5 \gamma_{2}^{[3]} = -0.25 + 1.5\gamma_{2}^{[3]},
		\end{gathered}
	\end{equation} 
	which, combined with Eq.~\eqref{eq: constraint for gamma^[3]}, yields
	\begin{equation}\label{eq:gamma^3}
		\begin{gathered}
			\gamma_{1}^{[3]} = \gamma_{2}^{[3]} = \gamma_{3}^{[3]} = 0.25 .
		\end{gathered}
	\end{equation}
	Define $U'_{3} \in \mathcal{G}_{3}$ by letting
	\begin{equation}
		\begin{gathered}
			U'_{3}(\{1\}) = U'_{3}(\{1,3\}) = U'_{3}(\{2,3\}) = U'_{3}(\{1,2,3\}) = U'_{3}(\{2\}) = 0,\\
			U'_{3}(\{3\}) = U'_{3}(\emptyset) = -4,\, \text{ and }\, U'_{3}(\{1,2\}) = 8.
		\end{gathered}
	\end{equation}
	Using the first order condition, 
	\begin{equation}
		\begin{gathered}
			\psi_{p}(d(\emptyset, \phi^{p}(U'_{3}); U_{3})) + \psi_{p}(d(\{2\}, \phi^{p}(U'_{3}); U_{3})) + \psi_{p}(d(\{3\}, \phi^{p}(U'_{3}); U_{3})) \\
			+ \psi_{p}(d(\{2,3\}, \phi^{p}(U'_{3}); U_{3})) = 0,
		\end{gathered}
	\end{equation}
	which, combined with Eqs.~\eqref{eq:formula for p if linear} and~\eqref{eq:gamma^3}, is equal to
	\begin{equation}
		\begin{gathered}
			\psi_{p}(-b_{p}(U'_{3})-4) = \psi_{p}(b_{p}(U'_{3})+2) .
		\end{gathered}
	\end{equation}
	Since $\psi_{p}$ is strictly increasing, 
	\begin{equation} \label{eq:b for U'3}
		\begin{gathered}
			b_{p}(U'_{3}) = -3 .
		\end{gathered}
	\end{equation}
	Using the first order condition again,
	\begin{equation}
		\begin{gathered}
			\psi_{p}(d(\{3\}, \phi^{p}(U'_{3}); U_{3})) + \psi_{p}(d(\{1,3\}, \phi^{p}(U'_{3}); U_{3})) + \psi_{p}(d(\{2,3\}, \phi^{p}(U'_{3}); U_{3})) \\
			+ \psi_{p}(d(\{1,2,3\}, \phi^{p}(U'_{3}); U_{3})) = 0 ,
		\end{gathered}
	\end{equation}
	which, combined with Eqs.~\eqref{eq:formula for p if linear},~\eqref{eq:gamma^3} and~\eqref{eq:b for U'3}, is equal to 
	\begin{equation}
		\begin{gathered}
			3\cdot 1^{p-1} = 3\psi_{p}(1) = \psi_{p}(3) = 3^{p-1}.
		\end{gathered}
	\end{equation}
	Solving $3^{p-2} = 1$ yields $p=2$.
\end{proof}

\SolveInfinity*
\begin{proof}
	The minimization problem can be equivalently rewritten as
	\begin{equation} \label{op:regularized inf}
		\begin{gathered}
			\minimize_{\mathbf{x}\in\mathbb{R}^{n}, b, e\in\mathbb{R}}\ e + \eta\cdot\|\mathbf{x}\|_{2}^{2}\\
			\text{subject to }\ |U_{n}(S) - b - \sum_{i\in S}x_{i}| \leq e \ \text{ for every } S \subseteq [n] .
		\end{gathered}
	\end{equation}
	Then, according to \citet{blum1972direct}, $\phi^{\infty}(U_{n}; \eta)$ exists. 
	To prove $\phi^{\infty}(U_{n};\eta) \rightarrow \phi^{\infty}(U_{n})$ as $\eta \rightarrow 0$, it is sufficient to prove that for every sequence $(\phi^{\infty}(U_{n}; \eta_{k}))_{k=1}^{\infty}$ such that $\eta_{k} \rightarrow 0$ as $k\rightarrow \infty$, there is $ \phi^{\infty}(U_{n}; \eta_{k}) \rightarrow \phi^{\infty}(U_{n})$ as $k\rightarrow \infty$.
	
	For convenience, we write
	\begin{equation}
		\begin{gathered}
			\mathbf{x}_{\eta} \coloneq \phi^{\infty}(U_{n}; \eta),
		\end{gathered}
	\end{equation}
	and let $b_{\eta}$ be the optimal solution of $b$ associated with $\mathbf{x}_{\eta}$. Besides, write $\mathbf{x}_{0} \coloneq \phi^{\infty}(U_{n})$, and define $b_{0}$ similarly.
	Additionally, we append $\mathbf{1}_{2^{n}}$ as the last column to $\mathbf{A}$, the result of which is denoted by $\mathbf{B}$, and write $\mathbf{y}_{\eta} \coloneq (\mathbf{x}_{\eta}, b_{\eta}) \in \mathbb{R}^{n+1}$. 
	
	Observe that
	\begin{equation} \label{eq:upper bound for |Ax-b|}
		\begin{aligned}
			\|\mathbf{A}\mathbf{x}_{\eta} + b_{\eta}\cdot\mathbf{1}_{2^{n}} - \mathbf{V}(U_{n})\|_{\infty} &= \|\mathbf{B}\mathbf{y}_{\eta} - \mathbf{V}(U_{n})\|_{\infty}\\
			&\leq \|\mathbf{B}\mathbf{y}_{\eta} - \mathbf{V}(U_{n})\|_{\infty} \,+\, \eta\cdot\|\mathbf{x}_{\eta}\|_{2}^{2}\\
			&\leq \|\mathbf{B}\mathbf{y}_{0} - \mathbf{V}(U_{n})\|_{\infty} \,+\, \eta\cdot\|\mathbf{x}_{0}\|_{2}^{2} .
		\end{aligned}
	\end{equation}
	It indicates that $(\mathbf{y}_{\eta_{k}})_{k=1}^{\infty}$ is eventually bounded. Precisely,
	\begin{equation}
		\begin{aligned}
			\|\mathbf{y}_{\eta}\|_{2} &\leq \frac{1}{\sigma_{\mathrm{min}}(\mathbf{B})}\|\mathbf{B}\mathbf{y}_{\eta}\|_{2}\\
			&\leq \frac{C_{\infty}}{\sigma_{\mathrm{min}}(\mathbf{B})}\|\mathbf{B}\mathbf{y}_{\eta}\|_{\infty}\\
			&\leq \frac{C_{\infty}}{\sigma_{\mathrm{min}}(\mathbf{B})} \left( \|\mathbf{B}\mathbf{y}_{\eta} - \mathbf{V}(U_{n})\|_{\infty} \,+\, \|\mathbf{V}(U_{n})\|_{\infty} \right)\\
			&\leq \frac{C_{\infty}}{\sigma_{\mathrm{min}}(\mathbf{B})} \left( \|\mathbf{B}\mathbf{y}_{0} - \mathbf{V}(U_{n})\|_{\infty} \,+\, \eta\cdot\|\mathbf{x}_{0}\|_{2}^{2} \,+\, \|\mathbf{V}(U_{n})\|_{\infty} \right).
		\end{aligned}
	\end{equation}
	For the first inequality, $\sigma_{\mathrm{min}}(\mathbf{B})$ denotes the smallest singular value of $\mathbf{B}$; $\sigma_{\mathrm{min}}(\mathbf{B})>0$ as $\mathbf{B}$ has full column rank.\footnote{We have proved that $\mathbf{B}^{\top}\mathbf{B}$ is invertible in \citep[Theorem 2]{li2024one}} The second inequality follows from the fact that all norms on $\mathbb{R}^{n}$ are equivalent.
	
	Since the sequence $(\mathbf{y}_{\eta_{k}})_{k=1}^{\infty}$ is eventually bounded, it is valid to select some convergent subsequence. Therefore, suffice it to prove that every convergent subsequence of $(\mathbf{x}_{\eta_{k}})_{k=1}^{\infty}$ converges to the same point $\mathbf{x}_{0}$.\footnote{This is equivalent to proving that $\liminf_{k\rightarrow\infty}a_{k} = \limsup_{k\rightarrow\infty}a_{k}$ for a sequence $(a_{k})_{k=1}^{\infty}$.}
	
	For simplicity, we abuse the notation $(\mathbf{y}_{\eta_{k}})_{k=1}^{\infty}$ to refer to a convergent subsequence of $(\mathbf{y}_{\eta_{k}})_{k=1}^{\infty}$.\footnote{Given a subsequence $(\mathbf{x}_{\eta_{k_{j}}})_{j=1}^{\infty} $ that converges to $ \mathbf{x}'$, it may happen that $b_{\eta_{k_{j}}}$ does not converge. Nevertheless, we can further select a subsequence of it such that the corresponding bias term converges, and $\mathbf{x}'$ is still the limit.} Thus, we have $\mathbf{y}_{\eta_{k}} \rightarrow \mathbf{y}'$ as $k\rightarrow \infty$ for some point $\mathbf{y}' \in \mathbb{R}^{n+1}$, the first $n$ entries of which is denoted by $\mathbf{x}'$. Since
	\begin{equation}
		\begin{gathered}
			\|\mathbf{B}\mathbf{y}_{0} - \mathbf{V}(U_{n})\|_{\infty} \leq \|\mathbf{B}\mathbf{y}_{\eta} - \mathbf{V}(U_{n})\|_{\infty},
		\end{gathered}
	\end{equation}
	combining this with Eq.~\eqref{eq:upper bound for |Ax-b|}, we have
	\begin{equation}
		\begin{gathered}
			0 \leq \|\mathbf{B}\mathbf{y}_{\eta} - \mathbf{V}(U_{n})\|_{\infty} - \|\mathbf{B}\mathbf{y}_{0} - \mathbf{V}(U_{n})\|_{\infty} \leq \eta\cdot \|\mathbf{x}_{0}\|_{2}^{2} .
		\end{gathered}
	\end{equation}
	This suggests that $\|\mathbf{B}\mathbf{y}_{\eta_{k}} - \mathbf{V}(U_{n})\|_{\infty} \rightarrow \|\mathbf{B}\mathbf{y}_{0} - \mathbf{V}(U_{n})\|_{\infty}$ as $k\rightarrow \infty$. Since the function $\mathbf{y} \mapsto \|\mathbf{By} - \mathbf{V}(U_{n})\|_{\infty}$ is continuous, we have
	\begin{equation} \label{eq:y' is optimal}
		\begin{gathered}
			\|\mathbf{B}\mathbf{y}' - \mathbf{V}(U_{n})\|_{\infty} = \|\mathbf{B}\mathbf{y}_{0} - \mathbf{V}(U_{n})\|_{\infty} .
		\end{gathered}
	\end{equation}
	
	On the other hand, since $\mathbf{x} \mapsto \|\mathbf{x}\|_{2}^{2}$ is continuous, there is
	\begin{equation}
		\begin{gathered}
			\lim_{k\rightarrow\infty}\|\mathbf{x}_{\eta_{k}}\|_{2}^{2} \rightarrow\|\mathbf{x}'\|_{2}^{2} .
		\end{gathered}
	\end{equation}
	Besides, note that
	\begin{equation}
		\begin{aligned}
			\|\mathbf{B}\mathbf{y}_{0} - \mathbf{V}(U_{n})\|_{\infty} + \eta\cdot\|\mathbf{x}_{\eta}\|_{2}^{2} &\leq \|\mathbf{B}\mathbf{y}_{\eta} - \mathbf{V}(U_{n})\|_{\infty} + \eta\cdot\|\mathbf{x}_{\eta}\|_{2}^{2} \\
			&\leq \|\mathbf{B}\mathbf{y}_{0} - \mathbf{V}(U_{n})\|_{\infty} + \eta\cdot\|\mathbf{x}_{0}\|_{2}^{2},
		\end{aligned}
	\end{equation}
	which implies
	\begin{equation}
		\begin{gathered}
			\|\mathbf{x}_{\eta}\|_{2}^{2} \leq \|\mathbf{x}_{0}\|_{2}^{2} .
		\end{gathered}
	\end{equation}
	Then,
	\begin{equation} \label{eq:norm for x'}
		\begin{gathered}
			\|\mathbf{x}'\|_{2}^{2} \leq \|\mathbf{x}_{0}\|_{2}^{2} .
		\end{gathered}
	\end{equation}
	Using Eqs.~\eqref{eq:y' is optimal} and~\eqref{eq:norm for x'}, and the uniqueness of $\mathbf{x}_{0}$, we obtain
	\begin{equation}
		\begin{gathered}
			\mathbf{x}' = \mathbf{x}_{0} .
		\end{gathered}
	\end{equation}
\end{proof}

\SolveELC*
\begin{proof}
	According to \citet{blum1972direct}, $\phi^{\mathrm{ELC}}(U_{n}; \eta)$ exists. 
	For every $\eta > 0$, let the produced optimal solution be denoted by $(\mathbf{x}_{\eta}, e_{\eta})$, and similarly we define $(\mathbf{x}_{0}, e_{0})$ as the optimal solution corresponding to the egalitarian least core. In other words, $\mathbf{x}_{\eta} = \phi^{\mathrm{ELC}}(U_{n};\eta)$ and $\mathbf{x}_{0} = \phi^{\mathrm{ELC}}(U_{n})$. Then, suffice it to prove that for every sequence $(\mathbf{x}_{\eta_{k}})_{k=1}^{n}$ such that $\eta_{k}\rightarrow0$ as $k\rightarrow\infty$, there is $\mathbf{x}_{\eta_{k}}\rightarrow \mathbf{x}_{0}$ as $k \rightarrow \infty$.
	
	Notice that
	\begin{equation}
		\begin{gathered}
			e_{0} \leq e_{\eta} \leq e_{\eta} + \eta\cdot\|\mathbf{x}_{\eta}\|_{2}^{2}\, \leq e_{0} + \eta\cdot\|\mathbf{x}_{0}\|_{2}^{2},
		\end{gathered}
	\end{equation}
	which yields
	\begin{equation} \label{eq:bound for e_eta}
		\begin{gathered}
			0 \leq e_{\eta} - e_{0} \leq \eta\cdot\|\mathbf{x}_{0}\|_{2}^{2}.
		\end{gathered}
	\end{equation}
	Therefore, $e_{\eta_{k}}\rightarrow e_{0}$ as $k\rightarrow\infty$.
	On the other hand,
	\begin{equation}
		\begin{gathered}
			e_{0} + \eta\cdot\|\mathbf{x}_{\eta}\|_{2}^{2} \leq e_{\eta} + \eta\cdot\|\mathbf{x}_{\eta}\|_{2}^{2} \leq e_{0} + \eta\cdot\|\mathbf{x}_{0}\|_{2}^{2},
		\end{gathered}
	\end{equation}
	which leads to
	\begin{equation}
		\begin{gathered}
			\|\mathbf{x}_{\eta}\|_{2}^{2} \leq \|\mathbf{x}_{0}\|_{2}^{2} .
		\end{gathered}
	\end{equation}
	Therefore, the sequence $(\mathbf{x}_{\eta_{k}})_{k=1}^{\infty}$ is eventually bounded, and thus we can select a convergent subsequence. In particular, there must be
	\begin{equation}
		\begin{gathered}
			\liminf_{k\rightarrow\infty}\|\mathbf{x}_{\eta_{k}}\|_{2}^{2} = \limsup_{k\rightarrow\infty}\|\mathbf{x}_{\eta_{k}}\|_{2}^{2} = \|\mathbf{x}_{0}\|_{2}^{2} .
		\end{gathered}
	\end{equation} 
	Otherwise, there would exist a subsequence of $(\mathbf{x}_{\eta_{k}})_{k=1}^{\infty}$ converging to some point $\mathbf{x}' \in \mathbb{R}^{n}$ such that
	\begin{equation}
		\begin{gathered}
			\|\mathbf{x}'\|_{2}^{2} = \liminf_{k\rightarrow\infty}\|\mathbf{x}_{\eta_{k}}\|_{2}^{2} < \|\mathbf{x}_{2}\|_{2}^{2}  ,
		\end{gathered}
	\end{equation}
	which, together with Eq.~\eqref{eq:bound for e_eta}, contradicts the uniqueness of $\mathbf{x}_{0}$. Note that $(\mathbf{x}', e_{0})$ is a feasible solution, since the domain is closed. Consequently,
	\begin{equation}
		\begin{gathered}
			\lim_{k\rightarrow\infty}\|\mathbf{x}_{\eta_{k}}\|_{2}^{2} = \|\mathbf{x}_{0}\|_{2}^{2}.
		\end{gathered}
	\end{equation}
	
	Then, every convergent subsequence of $(\mathbf{x}_{\eta_{k}})_{k=1}^{\infty}$ must converge to the same point $\mathbf{x}_{0}$, since $\mathbf{x}_{0}$ is the unique vector such that (i) $(e_{0}, \mathbf{x}_{0})$ is feasible to the problem of least core, and (ii) for any feasible $(e_{0}, \mathbf{x})$, it holds that $\|\mathbf{x}_{0}\|_{2}^{2}\leq \|\mathbf{x}\|_{2}^{2}$. As a result, we have $\mathbf{x}_{\eta_{k}}\rightarrow \mathbf{x}_{0}$ as $k\rightarrow \infty$.
\end{proof}

\section{Experiments} \label{app:exp}
The datasets used are summarized in Table~\ref{tab:sum}. More experimental results on the inclusion curves are presented in Figures~\ref{fig:appr-detailed} and~\ref{fig:exact-detailed}.

\paragraph{Statistical comparison.} As shown in Figures~\ref{fig:heatmap1} and~\ref{fig:heatmap2}, we also compare these methods statistically. For the value associated with methods $A$ (row) and $B$ (column), it is computed as
\begin{equation}
	\begin{gathered}
		\frac{\sum_{i=1}^{m} [\mathds{1}\{\mathrm{AUC}^{A}_{i} > \mathrm{AUC}^{B}_{i}\} - \mathds{1}\{\mathrm{AUC}^{A}_{i} < \mathrm{AUC}^{B}_{i}\}]}{m}
	\end{gathered}
\end{equation}
where $m$ is the total number of instances used and $\mathrm{AUC}_{i}^{A}$ is the AUC achieved by method $A$ on the utility function created using the $i$-th instance. Therefore, the larger it is, the better method $A$ is.

\begin{table*}[t]
	\centering
	\caption{Summary of the used datasets.}
	\label{tab:sum}
	\resizebox{\textwidth}{!}{
		\begin{tabular}{lrrllr}
			\toprule
			Dataset & \#Instances & \#Features & Source & Task&\#Classes\\ \midrule
			
			GPSP \citep{madeo2013gesture}& $ 9,873 $ & $ 32 $ & \url{https://openml.org/d/4538}&classification&$ 5 $\\
			
			FOTP \citep{bridge2014machine} & $ 6,118 $ & $ 51 $ & \url{https://openml.org/d/1475}& classification&$ 6 $\\
			
			wave\_energy & $ 72,000 $ & $ 48 $ & \url{https://openml.org/d/44975}&regression&-\\
			
			jannis & $ 83,733 $ & $ 54 $ & \url{https://openml.org/d/41168}&classification&$ 4 $\\
			
			spambase & $ 4,601 $ & $ 57 $ & \url{https://openml.org/d/44}&classification&$ 2 $\\
			
			superconduct & $ 21,263 $ & $ 81 $& \url{https://openml.org/d/43174}&regression&-\\
			
			letter \citep{frey1991letter} & $20,000$ & $16$ & \url{https://openml.org/d/6} & classification&$26$\\
			
			pendigits & $10,992$  & $16$ & \url{https://openml.org/d/32} & classification & $10$\\
			
			EES \citep{eeg_eye_state_264} & $14,980$ & $14$ & \url{https://openml.org/d/1471} & classification & $2$\\
			
			WQW \citep{cortez2009modeling} & $4,898$ & $11$ & \url{https://openml.org/d/40498} & classification&$7$\\
			
			elevators & $16,599$ & $18$ & \url{https://openml.org/d/846} & classification & 2\\
			
			credit & $16,714$ & $10$ & \url{https://openml.org/d/44089} & classification & 2\\	
			
			\bottomrule		
	\end{tabular}}
\end{table*}

\begin{figure*}[t]
	\centering
	\begin{tabular}{ccc}
		\multicolumn{3}{c}{\includegraphics[width=0.9\linewidth]{legend_comparison.pdf}} \\
		\includegraphics[width=0.3\linewidth]{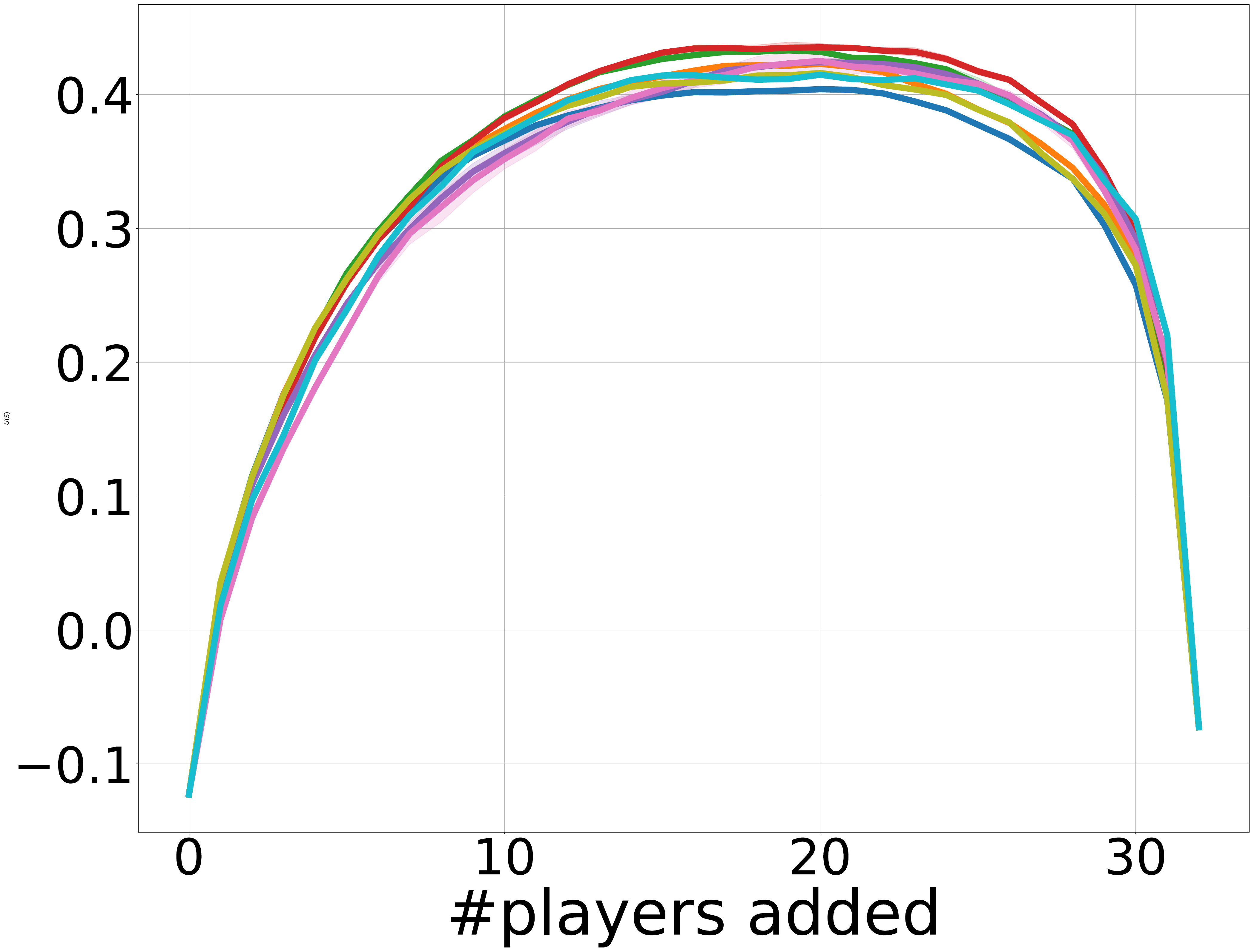} & \includegraphics[width=0.3\linewidth]{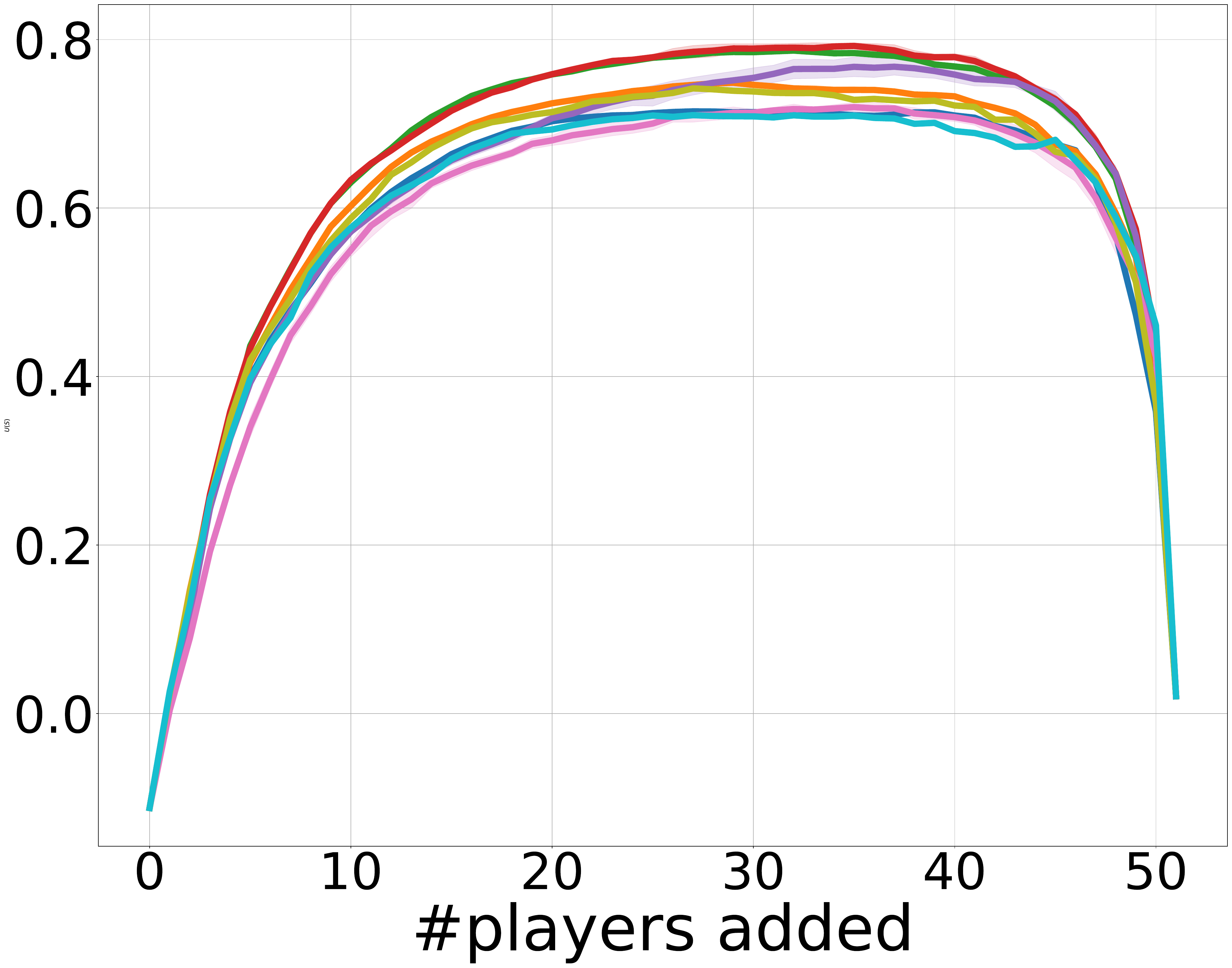} &
		\includegraphics[width=0.3\linewidth]{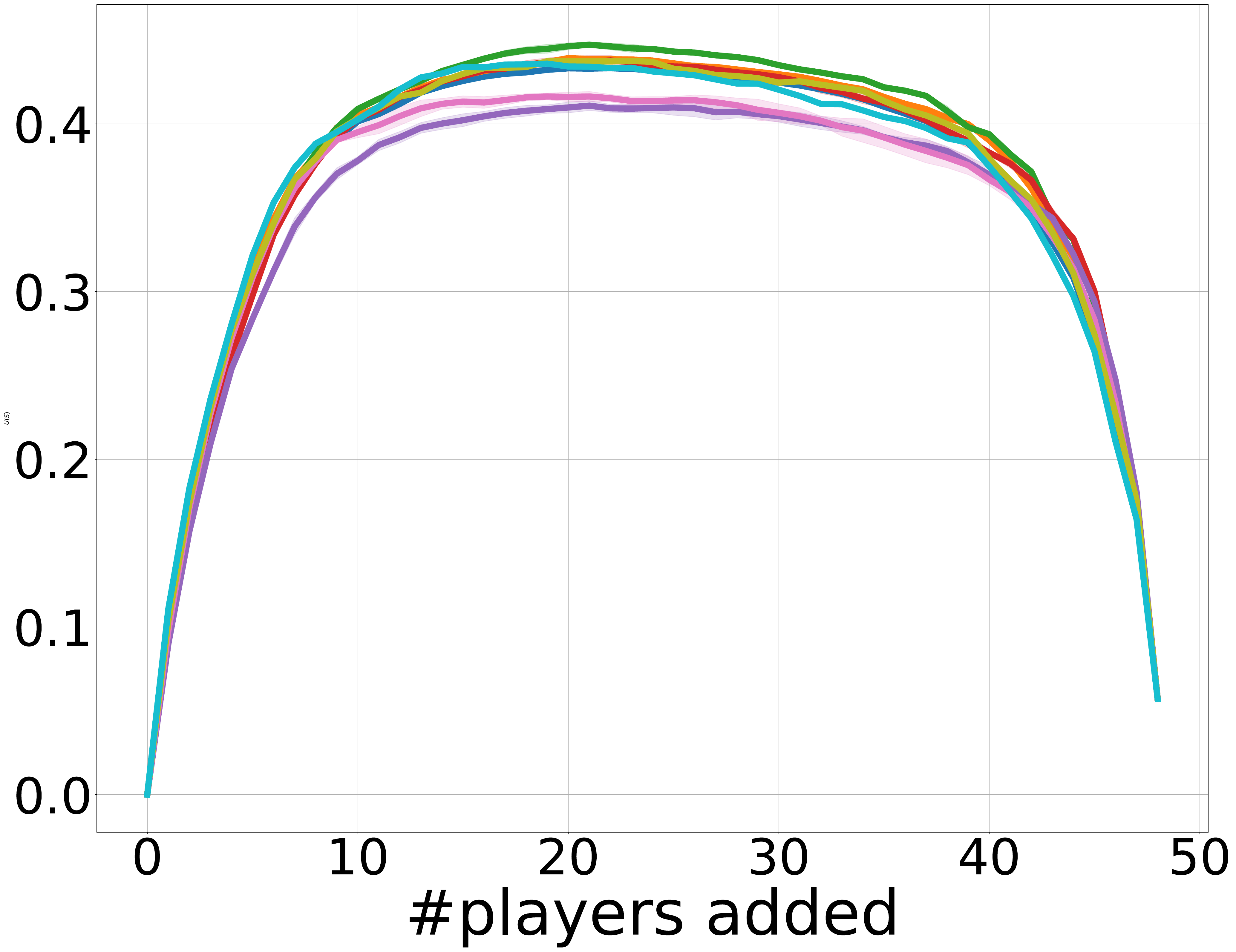} \\
		\phantom{aaaaa}GPSP & \phantom{aaaaa}FOTP & \phantom{aaaaa}wave\_energy\\
		\includegraphics[width=0.3\linewidth]{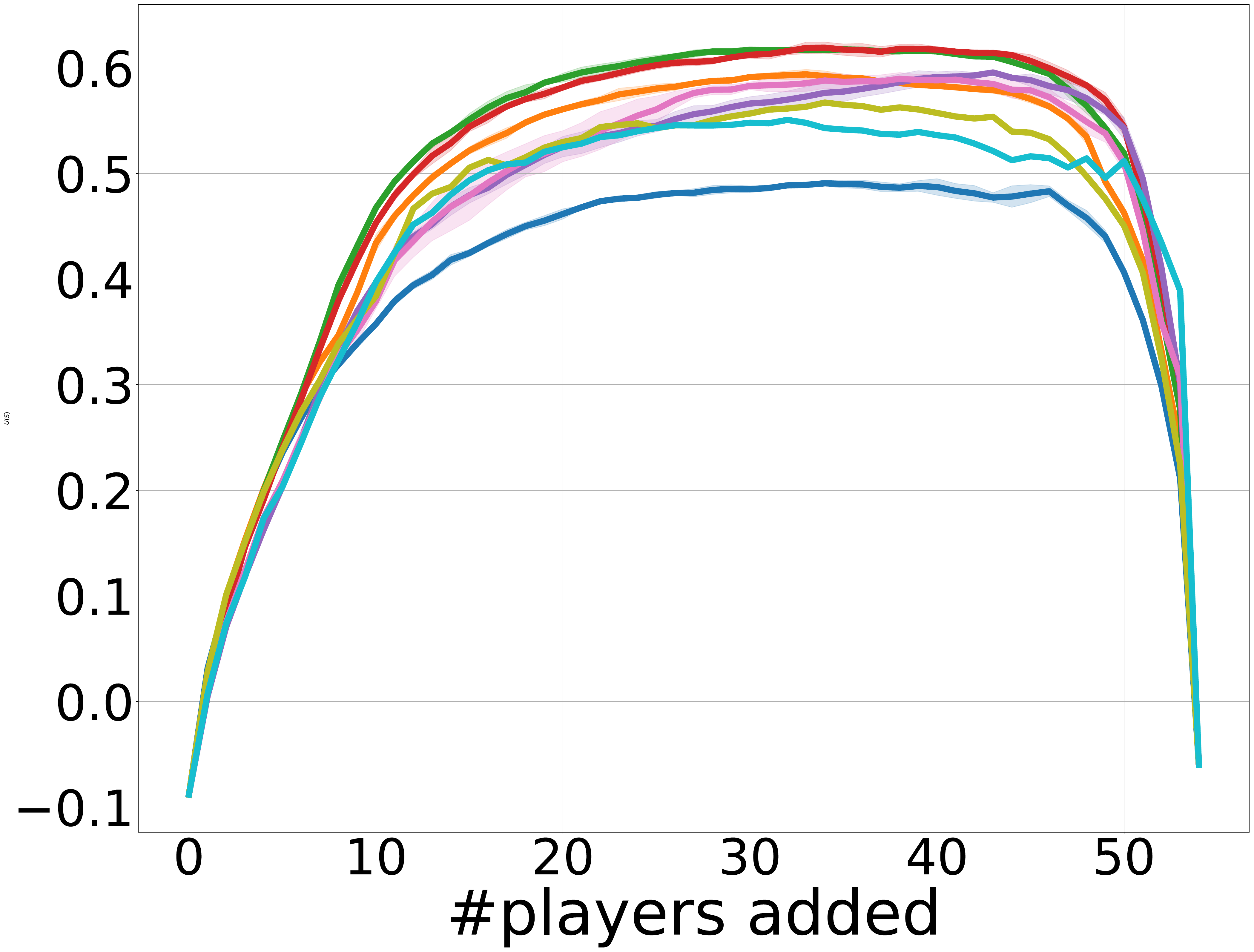} & \includegraphics[width=0.3\linewidth]{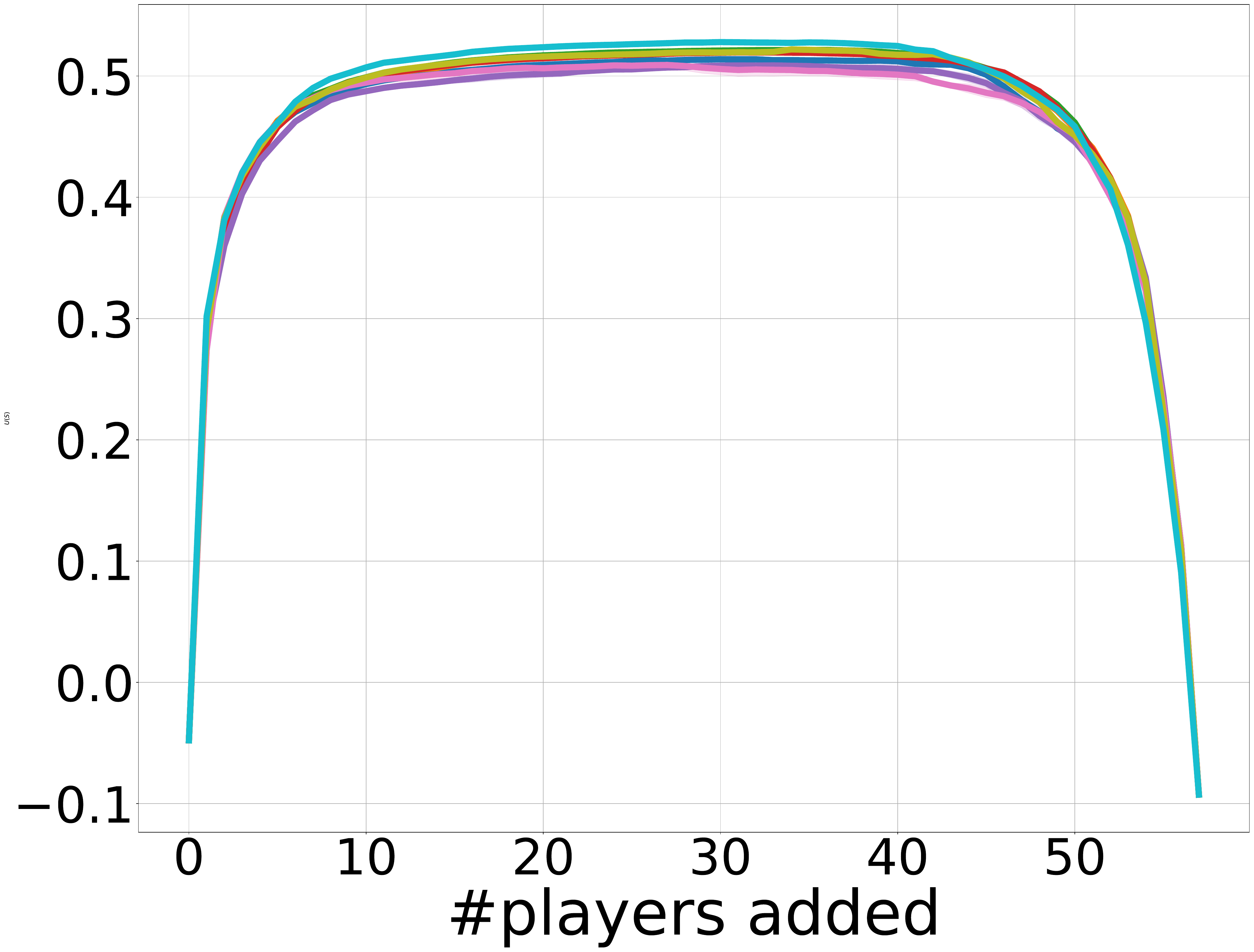} &
		\includegraphics[width=0.3\linewidth]{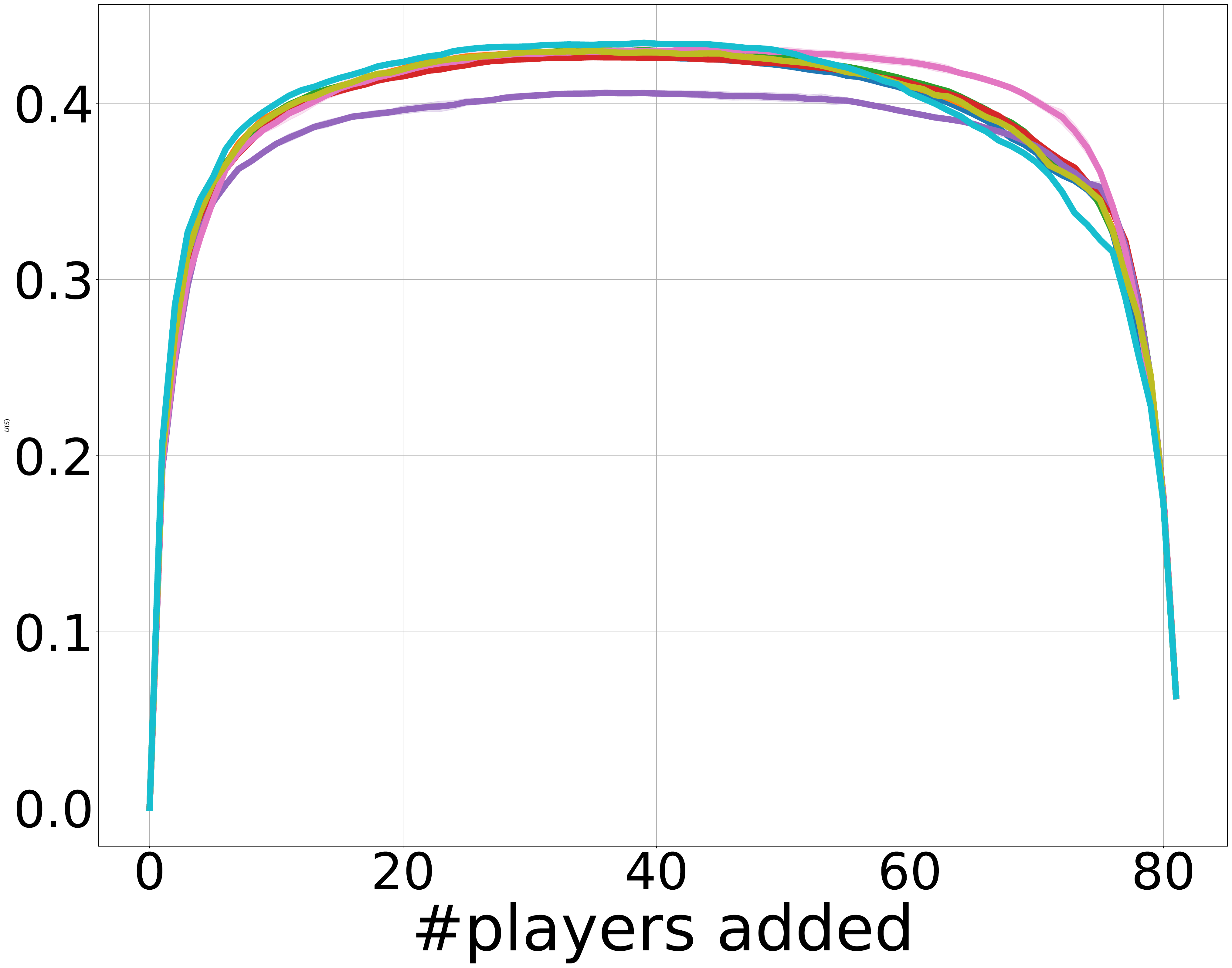} \\
		\phantom{aaaaa}jannis & \phantom{aaaaa}spambase & \phantom{aaaaa}superconduct
	\end{tabular}
	\caption{Comparison of attribution methods on six datasets for player ranking. A larger area under the curve indicates better performance. Since all utility functions contain more than $30$ players, nonlinear methods are approximated by sampling $\#\mathrm{players}\times 1,000$ subsets, with mean and standard deviation reported over five random seeds. Linear methods are computed exactly. Beta Shapley achieves the best inclusion AUC among the $10$ candidates.}
	\label{fig:appr-detailed}
\end{figure*}

\begin{figure*}[t]
	\centering
	\begin{tabular}{ccc}
		\multicolumn{3}{c}{\includegraphics[width=\linewidth]{legend_comparison.pdf}} \\
		\includegraphics[width=0.3\linewidth]{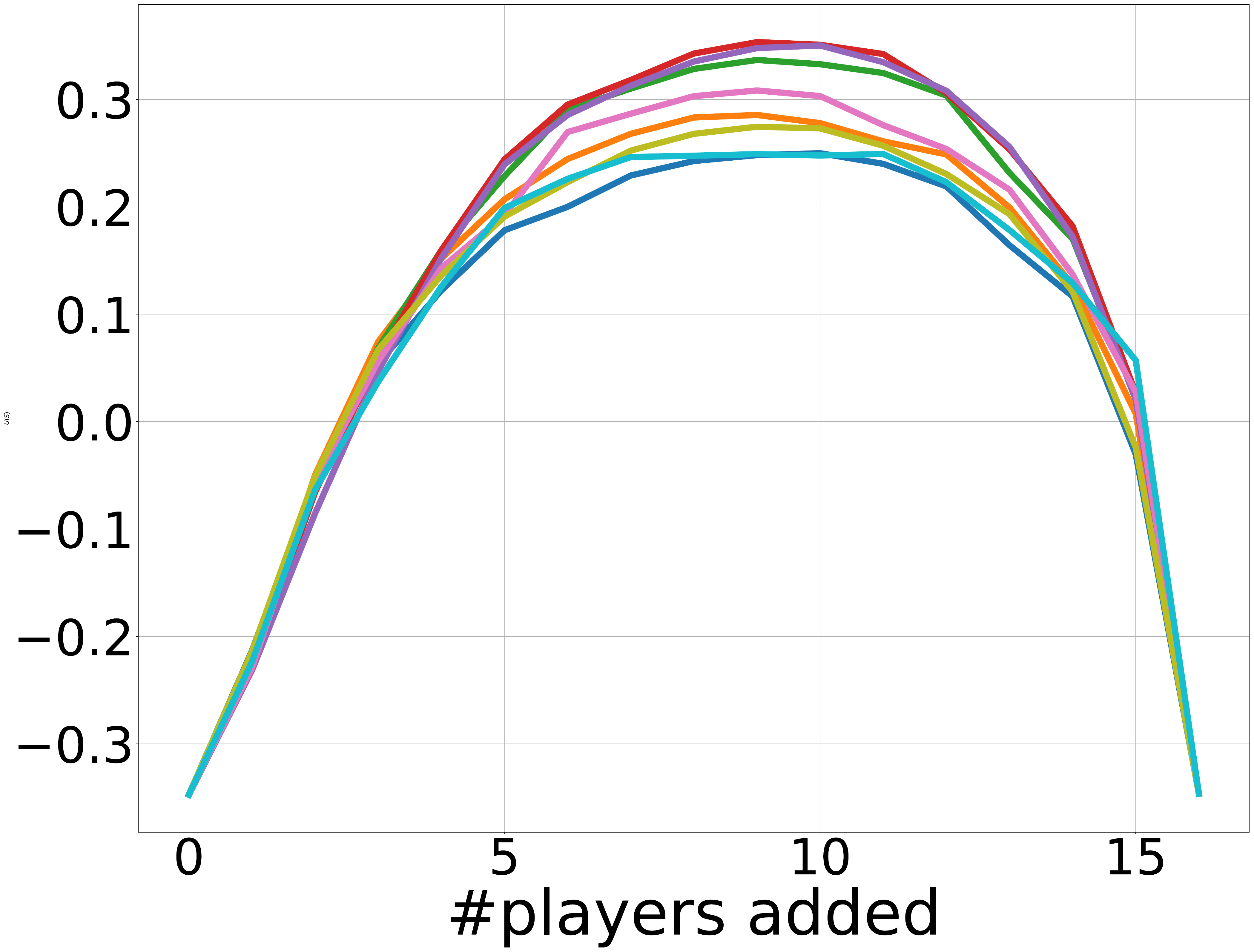} & \includegraphics[width=0.3\linewidth]{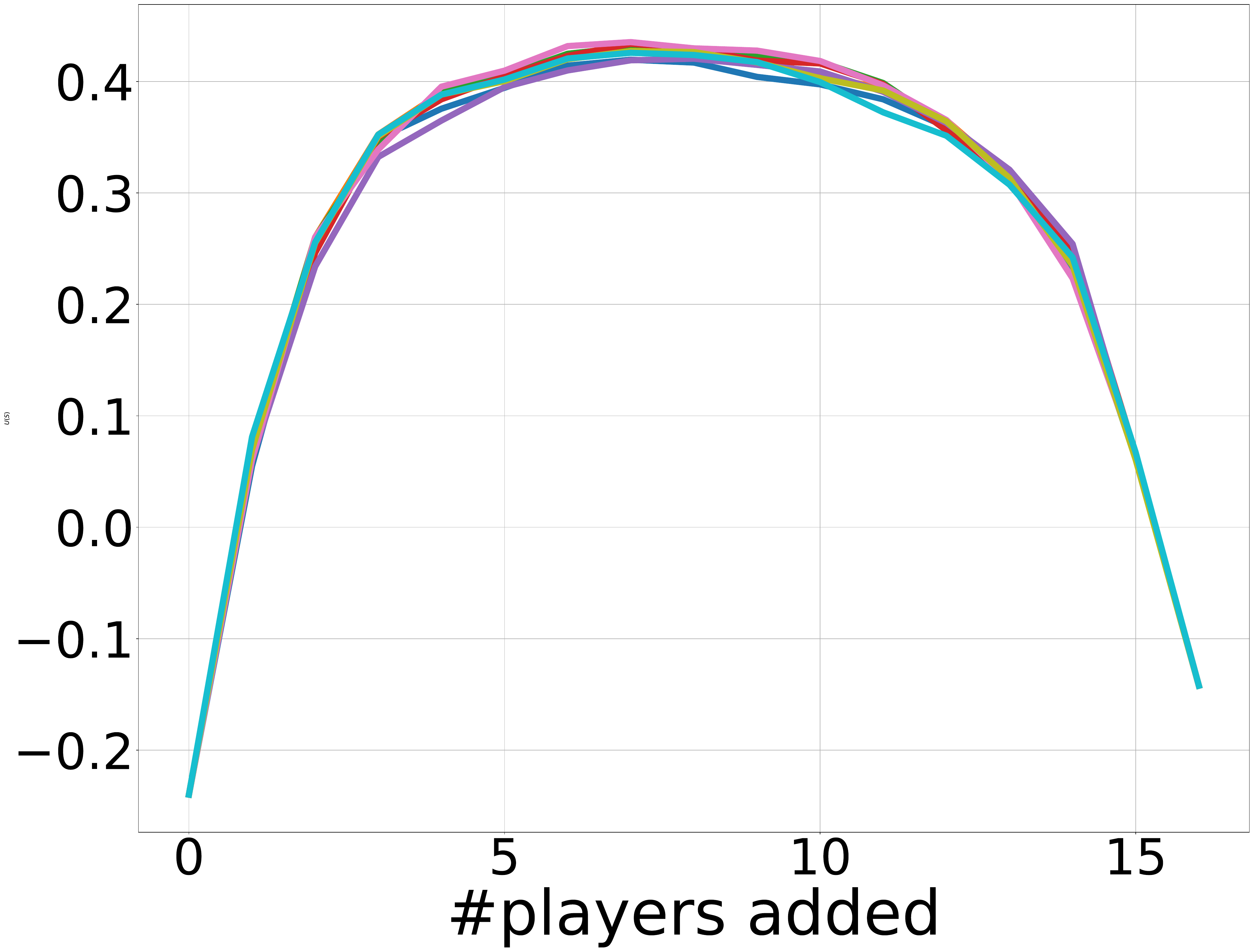} &
		\includegraphics[width=0.3\linewidth]{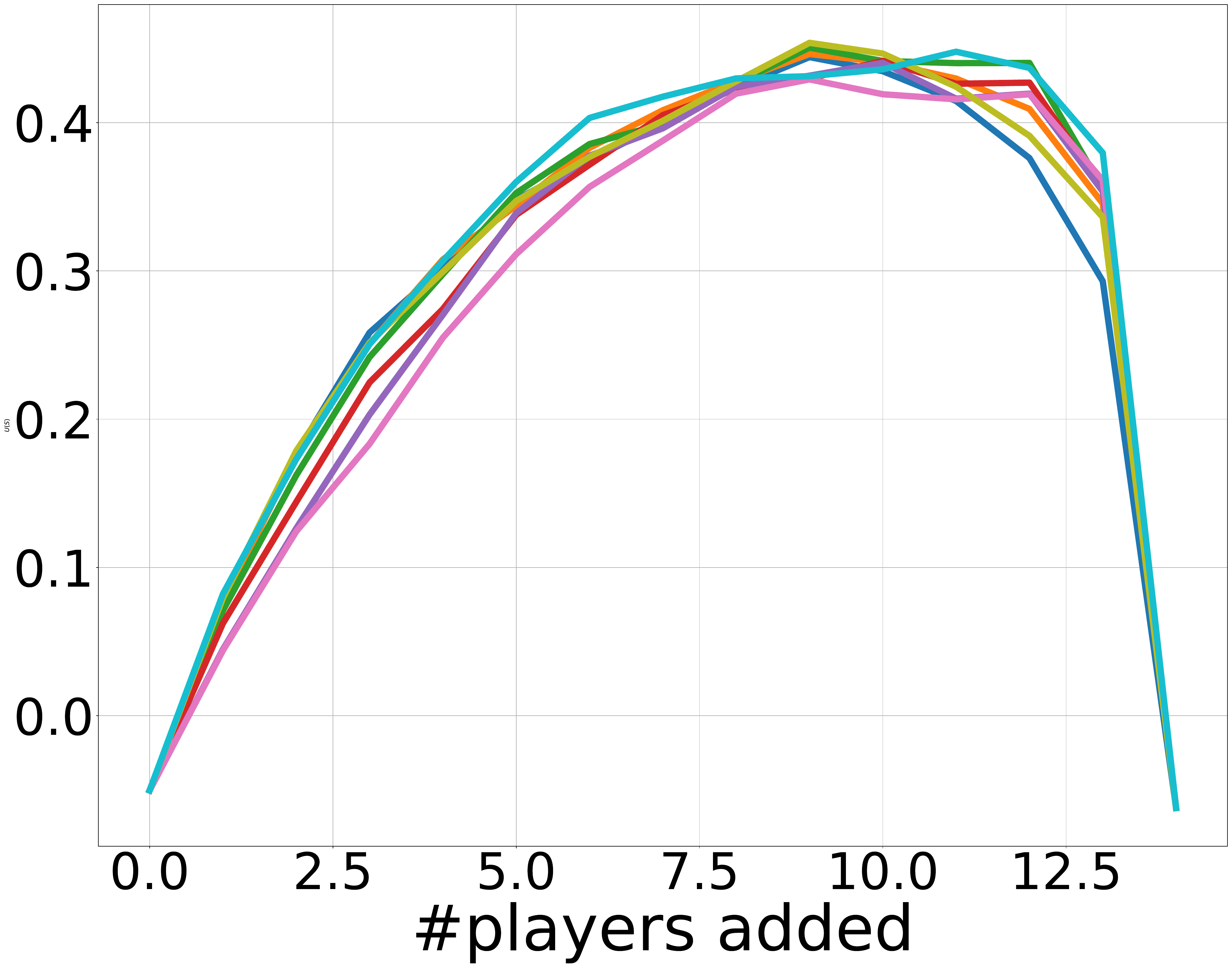} \\
		\phantom{aaaaa}letter & \phantom{aaaaa}pendigits & \phantom{aaaaa}EES\\
		\includegraphics[width=0.3\linewidth]{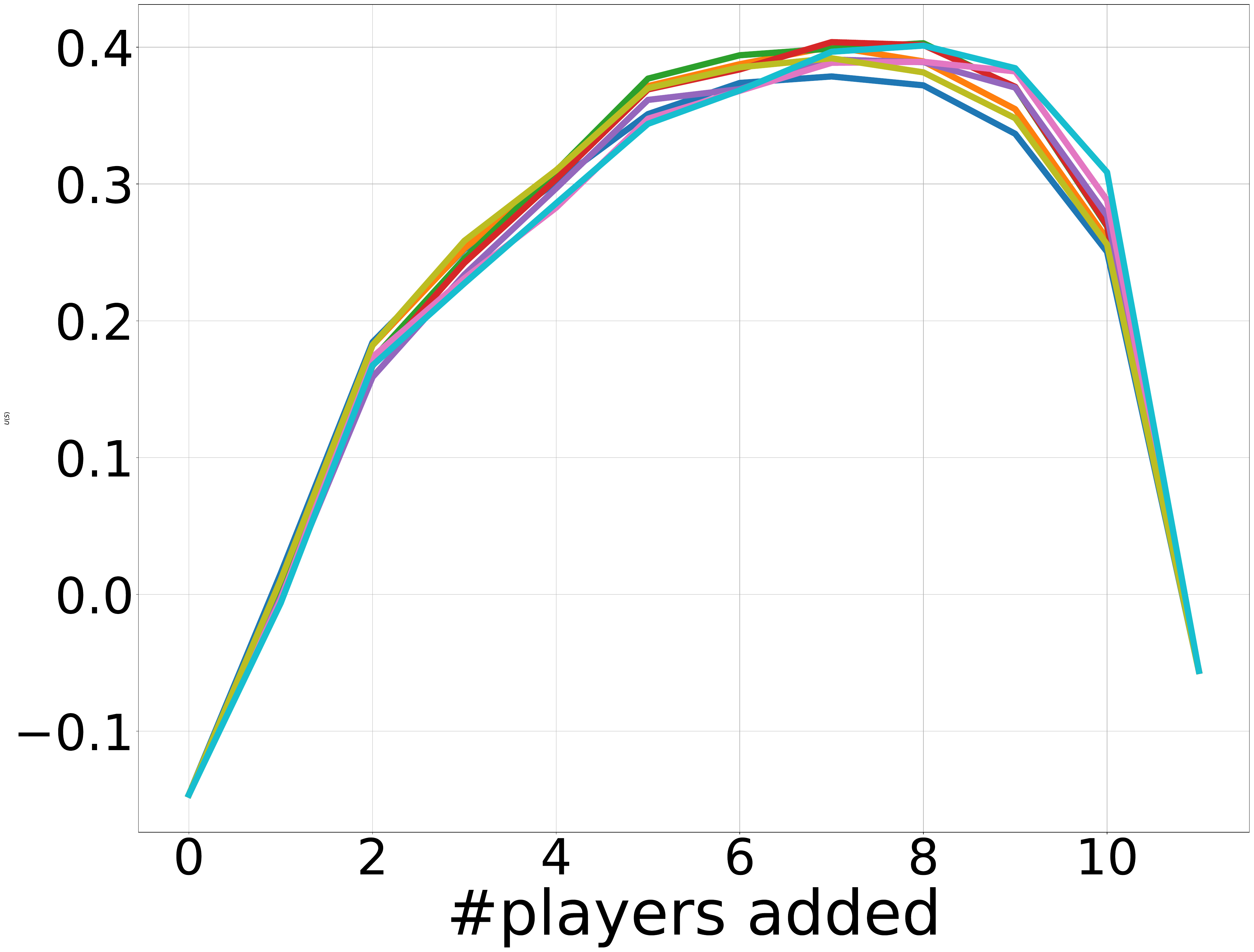} & \includegraphics[width=0.3\linewidth]{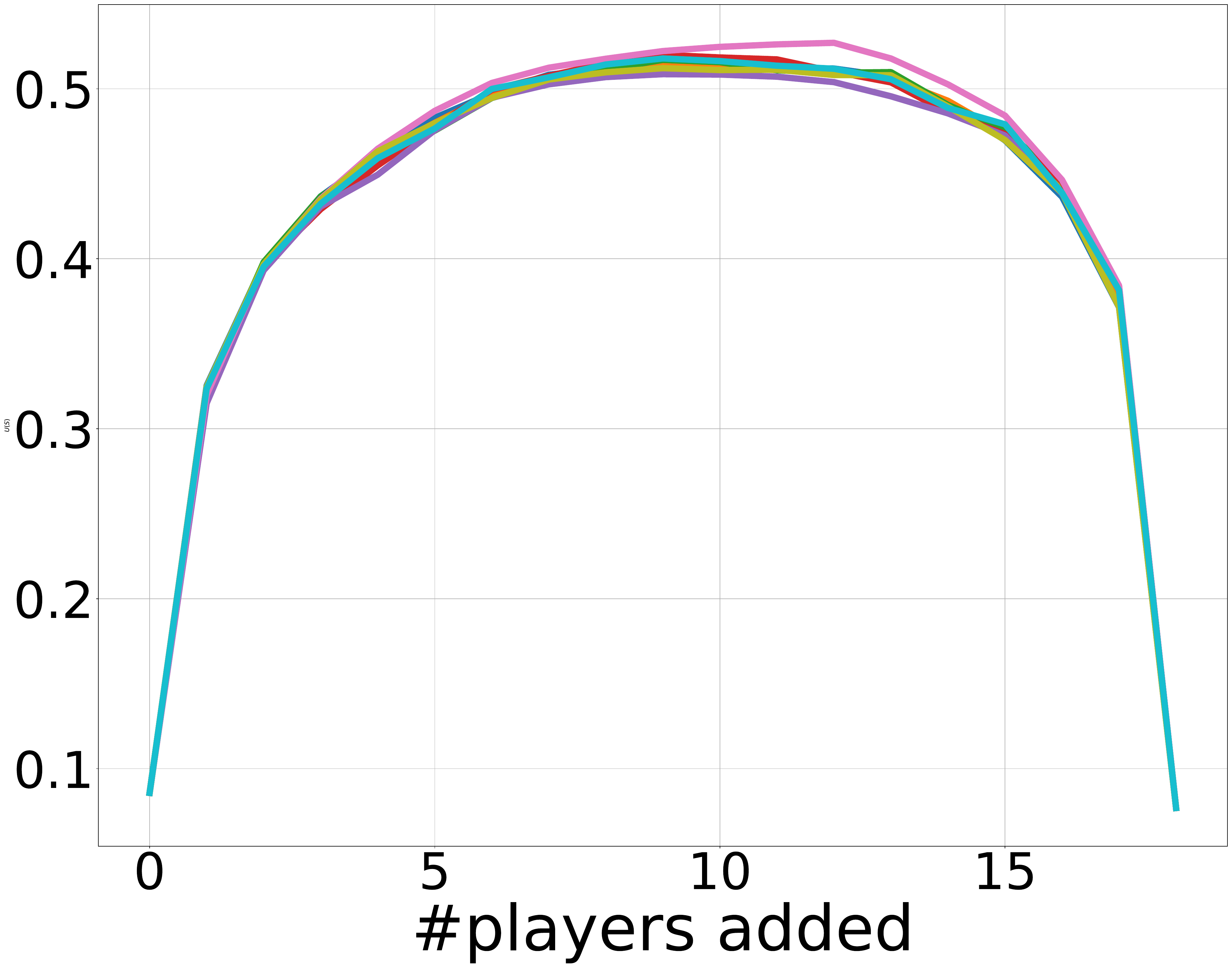} &
		\includegraphics[width=0.3\linewidth]{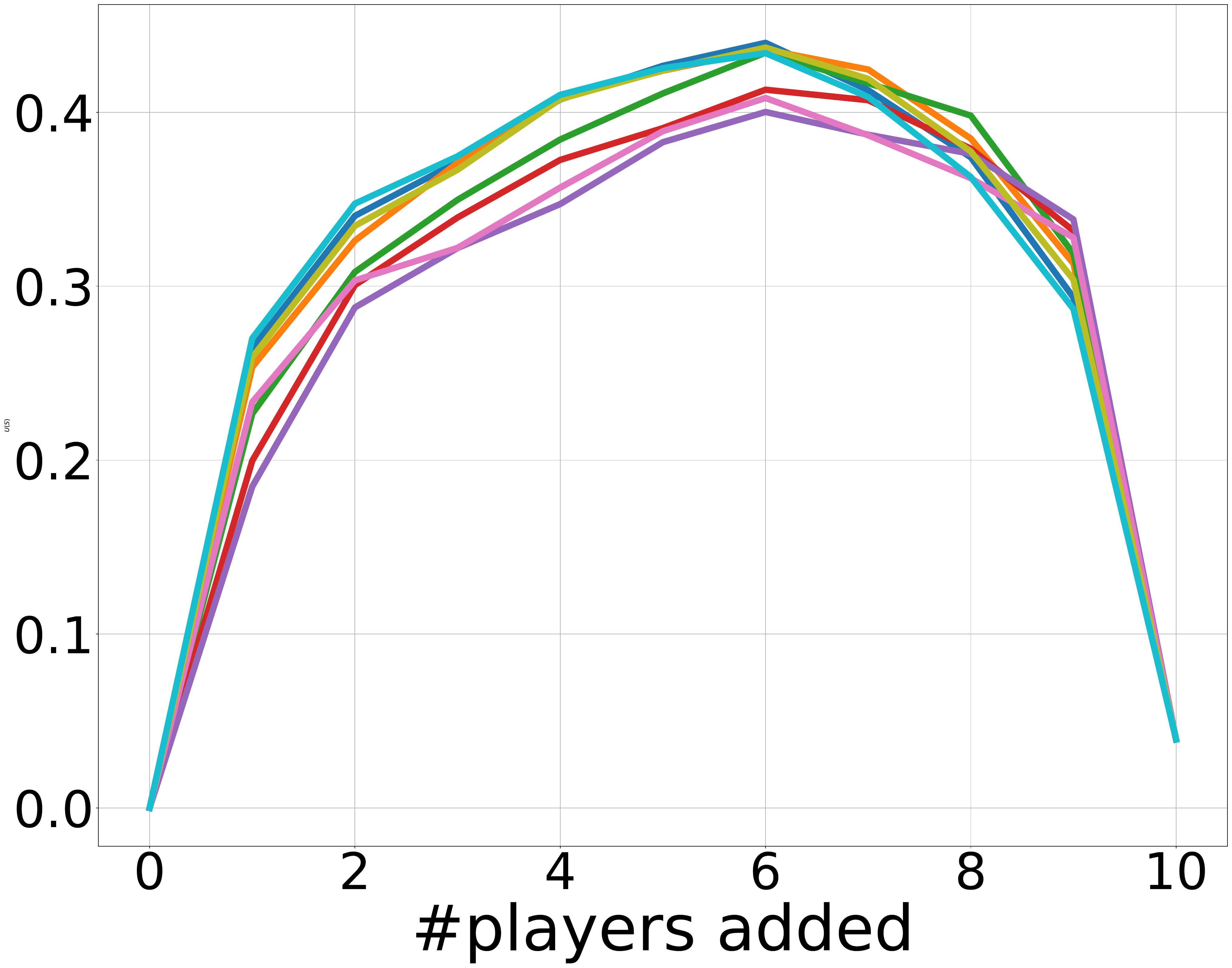} \\
		\phantom{aaaaa}WQW & \phantom{aaaaa}elevators & \phantom{aaaaa}credit
	\end{tabular}
	\caption{Comparison of attribution methods on six datasets for player ranking. A larger area under the curve indicates better performance. Since all utility functions have fewer than $20$ players, all subsets are enumerated exactly for all methods. Beta Shapley achieves the best inclusion AUC among the $10$ candidates.}
	\label{fig:exact-detailed}
\end{figure*}

\begin{figure*}[t]
	\centering
	\begin{tabular}{ccc}
		\includegraphics[width=0.3\linewidth]{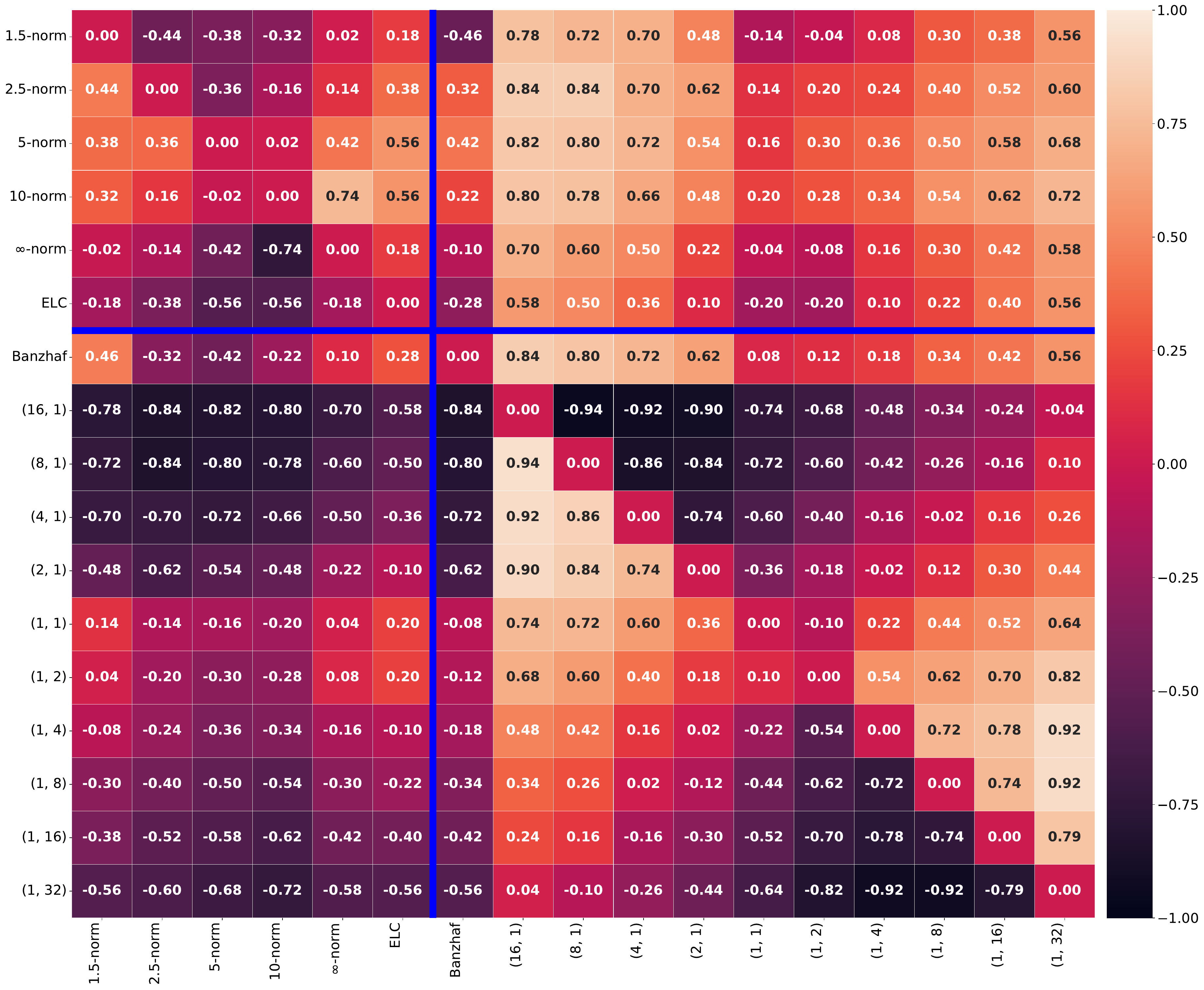} & \includegraphics[width=0.3\linewidth]{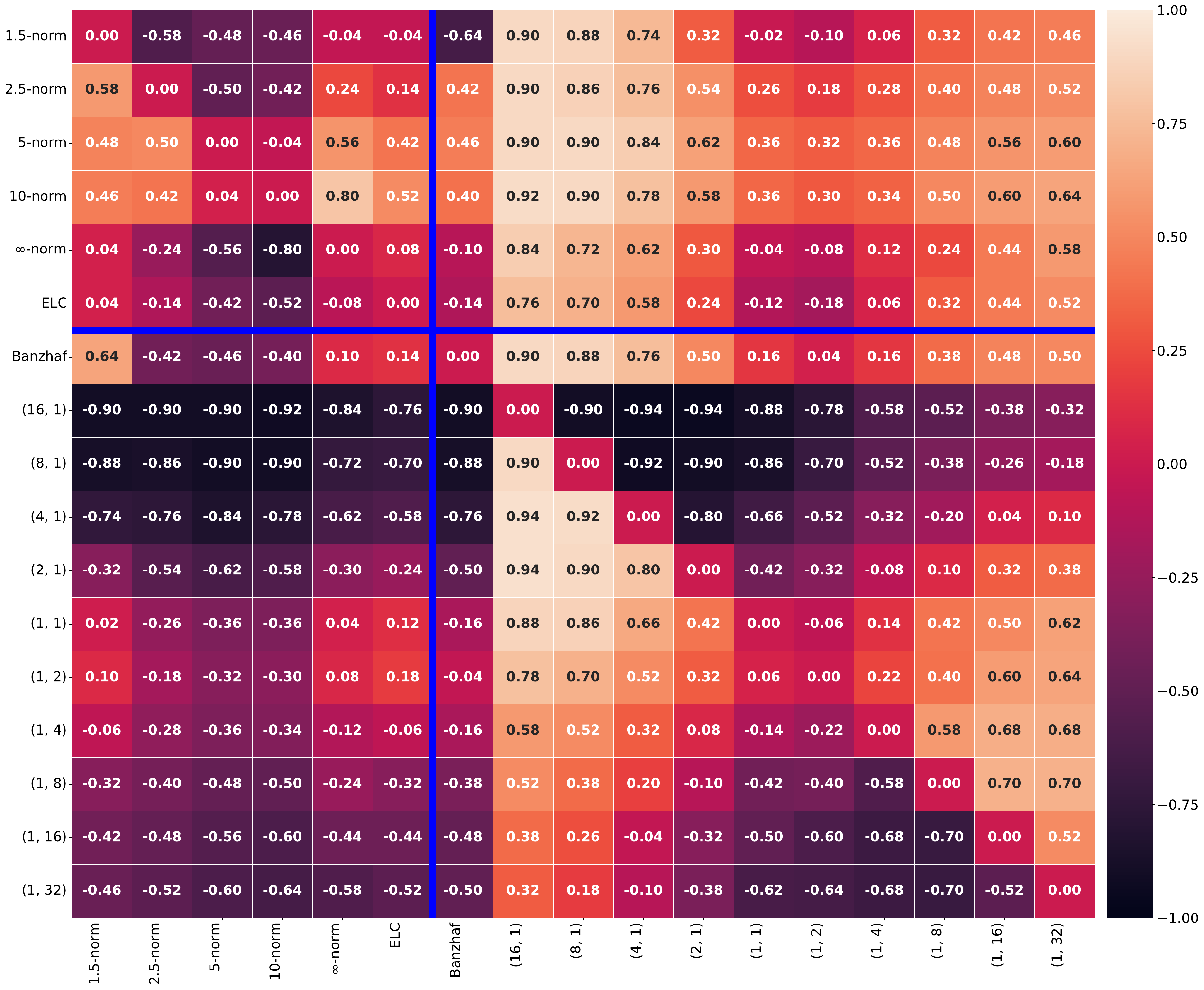} &
		\includegraphics[width=0.3\linewidth]{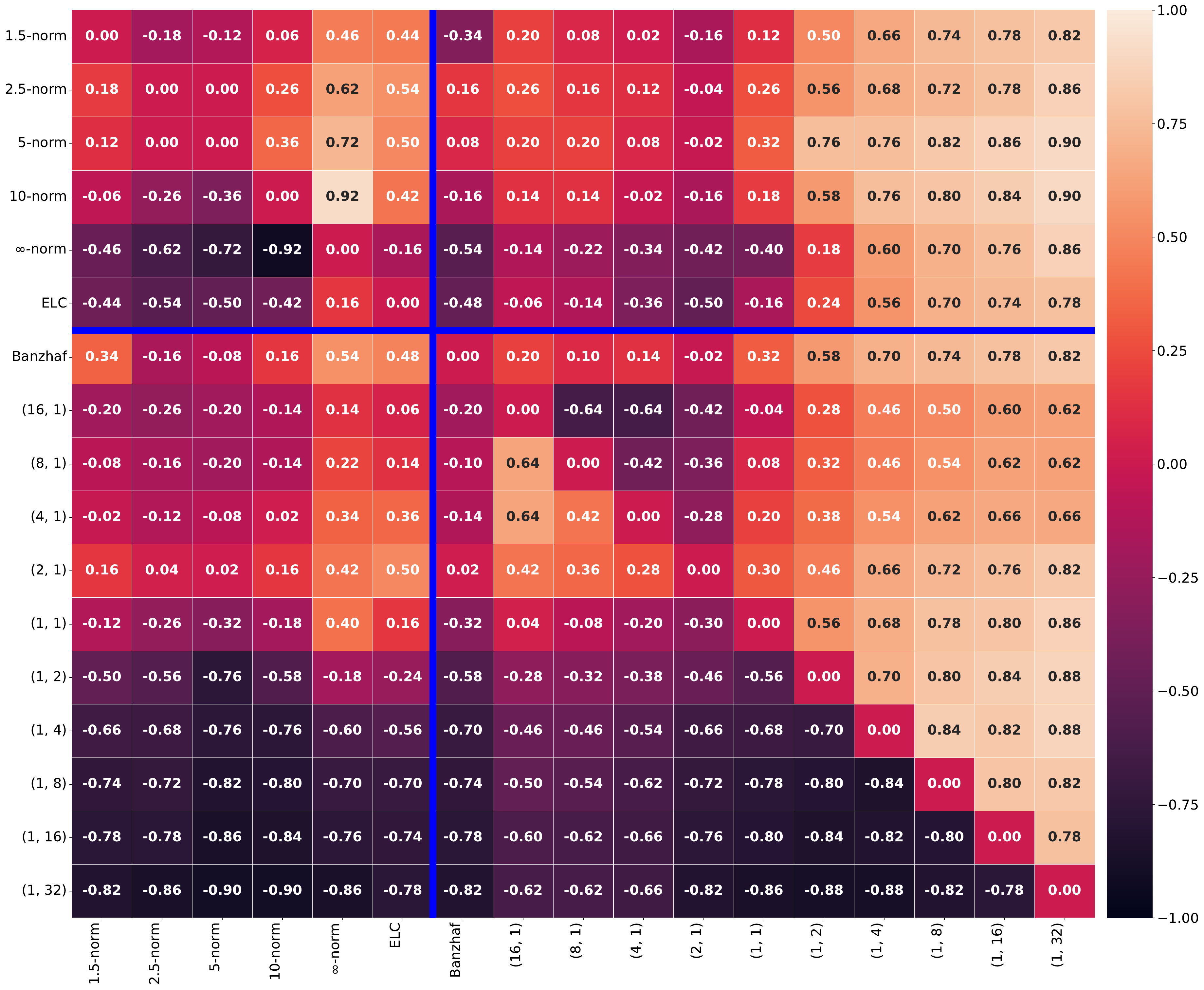} \\
		GPSP & FOTP & wave\_energy\\
		\includegraphics[width=0.3\linewidth]{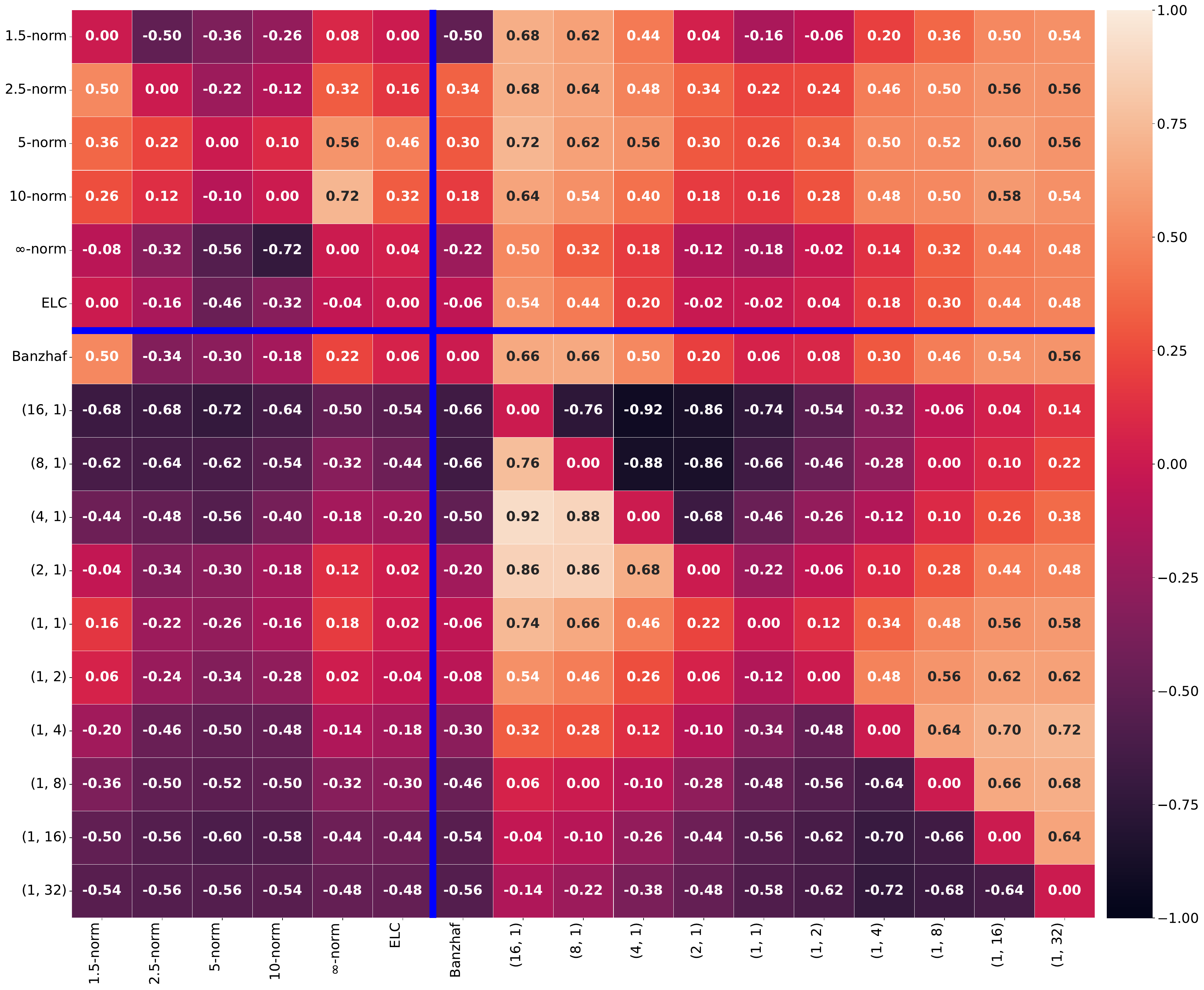} & \includegraphics[width=0.3\linewidth]{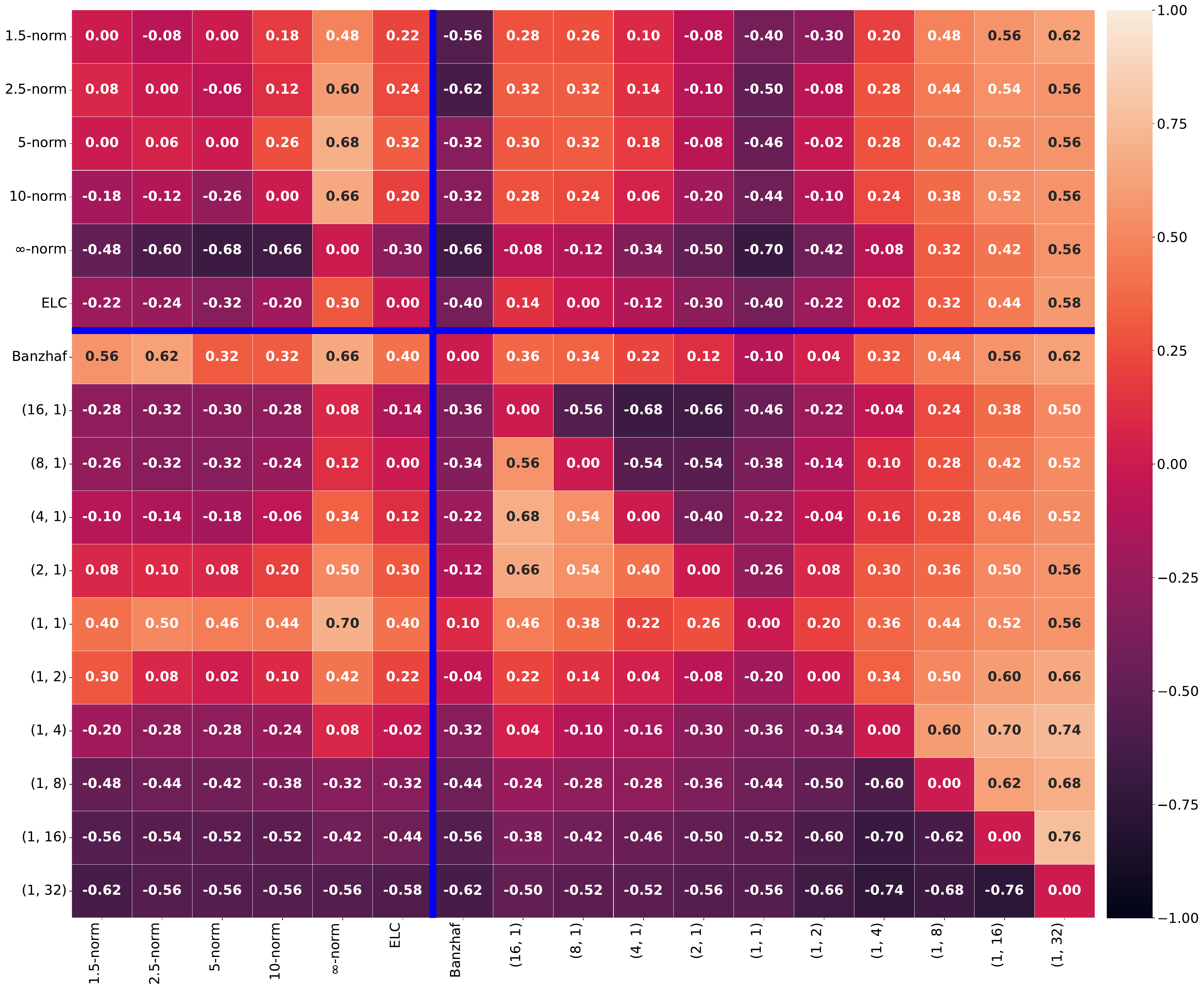} &
		\includegraphics[width=0.3\linewidth]{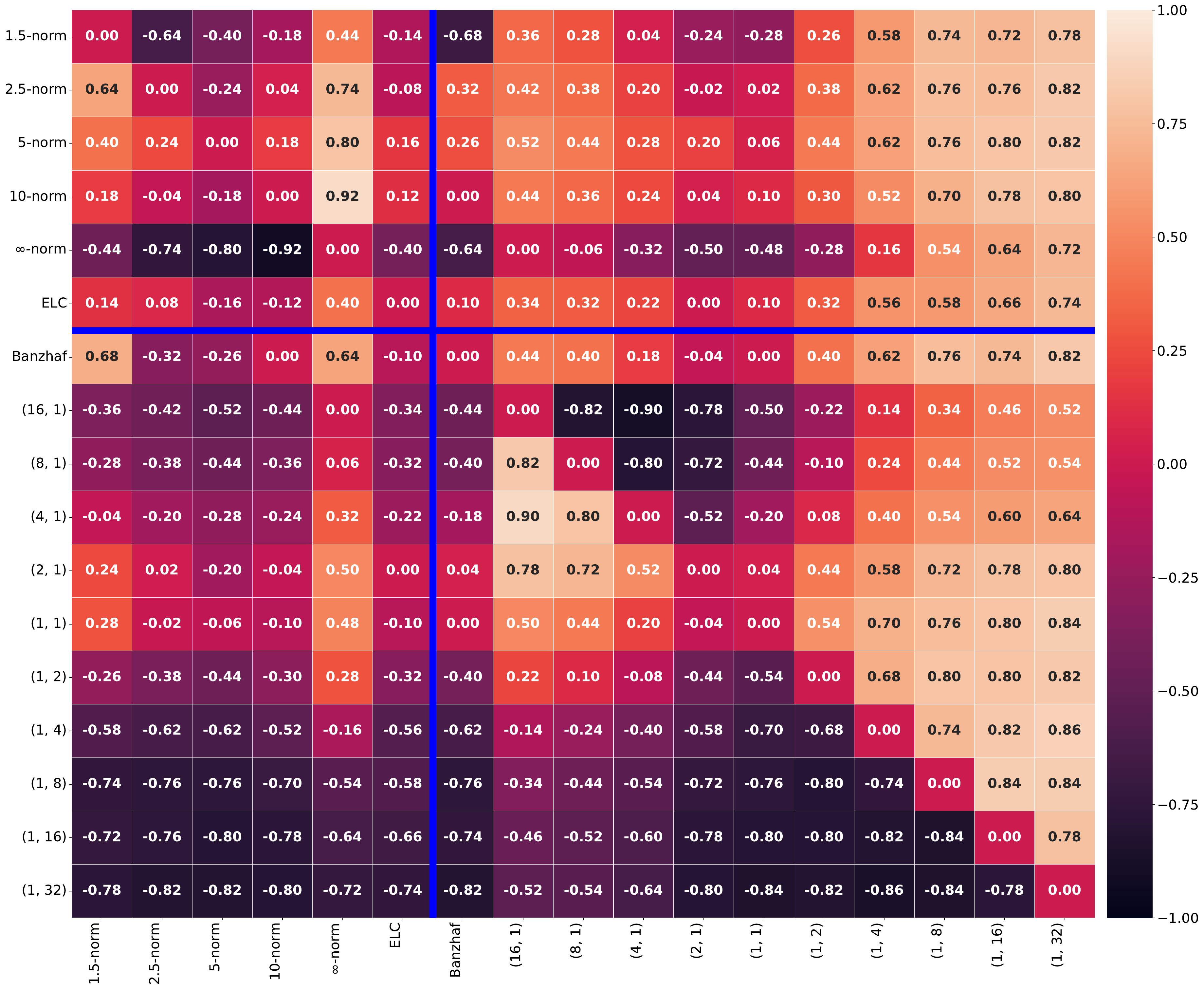} \\
		jannis & spambase & superconduct
	\end{tabular}
	\caption{Statistical comparison of attribution methods on six datasets for player ranking. Since all utility functions contain more than $30$ players, nonlinear methods are approximated by sampling $\#\mathrm{players}\times 1,000$ subsets, with $AUC$ averaged over five random seeds. Linear methods are computed exactly. The blue lines separate the linear and nonlinear methods.}
	\label{fig:heatmap1}
\end{figure*}

\begin{figure*}[t]
	\centering
	\begin{tabular}{ccc}
		\includegraphics[width=0.3\linewidth]{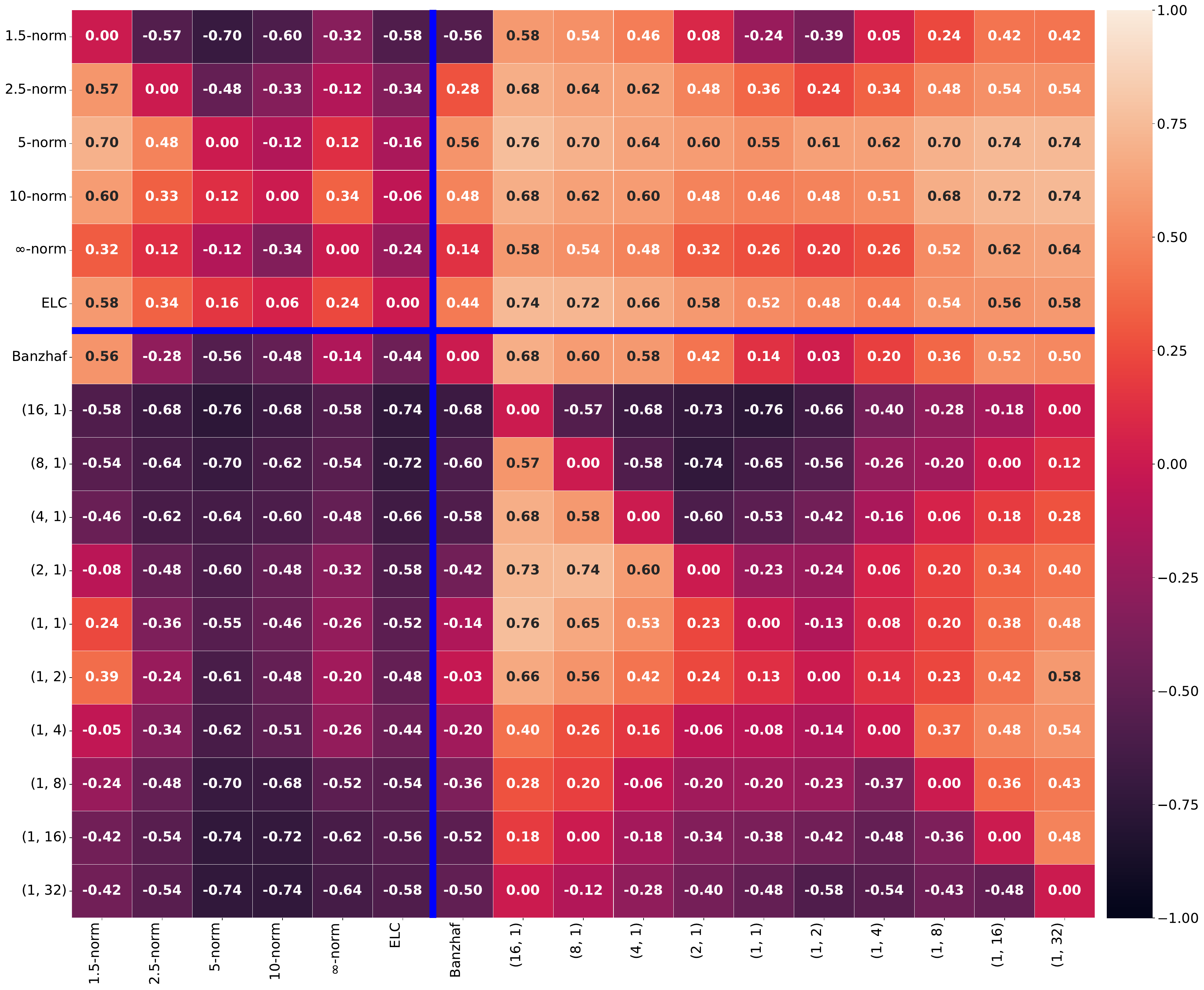} & \includegraphics[width=0.3\linewidth]{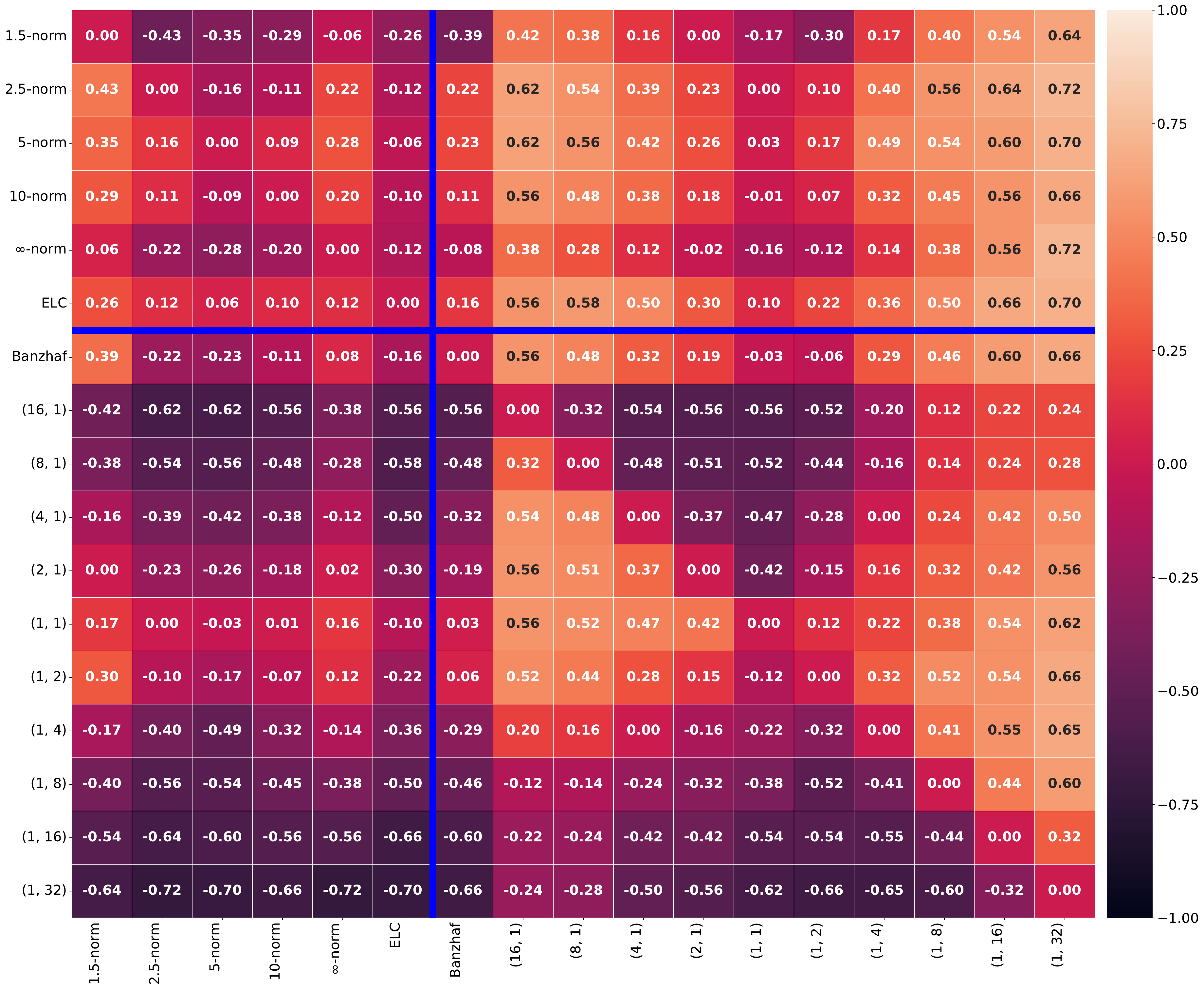} &
		\includegraphics[width=0.3\linewidth]{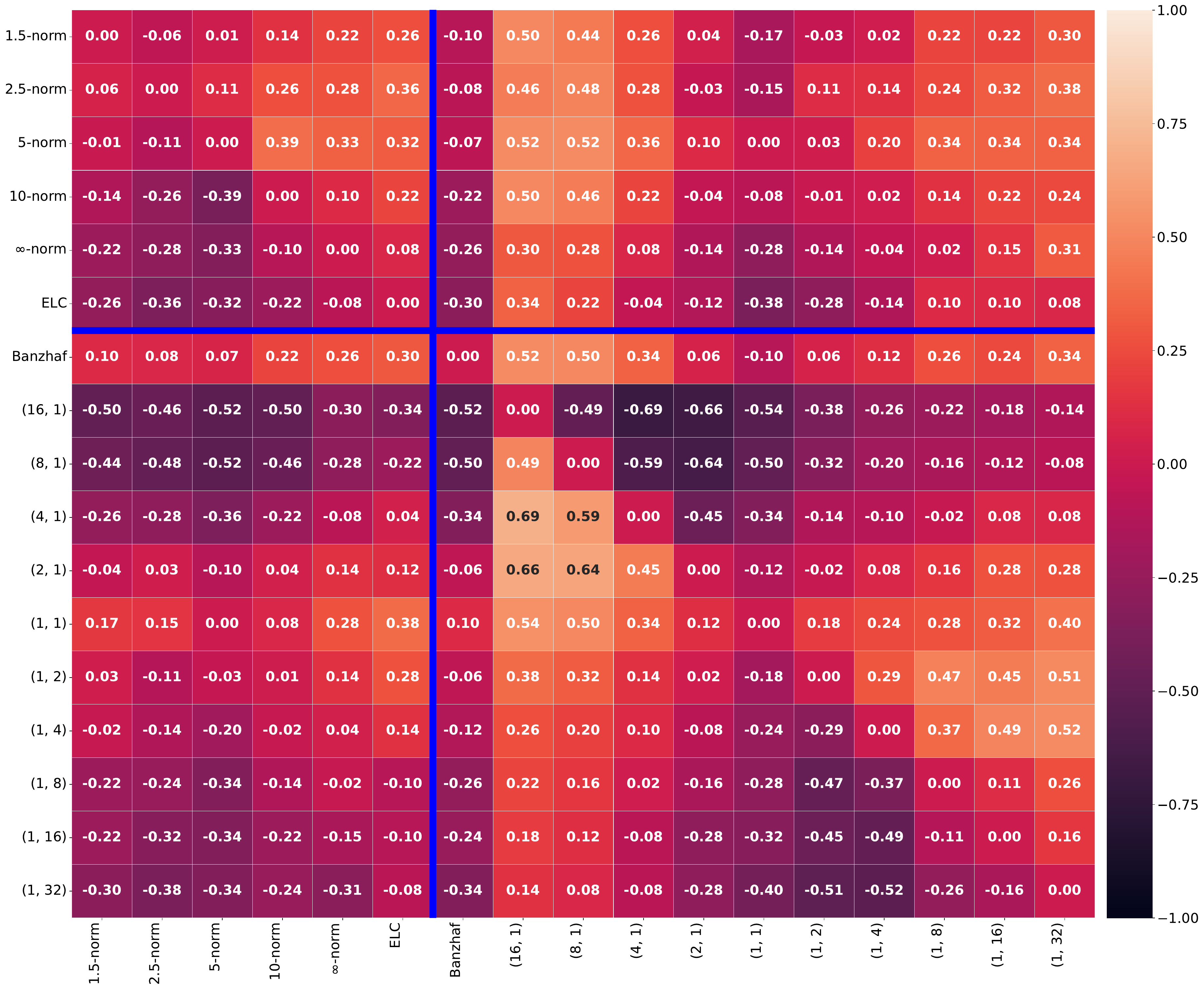} \\
		letter & pendigits & EES\\
		\includegraphics[width=0.3\linewidth]{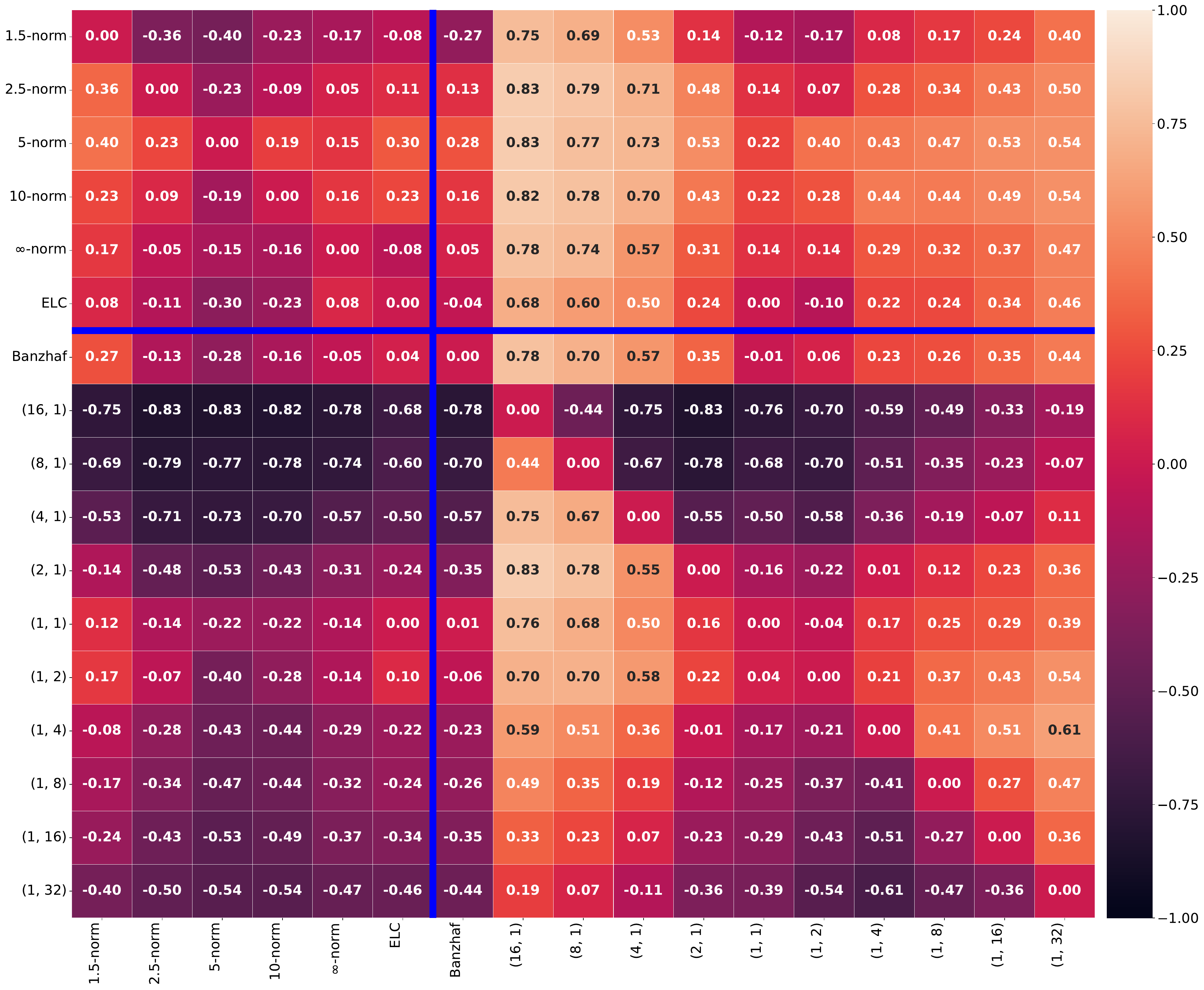} & \includegraphics[width=0.3\linewidth]{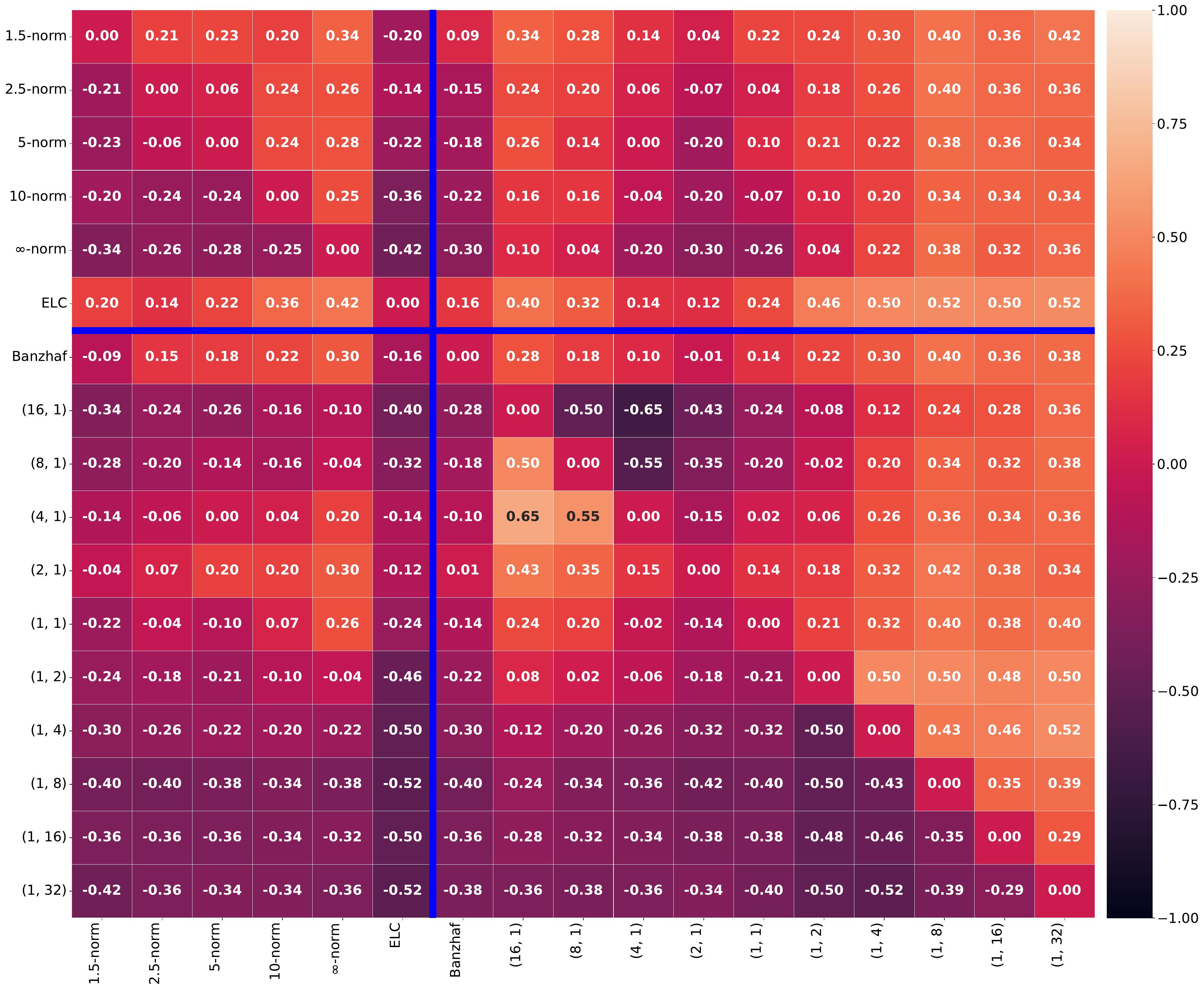} &
		\includegraphics[width=0.3\linewidth]{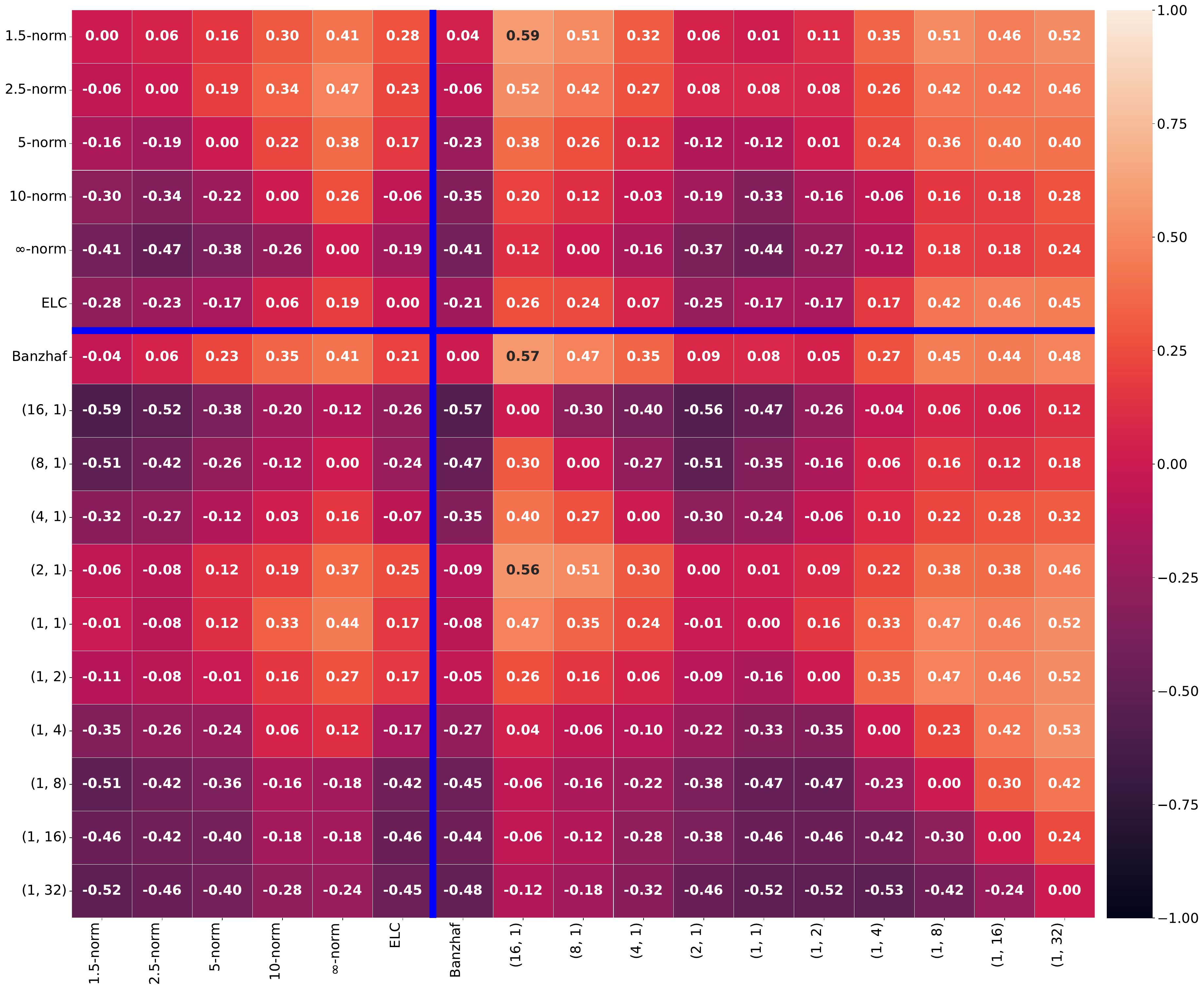} \\
		WQW & elevators & credit
	\end{tabular}
	\caption{Statistical comparison of attribution methods on six datasets for player ranking. Since all utility functions have fewer than $20$ players, all subsets are enumerated exactly for all methods. The blue lines separate the linear and nonlinear methods.}
	\label{fig:heatmap2}
\end{figure*}

%% file: files/references.bib
@inproceedings{
	li2026treegrad,
	title={{TreeGrad-Ranker}: Feature Ranking via {$O(L)$}-Time Gradients for Decision Trees},
	author={Li, Weida and Yu, Yaoliang and Low, Bryan Kian Hsiang},
	booktitle={The Fourteenth International Conference on Learning Representations},
	year={2026}
}

@inproceedings{petsiuk2018rise,
	title={Rise: Randomized input sampling for explanation of black-box models},
	author={Petsiuk, Vitali and Das, Abir and Saenko, Kate},
	booktitle={The British Machine Vision Conference ({BMVC})},
	year={2018}
}

@article{grabisch2000equivalent,
	title={Equivalent representations of set functions},
	author={Grabisch, Michel and Marichal, Jean-Luc and Roubens, Marc},
	journal={Mathematics of Operations Research},
	volume={25},
	number={2},
	pages={157--178},
	year={2000},
	publisher={INFORMS}
}

@inproceedings{wang2023data,
	title={Data {Banzhaf}: A robust data valuation framework for machine learning},
	author={Wang, Jiachen T and Jia, Ruoxi},
	booktitle={International Conference on Artificial Intelligence and Statistics},
	pages={6388--6421},
	year={2023},
	organization={PMLR}
}

@article{cohen2024contextcite,
	title={{ContextCite}: Attributing model generation to context},
	author={Cohen-Wang, Benjamin and Shah, Harshay and Georgiev, Kristian and Madry, Aleksander},
	journal={Advances in Neural Information Processing Systems},
	volume={37},
	pages={95764--95807},
	year={2024}
}

@article{yokote2016new,
	title={A new basis and the {Shapley} value},
	author={Yokote, Koji and Funaki, Yukihiko and Kamijo, Yoshio},
	journal={Mathematical social sciences},
	volume={80},
	pages={21--24},
	year={2016},
	publisher={Elsevier}
}

@article{diamond2016cvxpy,
	author  = {Steven Diamond and Stephen Boyd},
	title   = {{CVXPY}: {A} {P}ython-embedded modeling language for convex optimization},
	journal = {Journal of Machine Learning Research},
	year    = {2016},
	volume  = {17},
	number  = {83},
	pages   = {1--5},
}

@article{agrawal2018rewriting,
	author  = {Agrawal, Akshay and Verschueren, Robin and Diamond, Steven and Boyd, Stephen},
	title   = {A rewriting system for convex optimization problems},
	journal = {Journal of Control and Decision},
	year    = {2018},
	volume  = {5},
	number  = {1},
	pages   = {42--60},
}

@misc{eeg_eye_state_264,
	author       = {Roesler, Oliver},
	title        = {{EEG Eye State}},
	year         = {2013},
	howpublished = {UCI Machine Learning Repository},
	note         = {{DOI}: https://doi.org/10.24432/C57G7J}
}

@article{frey1991letter,
	title={Letter recognition using Holland-style adaptive classifiers},
	author={Frey, Peter W and Slate, David J},
	journal={Machine learning},
	volume={6},
	pages={161--182},
	year={1991},
	publisher={Springer}
}

@article{cortez2009modeling,
	title={Modeling wine preferences by data mining from physicochemical properties},
	author={Cortez, Paulo and Cerdeira, Ant{\'o}nio and Almeida, Fernando and Matos, Telmo and Reis, Jos{\'e}},
	journal={Decision support systems},
	volume={47},
	number={4},
	pages={547--553},
	year={2009},
	publisher={Elsevier}
}

@article{jia2019efficient,
	title={Efficient task-specific data valuation for nearest neighbor algorithms},
	author={Jia, Ruoxi and Dao, David and Wang, Boxin and Hubis, Frances Ann and Gurel, Nezihe Merve and Li, Bo and Zhang, Ce and Spanos, Costas and Song, Dawn},
	journal={Proceedings of the VLDB Endowment},
	volume={12},
	number={11},
	pages={1610--1623},
	year={2019},
	publisher={VLDB Endowment}
}

@article{blum1972direct,
	title={Direct proof of the existence theorem for quadratic programming},
	author={Blum, E and Oettli, Werner},
	journal={Operations Research},
	volume={20},
	number={1},
	pages={165--167},
	year={1972},
	publisher={INFORMS}
}

@article{arin2008axiomatic,
	title={An axiomatic approach to egalitarianism in {TU}-games},
	author={Arin, Javier and Kuipers, Jeroen and Vermeulen, Dries},
	journal={International Journal of Game Theory},
	volume={37},
	number={4},
	pages={565--580},
	year={2008},
	publisher={Springer}
}

@article{shapley1971cores,
	title={Cores of convex games},
	author={Shapley, Lloyd S},
	journal={International journal of game theory},
	volume={1},
	pages={11--26},
	year={1971},
	publisher={Springer}
}

@inproceedings{gemp2024approximating,
	title={Approximating the Core via Iterative Coalition Sampling},
	author={Gemp, Ian and Lanctot, Marc and Marris, Luke and Mao, Yiran and Du{\'e}{\~n}ez-Guzm{\'a}n, Edgar and Perrin, Sarah and Gyorgy, Andras and Elie, Romuald and Piliouras, Georgios and Kaisers, Michael and others},
	booktitle={Proceedings of the 23rd International Conference on Autonomous Agents and Multiagent Systems},
	pages={669--678},
	year={2024}
}

@inproceedings{yan2021if,
	title={If you like {Shapley} then you’ll love the core},
	author={Yan, Tom and Procaccia, Ariel D},
	booktitle={Proceedings of the AAAI Conference on Artificial Intelligence},
	volume={35},
	pages={5751--5759},
	year={2021}
}

@inproceedings{kumar2020problems,
	title={Problems with {Shapley}-value-based explanations as feature importance measures},
	author={Kumar, I Elizabeth and Venkatasubramanian, Suresh and Scheidegger, Carlos and Friedler, Sorelle},
	booktitle={International conference on machine learning},
	pages={5491--5500},
	year={2020},
	organization={PMLR}
}

@inproceedings{muschalik2024beyond,
	title={Beyond {treeSHAP}: Efficient computation of any-order {Shapley} interactions for tree ensembles},
	author={Muschalik, Maximilian and Fumagalli, Fabian and Hammer, Barbara and H{\"u}llermeier, Eyke},
	booktitle={Proceedings of the AAAI Conference on Artificial Intelligence},
	volume={38},
	pages={14388--14396},
	year={2024}
}

@inproceedings{madeo2013gesture,
	title={Gesture unit segmentation using support vector machines: Segmenting gestures from rest positions},
	author={Madeo, Renata CB and Lima, Clodoaldo AM and Peres, Sarajane M},
	booktitle={Proceedings of the 28th Annual ACM Symposium on Applied Computing},
	pages={46--52},
	year={2013}
}

@article{pedregosa2011scikit,
	title={Scikit-learn: Machine learning in {Python}},
	author={Pedregosa, Fabian and Varoquaux, Ga{\"e}l and Gramfort, Alexandre and Michel, Vincent and Thirion, Bertrand and Grisel, Olivier and Blondel, Mathieu and Prettenhofer, Peter and Weiss, Ron and Dubourg, Vincent and others},
	journal={Journal of machine learning research},
	volume={12},
	number={Oct},
	pages={2825--2830},
	year={2011}
}

@article{bridge2014machine,
	title={Machine learning for first-order theorem proving: Learning to select a good heuristic},
	author={Bridge, James P and Holden, Sean B and Paulson, Lawrence C},
	journal={Journal of automated reasoning},
	volume={53},
	pages={141--172},
	year={2014},
	publisher={Springer}
}

@article{bilodeau2024impossibility,
	title={Impossibility theorems for feature attribution},
	author={Bilodeau, Blair and Jaques, Natasha and Koh, Pang Wei and Kim, Been},
	journal={Proceedings of the National Academy of Sciences},
	volume={121},
	number={2},
	pages={e2304406120},
	year={2024},
	publisher={National Acad Sciences}
}

@inproceedings{
	covert2023learning,
	title={Learning to Estimate {Shapley} Values with Vision Transformers},
	author={Ian Connick Covert and Chanwoo Kim and Su-In Lee},
	booktitle={The Eleventh International Conference on Learning Representations },
	year={2023},
}

@inproceedings{jia2019towards,
	title={Towards Efficient Data Valuation Based on the {Shapley} Value},
	author={Jia, Ruoxi and Dao, David and Wang, Boxin and Hubis, Frances Ann and Hynes, Nick and G{\"u}rel, Nezihe Merve and Li, Bo and Zhang, Ce and Song, Dawn and Spanos, Costas J},
	booktitle={The 22nd International Conference on Artificial Intelligence and Statistics},
	pages={1167--1176},
	year={2019}
}

@article{lundberg2017unified,
	title={A Unified Approach to Interpreting Model Predictions},
	author={Lundberg, Scott M and Lee, Su-In},
	journal={Advances in Neural Information Processing Systems},
	volume={30},
	year={2017}
}

@article{marichal2011weighted,
	title={Weighted {Banzhaf} power and interaction indexes through weighted approximations of games},
	author={Marichal, Jean-Luc and Mathonet, Pierre},
	journal={European journal of operational research},
	volume={211},
	number={2},
	pages={352--358},
	year={2011},
}

@article{banzhaf1965weighted,
	title={Weighted voting doesn't work: A mathematical analysis},
	author={Banzhaf III, John F},
	journal={Rutgers Law Review},
	volume={19},
	number={2},
	pages={317--343},
	year={1965},
	publisher={HeinOnline}
}

@article{lundberg2020local,
	title={From Local Explanations to Global Understanding with Explainable {AI} for Trees},
	author={Lundberg, Scott M and Erion, Gabriel and Chen, Hugh and DeGrave, Alex and Prutkin, Jordan M and Nair, Bala and Katz, Ronit and Himmelfarb, Jonathan and Bansal, Nisha and Lee, Su-In},
	journal={Nature machine intelligence},
	volume={2},
	number={1},
	pages={56--67},
	year={2020}
}

@article{kwon2022weightedshap,
	title={{WeightedSHAP}: Analyzing and Improving {Shapley} Based Feature Attributions},
	author={Kwon, Yongchan and Zou, James Y},
	journal={Advances in Neural Information Processing Systems},
	volume={35},
	pages={34363--34376},
	year={2022}
}

@article{li2023robust,
	title={Robust Data Valuation with Weighted {Banzhaf} Values},
	author={Li, Weida and Yu, Yaoliang},
	journal={Advances in Neural Information Processing Systems},
	volume={36},
	year={2023}
}

@article{Yu2022linear,
	title={Linear {TreeShap}},
	author={Yu, Peng and Xu, Chao and Bifet, Albert and Read, Jesse},
	journal={Advances in Neural Information Processing Systems},
	volume={35},
	pages={25818--25828},
	year={2022}
}

@inproceedings{
	li2024one,
	title={One Sample Fits All: Approximating All Probabilistic Values Simultaneously and Efficiently},
	author={Weida Li and Yaoliang Yu},
	booktitle={The Thirty-eighth Annual Conference on Neural Information Processing Systems},
	year={2024},
}

@inproceedings{karczmarz2022improved,
	title={Improved Feature Importance Computation for Tree Models Based on the {Banzhaf} Value},
	author={Karczmarz, Adam and Michalak, Tomasz and Mukherjee, Anish and Sankowski, Piotr and Wygocki, Piotr},
	booktitle={Uncertainty in Artificial Intelligence},
	pages={969--979},
	year={2022},
	organization={PMLR}
}

@article{shapley1953value,
	title={A Value for N-Person Games},
	author={Shapley, Lloyd S},
	journal={Annals of Mathematics Studies},
	volume={28},
	year={1953},
	pages={307-317},
}

@inbook{weber1988probabilistic,
	title={Probabilistic values for games}, 
	booktitle={The Shapley Value: Essays in Honor of Lloyd S. Shapley}, 
	publisher={Cambridge University Press}, author={Weber, Robert James},
	year={1988}, 
	pages={101–120}
}

@article{dubey1981value,
	title={Value Theory without Efficiency},
	author={Dubey, Pradeep and Neyman, Abraham and Weber, Robert James},
	journal={Mathematics of Operations Research},
	volume={6},
	number={1},
	pages={122--128},
	year={1981},
}

@inproceedings{kwon2022beta,
	title={Beta {Shapley}: A Unified and Noise-reduced Data Valuation Framework for Machine Learning},
	author={Kwon, Yongchan and Zou, James Y},
	booktitle={International Conference on Artificial Intelligence and Statistics},
	pages={8780--8802},
	year={2022},
}

@inproceedings{CharnesGKR88,
	title={Extremal Principle Solutions of Games in Characteristic Function Form: Core, {C}hebychev and {S}hapley Value Generalizations},
	author={Abraham Charnes and Boaz Golany and Michael S. Keane and John J. Rousseau},
	booktitle={Econometrics of Planning and Efficiency},
	pages={123--133},
	year={1988},
}

@inproceedings{wang2024rethinking,
	title={Rethinking {Data Shapley} for Data Selection Tasks: Misleads and Merits},
	author={Jiachen T. Wang and Tianji Yang and James Zou and Yongchan Kwon and Ruoxi Jia},
	booktitle={Forty-first International Conference on Machine Learning},
	year={2024},
}
